\documentclass[10pt]{article} %
\usepackage[accepted]{tmlr}

\usepackage{amsmath,amsfonts,bm}

\def\eqref#1{equation~\ref{#1}}

\def\1{\bm{1}}

\def\rvd{{\mathbf{d}}}

\def\rvf{{\mathbf{f}}}
\def\rvg{{\mathbf{g}}}

\def\rvk{{\mathbf{k}}}

\def\rvn{{\mathbf{n}}}

\def\rvv{{\mathbf{v}}}

\def\rvx{{\mathbf{x}}}
\def\rvy{{\mathbf{y}}}
\def\rvz{{\mathbf{z}}}

\def\vmu{{\bm{\mu}}}
\def\vtheta{{\bm{\theta}}}

\def\mI{{\bm{I}}}

\DeclareMathAlphabet{\mathsfit}{\encodingdefault}{\sfdefault}{m}{sl}
\SetMathAlphabet{\mathsfit}{bold}{\encodingdefault}{\sfdefault}{bx}{n}

\def\gJ{{\mathcal{J}}}

\def\gN{{\mathcal{N}}}
\def\gO{{\mathcal{O}}}

\def\sB{{\mathbb{B}}}

\newcommand{\pdata}{p_{\rm{data}}}

\newcommand{\E}{\mathbb{E}}

\newcommand{\R}{\mathbb{R}}

\newcommand{\Var}{\mathrm{Var}}

\DeclareMathOperator*{\argmin}{arg\,min}

\usepackage{amssymb}
\usepackage{mathtools}
\usepackage{amsthm}

\usepackage[colorlinks,linkcolor=black,citecolor=blue,urlcolor=blue,bookmarks=true]{hyperref}

\usepackage{url}
\usepackage{booktabs}
\usepackage{multirow}
\usepackage[capitalize,noabbrev]{cleveref}
\usepackage{tikz}
\usepackage{microtype}      %
\usepackage{xcolor}

\usepackage{cleveref}
\Crefname{figure}{Fig.}{Figs.}
\Crefname{table}{Tab.}{Tabs.}
\Crefname{section}{Sec.}{Secs.}
\Crefname{appendix}{App.}{Apps.}
\Crefname{equation}{Eq.}{Eqs.}
\Crefname{algorithm}{Alg.}{Algs.}

\usepackage{subcaption}
\usepackage{lipsum}  
\usepackage{wrapfig}
\usepackage{tabu}
\usepackage{pbox}
\theoremstyle{definition}
\newtheorem{definition}{Definition}[section]
\newtheorem{theorem}{Theorem}
\newtheorem*{theorem*}{Theorem}
\usepackage{algorithm}
\usepackage{algorithmic}

\usepackage{pgfplots}
\pgfplotsset{compat=1.17}

\usepackage{pgfkeys}
\def\addlegendimage{\csname pgfplots@addlegendimage\endcsname}

\definecolor{set21}{RGB}{102.0 194.0 165.0}
\definecolor{set22}{RGB}{252.0 141.0 98.0}
\definecolor{set23}{RGB}{141.0 160.0 203.0}
\definecolor{set24}{RGB}{231.0 138.0 195.0}
\definecolor{set25}{RGB}{166.0 216.0 84.0}
\definecolor{set26}{RGB}{255.0 217.0 47.0}
\definecolor{set27}{RGB}{229.0 196.0 148.0}
\definecolor{set28}{RGB}{179.0 179.0 179.0}
\tikzset{every picture/.style=semithick}

\definecolor{nvgreen}{HTML}{76B900}
\definecolor{myblue}{HTML}{03a9fc}
\definecolor{myred}{HTML}{A03F77}

\newcommand{\TIM}[1]{{{#1}}} 
\newcommand{\TC}[1]{{{#1}}} 
\newcommand{\TO}[1]{}

\title{Differentially Private Diffusion Models}

\author{\name Tim Dockhorn \email timdockhorn@gmail.com \\
      \addr Stability AI \thanks{Work done during internship at NVIDIA previous to employment at Stability AI.}
      \AND
      \name Tianshi Cao \email tianshic@nvidia.com \\
      \addr NVIDIA \\
      University of Toronto \\
      Vector Institute
      \AND
      \name Arash Vahdat \email avahdat@nvidia.com \\
      \addr NVIDIA \\
      \AND
            \name Karsten Kreis \email kkreis@nvidia.com \\
      \addr NVIDIA}

\def\month{08}  %
\def\year{2023} %
\def\openreview{\url{https://openreview.net/forum?id=ZPpQk7FJXF}} %

\begin{document}

\maketitle

\begin{abstract}
While modern machine learning models rely on increasingly large training datasets, data is often limited in privacy-sensitive domains. Generative models trained with differential privacy (DP) on sensitive data can sidestep this challenge, providing access to synthetic data instead. We build on the recent success of diffusion models (DMs) and introduce \textit{Differentially Private Diffusion Models} (DPDMs), which enforce privacy using differentially private stochastic gradient descent (DP-SGD). We investigate the DM parameterization and the sampling algorithm, which turn out to be crucial ingredients in DPDMs, and propose \textit{noise multiplicity}, a powerful modification of DP-SGD tailored to the training of DMs. We validate our novel DPDMs on image generation benchmarks and achieve state-of-the-art performance in all experiments. Moreover, on standard benchmarks, classifiers trained on DPDM-generated synthetic data perform on par with task-specific DP-SGD-trained classifiers, which has not been demonstrated before for DP generative models. Project page and code: \url{https://nv-tlabs.github.io/DPDM}.
\end{abstract}

\vspace{-1mm}
\section{Introduction} \label{sec:introduction}
\vspace{-1mm}
Modern deep learning usually requires significant amounts of training data. However, sourcing large datasets in privacy-sensitive domains is often difficult. To circumvent this challenge, generative models trained on sensitive data can provide access to large synthetic data instead, which can be used flexibly to train downstream models. Unfortunately, typical overparameterized neural networks have been shown to provide little to no privacy to the data they have been trained on. For example, an adversary may be able to recover training images of deep classifiers using gradients of the networks~\citep{yin2021see} or reproduce training text sequences from large transformers~\citep{carlini2021extracting}. Generative models
may even overfit directly, generating data indistinguishable from the data they have been trained on.
In fact, overfitting and privacy-leakage of generative models are more relevant than ever, considering recent works that train powerful photo-realistic image generators on large-scale Internet-scraped data~\citep{rombach2021highresolution,ramesh2022dalle2,saharia2022imagen,balaji2022eDiffi}.\looseness=-1

\begin{figure}
\vspace{-3mm}
    \centering
    \includegraphics[width=0.72\textwidth]{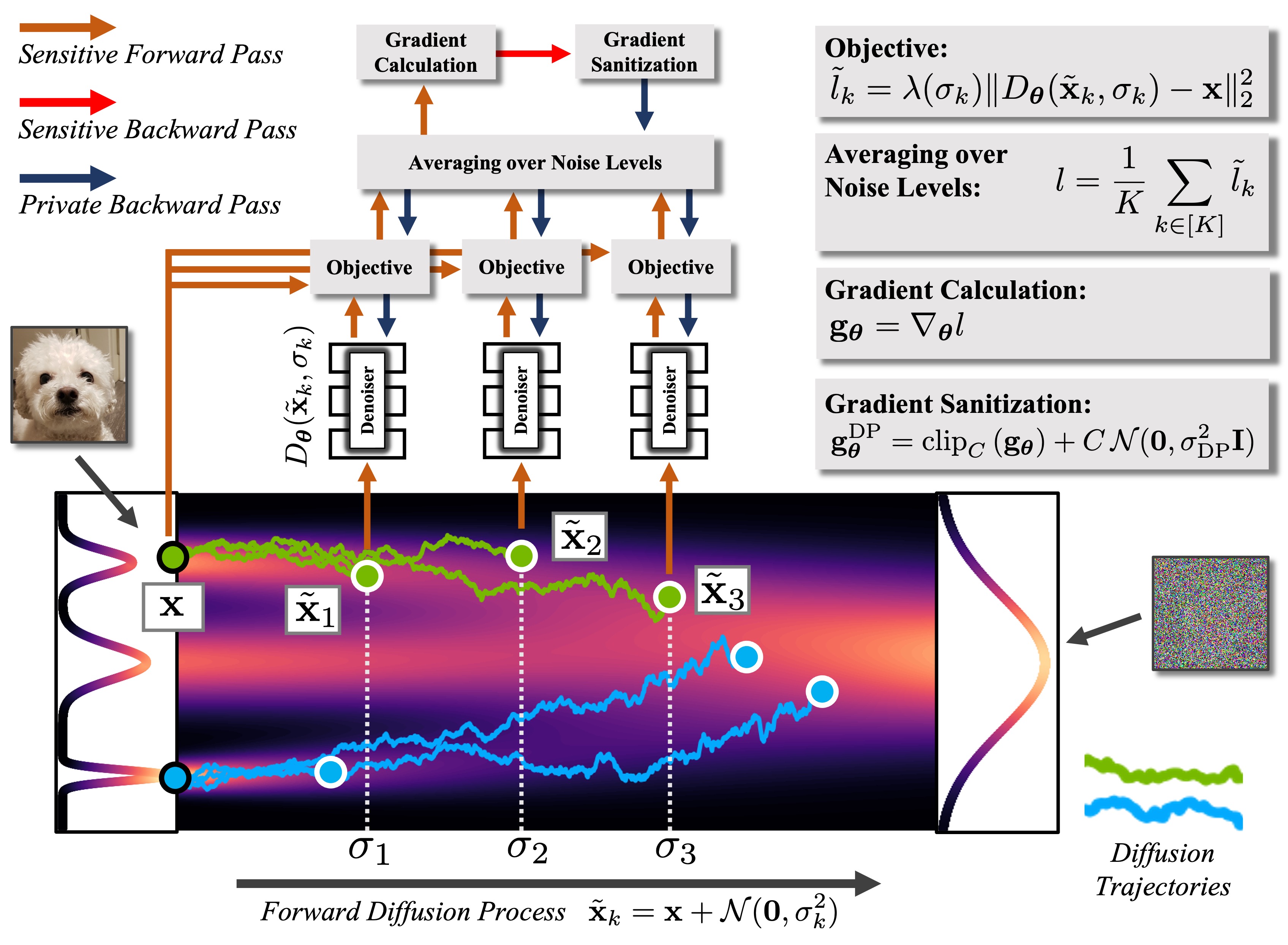}
    \caption{Information flow during training in our \textit{Differentially Private Diffusion Model} (DPDM) for a single training sample in \textbf{\textcolor{nvgreen}{green}} (\textit{i.e.} batchsize $B{=}1$, another sample shown in \textbf{\textcolor{myblue}{blue}}). We rely on DP-SGD to guarantee privacy and use \textit{noise multiplicity}; here, $K{=}3$. 
    The diffusion is visualized for a one-dim. toy distribution (marginal probabilities in \textbf{\textcolor{myred}{purple}}); our main experiments use high-dim. images.
    Note that for brevity in the visualization we dropped the index $i$, which indicates the minibatch element in \Cref{eq:noisy_objective,eq:noise_mult}.\looseness=-1}
    \label{fig:pipeline}
    \vspace{-1mm}
\end{figure}

To protect the privacy of training data, one may train their model using differential privacy (DP). DP is a rigorous privacy framework that applies to statistical queries~\citep{dwork2006calibrating, dwork2014algorithmic}. 
In our case, this query corresponds to the training of a neural network using sensitive data.
Differentially private stochastic gradient descent (DP-SGD)~\citep{abadi2016deep} is the workhorse of DP training of neural networks. It preserves privacy by clipping and noising the parameter gradients during training. This leads to an inevitable trade-off between privacy and utility; for instance, small clipping constants and large noise injection result in very private models that may be of little practical use.\looseness=-1

DP-SGD has, for example, been employed to train generative adversarial networks (GANs)~\citep{frigerio2019differentially, torkzadehmahani2019dp, xie2018differentially}, which are
particularly susceptible to privacy-leakage~\citep{webster2021person}.
However, while GANs in the non-private setting can synthesize photo-realistic images~\citep{brock2018large, karras2020analyzing, karras2020training, karras2021alias}, their application in the private setting is challenging. GANs are difficult to optimize~\citep{arjovsky2017towards, pmlr-v80-mescheder18a} and prone to mode collapse; 
both phenomena
may be amplified during DP-SGD training.\looseness=-1

Recently, Diffusion Models (DMs) have emerged as a powerful class of generative models~\citep{song2020, ho2020, sohl2015}, demonstrating outstanding performance in image synthesis~\citep{ho2021arxiv,nichol21,dhariwal2021diffusion,rombach2021highresolution,ramesh2022dalle2,saharia2022imagen}. In DMs, a diffusion process gradually perturbs the data towards random noise, while a deep neural network learns to denoise. DMs stand out not only by high synthesis quality, but also sample diversity, and a simple and robust training objective. This makes them arguably well suited for training under DP perturbations. Moreover, generation in DMs corresponds to an iterative denoising process, breaking the difficult generation task into many small denoising steps that are individually simpler than the one-shot synthesis task performed by GANs and other traditional methods. In particular, the denoising neural network that is learnt in DMs and applied repeatedly at each synthesis step is less complex and smoother than the generator networks of one-shot methods, as we validate in experiments on toy data \TIM{(synthetically generated mixture of 2D Gaussians)}. Therefore, training of the denoising neural network is arguably less sensitive to gradient clipping and noise injection required for DP.\looseness=-1

Based on these observations, we propose \emph{Differentially Private Diffusion Models} (DPDMs), DMs trained with rigorous DP guarantees based on DP-SGD. We thoroughly study the DM parameterization and sampling algorithm, and tailor them to the DP setting. We find that the stochasticity in DM sampling, which is empirically known to be error-correcting~\citep{karras2022elucidating}, can be particularly helpful in DP-SGD training to obtain satisfactory perceptual quality. We also propose \emph{noise multiplicity}, where a single training data sample is re-used for training at multiple perturbation levels along the diffusion process (see \Cref{fig:pipeline}). This simple yet powerful modification of the DM training objective improves learning at no additional privacy cost. We validate DPDMs on standard DP image generation tasks, and achieve state-of-the-art performance by large margins, both in terms of perceptual quality and performance of downstream classifiers trained on synthetically generated data from our models. 
For example, on MNIST we improve the state-of-the-art FID from 56.2 to 23.4 and downstream classification accuracy from 81.5\% to 95.3\% for the privacy setting DP-$(\varepsilon{=}1, \delta{=}10^{-5})$.
We also find that classifiers trained on DPDM-generated synthetic data perform on par with task-specific DP-classifiers trained on real data, which has not been demonstrated before for DP generative models.\looseness=-1

We make the following contributions: \textbf{(i)} We carefully motivate training DMs with DP-SGD and introduce DPDMs, the first DMs trained under DP guarantees. \textbf{(ii)} We study DPDM parameterization, training setting and sampling in detail, and optimize it for the DP setup. \textbf{(iii)} We propose \textit{noise multiplicity} to efficiently boost DPDM performance. \textbf{(iv)} Experimentally, we significantly surpass the state-of-the-art in DP synthesis on widely-studied image modeling benchmarks. \textbf{(v)} We demonstrate for the first time that
classifiers trained on DPDM-generated data perform on par with task-specific DP-trained discriminative models. This implies a very high utility of the synthetic data generated by DPDMs, delivering on the promise of DP generative models as an effective data sharing medium. Finally, we hope that our work has implications for the literature on DMs, which are now routinely trained on ultra large-scale datasets of diverse origins.\looseness=-1

\vspace{-1mm}
\section{Background} \label{sec:background}
\subsection{Diffusion Models} \label{sec:diffusion_models}
We consider continuous-time DMs~\citep{song2020} and follow the presentation of~\citet{karras2022elucidating}. Let $\pdata(\rvx)$ denote the data distribution and $p(\rvx; \sigma)$ the distribution obtained by adding i.i.d. $\sigma^2$-variance Gaussian noise
to the data distribution.
For sufficiently large $\sigma_\mathrm{max}$,  $p(\rvx; \sigma_\mathrm{max}^2)$ is almost indistinguishable from $\sigma^2_\mathrm{max}$-variance Gaussian noise.
Capitalizing on this observation, DMs sample (high variance) Gaussian noise $\rvx_0 \sim\gN\left(\bm{0}, \sigma_\mathrm{max}^2\right)$ and sequentially denoise $\rvx_0$ into $\rvx_i\sim p(\rvx_i; \sigma_i)$, $i \in [0,...,M]$, with $\sigma_{i} < \sigma_{i-1}$ ($\sigma_0=\sigma_\mathrm{max}$). If $\sigma_M=0$, then $\rvx_0$ is distributed according to the data.\looseness=-1

\textbf{Sampling.} In practice, the sequential denoising is often implemented through the simulation of the \emph{Probability Flow} ordinary differential equation (ODE)~\citep{song2020}\looseness=-1
\begin{align} \label{eq:probability_flow_ode}
    d\rvx = -\dot \sigma(t) \sigma(t) \nabla_\rvx \log p(\rvx; \sigma(t)) \, dt,
\end{align}
where $\nabla_\rvx \log p(\rvx; \sigma)$ is the \emph{score function}~\citep{hyvarinen2005scorematching}. The schedule $\sigma(t): [0, 1] \to \R_+$ is user-specified and  $\dot \sigma(t)$ denotes the time derivative of $\sigma(t)$. Alternatively, we may also sample from a 
stochastic differential equation (SDE)~\citep{song2020,karras2022elucidating}:
\begin{align} \label{eq:diffusion_sde}
    d\rvx = \underbrace{- \dot \sigma(t) \sigma(t) \nabla_\rvx \log p(\rvx; \sigma(t)) \, dt}_{\text{Probability Flow ODE; see~\Cref{eq:probability_flow_ode}}} \; 
    \underbrace{-\beta(t) \sigma^2(t) \nabla_\rvx \log p(\rvx; \sigma(t)) \, dt + \sqrt{2 \beta(t)} \sigma(t)\, d\omega_t}_{\text{Langevin diffusion component}},
\end{align}
where $d\omega_t$ is the standard Wiener process. 
In principle, given initial samples $\rvx_0 \sim\gN\left(\bm{0}, \sigma_\mathrm{max}^2\right)$, simulating either Probability Flow ODE or SDE produces samples from the same distribution.
In practice, though, neither ODE nor SDE can be simulated exactly: Firstly, any numerical solver inevitably introduces discretization errors. Secondly, the score function is only accessible through a model $s_\vtheta(\rvx; \sigma)$ that needs to be learned; replacing the score function with an imperfect model also introduces an error. Empirically, the ODE formulation has been used frequently to develop fast solvers~\citep{song2021denoising, zhang2022fast, lu2022dpm,liu2022pseudo,dockhorn2022genie}, whereas the SDE formulation often leads to higher quality samples (while requiring more steps)~\citep{karras2022elucidating}. One possible explanation for the latter observation is that the Langevin diffusion component in the SDE at any time during the synthesis process actively drives the process towards the desired marginal distribution $p(\rvx; \sigma)$, whereas errors accumulate in the ODE formulation, even when using many synthesis steps.
In fact, it has been shown that as the score model $s_\vtheta$ improves, the performance boost that can be obtained by an SDE solver diminishes~\citep{karras2022elucidating}. 
Finally, note that we are using classifier-free guidance~\citep{ho2021classifierfree} to perform class-conditional sampling in this work. For details on classifier-free guidance and the numerical solvers for~\Cref{eq:probability_flow_ode} and~\Cref{eq:diffusion_sde}, we refer to~\Cref{sec:diffusion_solvers}.

\textbf{Training.} 
DM training reduces to learning the score model $s_\vtheta$. The model can, for example, be parameterized as $\nabla_\rvx \log p(\rvx; \sigma) \approx s_\vtheta = (D_\vtheta(\rvx; \sigma) - \rvx)/ \sigma^2$~\citep{karras2022elucidating}, 
where $D_\vtheta$ is a learnable \emph{denoiser} that, given a noisy data point $\rvx + \rvn$, $\rvx \sim \pdata(\rvx)$, $\rvn \sim \gN\left(\bm{0}, \sigma^2\right)$ and conditioned on the noise level $\sigma$, tries to predict the clean 
$\rvx$. The denoiser $D_\vtheta$ can be trained by minimizing an 
$L_2$-loss\looseness=-1
\begin{align} \label{eq:diffusion_objective}
    \E_{\rvx \sim \pdata(\rvx), (\sigma, \rvn) \sim p(\sigma, \rvn)} \left[\lambda_\sigma \|D_\vtheta(\rvx + \rvn, \sigma) - \rvx \|_2^2 \right],
\end{align}
where $p(\sigma, \rvn) = p(\sigma)\,\gN\left(\rvn; \bm{0}, \sigma^2\right)$ and $\lambda_\sigma \colon \R_+ \to \R_+$ is a weighting function. Previous works 
proposed various denoiser models $D_\vtheta$, noise distributions $p(\sigma)$, and weightings $\lambda_\sigma$. We refer to the triplet $(D_\vtheta, p, \lambda)$ as the DM \textit{config}. 
Here, we consider four such configs: \emph{variance preserving} (VP)~\citep{song2020}, \emph{variance exploding} (VE)~\citep{song2020}, $\rvv$-prediction~\citep{salimans2022progressive}, and EDM~\citet{karras2022elucidating}; \Cref{sec:diffusion_backbones} for details.\looseness=-1

\vspace{-1.5mm}
\subsection{Differential Privacy} \label{sec:differential_privacy}
\vspace{-1.5mm}
DP is a rigorous mathematical definition of privacy applied to statistical queries; in our work the queries correspond to the training of a neural network using sensitive training data.
Informally, training is said to be DP, if, given the trained weights $\vtheta$ of the network, an adversary cannot tell with certainty whether a particular data point was part of the training data. This degree of certainty is controlled by two positive parameters $\varepsilon$ and $\delta$: training becomes more private as $\varepsilon$ and $\delta$ decrease. Note, however, that there is an inherent trade-off between utility and privacy: very private models may be of little to no practical use. To guarantee a sufficient amount of privacy, as a rule of thumb, $\delta$ should not be larger than $1/N$, where $N$ is number of training points $\{\rvx_i\}_{i=1}^N$, and $\varepsilon$ should be a small constant. More formally, we refer to $(\varepsilon, \delta)$-DP defined as follows~\citep{dwork2006calibrating}:
\begin{definition}(Differential Privacy) \label{def:dp}
    A randomized mechanism $\mathcal{M}: \mathcal{D} \to \mathcal{R}$ with domain $\mathcal{D}$ and range 
    $\mathcal{R}$ satisfies $(\varepsilon, \delta)$-DP if for any 
    two datasets
    $d, d' \in \mathcal{D}$ differing by at most one entry, and for any subset of outputs $S \subseteq \mathcal{R}$ it holds that
    \vspace{-0.2cm}
    \begin{align}
        \mathbf{Pr}\left[\mathcal{M}(d) \in S \right] \leq e^{\varepsilon} \mathbf{Pr}\left[\mathcal{M}(d') \in S \right] + \delta.
    \end{align}
\end{definition}

\vspace{-0.2cm}
\textbf{DP-SGD.} 
We require a DP algorithm that trains a neural network using sensitive 
data. The workhorse for this particular task is differentially private stochastic gradient descent (DP-SGD)~\citep{abadi2016deep}. DP-SGD is a modification of SGD for which per-sample-gradients are clipped and noise is added to the clipped gradients; 
the DP-SGD parameter updates are defined as follows
\vspace{-0.1cm}
\begin{align} \label{eq:dpsgd_update}
    \vtheta \leftarrow \vtheta - \frac{\eta}{B} \left( \sum\nolimits_{i \in \sB} \texttt{clip}_C\left(\nabla_\vtheta l_i(\vtheta)\right) + C \rvz \right),
\end{align}
where $\rvz \sim \gN(\bm{0}, \sigma_\mathrm{DP}^2 \mI)$, $\sB$ is a $B$-sized subset of $\{1, \dots, N\}$ drawn uniformly at random, $l_i$ is the loss function for data point $\rvx_i$, $\eta$ is the learning rate, and the clipping function is $\texttt{clip}_C(\rvg) = \min\left\{1, C /\|\rvg\|_2\right\} \rvg$. %
DP-SGD can be adapted to other first-order optimizers, such as Adam~\citep{mcmahan2018general}.

\textbf{Privacy Accounting.}
According to the \emph{Gaussian mechanism}~\citep{dwork2014algorithmic}, a single DP-SGD update (\Cref{eq:dpsgd_update}) satisfies ($\varepsilon, \delta$)-DP if $\sigma^2_\mathrm{DP} > 2 \log(1.25 / \delta) C^2 / \varepsilon^2$. Privacy accounting methods can be used to \emph{compose} the privacy cost of multiple DP-SGD training updates and to determine the variance $\sigma^2_\mathrm{DP}$ needed to satisfy ($\varepsilon, \delta$)-DP for a particular number of DP-SGD updates with clipping constant $C$ and subsampling rate $B/N$.
Also see~\Cref{sec:dp_proof}.\looseness=-1

\section{Differentially Private Diffusion Models}
\label{sec:dp_diffusion}
\vspace{-1mm}
We propose DPDMs, DMs trained with rigorous DP guarantees based on DP-SGD~\citep{abadi2016deep}. \TC{DP-SGD is a well-established method to train DP neural networks and our intention is not to re-invent DP-SGD; instead, the novelty in this work lies in the combination of DMs with DP-SGD, modifications of DP-SGD specifically tailored to the DP training of DMs, as well as the study of design choices and training recipes that greatly influence the performance of DPDMs. Combinations of DP-SGD with GANs have been widely studied~\citep{frigerio2019differentially, torkzadehmahani2019dp, xie2018differentially, bie2022private}, motivating a similar line of research for DMs. To the best of our knowledge, we are the first to explore DP training of DMs.} In~\Cref{sec:motivation}, we discuss the motivation for using DMs for DP generative modeling. In~\Cref{sec:dpsgd_training}, we then discuss training and methodological details as well as DM design choices, and we prove that DPDMs satisfy DP.\looseness=-1

\vspace{-2mm}
\subsection{Motivation} \label{sec:motivation}
\vspace{-1mm}
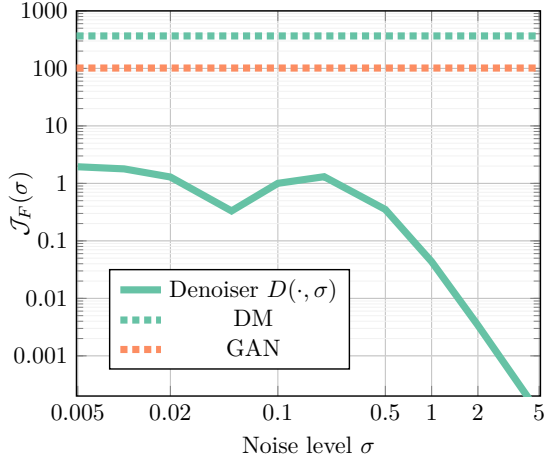
\begin{figure}
\vspace{-3mm}
    \centering
    \begin{tikzpicture}[scale=0.9]
\begin{axis}[, xtick={0.005, 0.02, 0.1, 0.5, 1, 2, 5}, xticklabels={0.005, 0.02, 0.1, 0.5, 1, 2, 5}, ytick={0.001, 0.01, 0.1, 1, 10, 100, 1000}, yticklabels={0.001, 0.01, 0.1, 1, 10, 100, 1000}, ymin=0.0002, ymax=1001, xmin=0.0049, xmax=5.1, xlabel=Noise level $\sigma$, ylabel=$\gJ_F(\sigma)$, grid=both, grid style={line width=.1pt, draw=gray!10}, major grid style={line width=.2pt,draw=gray!50}, every axis plot/.append style={line width=3pt}, every axis plot/.append style={mark size=3pt}, xlabel near ticks, ylabel near ticks, xmode=log, ymode=log, legend style={at={(0.07, 0.33)},anchor=north west}, ylabel shift=-15pt]
\addplot[color=set21] coordinates {
(0.005, 1.9453)
(0.01, 1.7871)
(0.02, 1.2908)
(0.05, 0.3316)
(0.1, 1.0022)
(0.2, 1.3067)
(0.5, 0.3517)
(1.0, 0.0435)
(2.0, 0.0034)
(5.0, 0.0001)
};
\addlegendentry{Denoiser $D(\cdot, \sigma)$}
\addplot[domain=0.0049:5.1, color=set21, dotted] {366.9458};
\addlegendentry{DM}
13.6940
\addplot[domain=0.0049:5.1, color=set22, dotted] {101.4823};
\addlegendentry{GAN}
\end{axis}
\end{tikzpicture}
    \caption{Frobenius norm of the Jacobian $\gJ_F(\sigma)$ of the denoiser $D(\cdot, \sigma)$ and constant Frobenius norms of the Jacobians $\gJ_F$ of the sampling functions defined by the DM and a GAN. \Cref{sec:toy_experiments} for experiment details.
    }
    \label{fig:complexity_input_resnet_small_not_wrapped}
    \vspace{-1mm}
\end{figure}

\textbf{(i) Objective function.}
GANs have so far been the workhorse of DP generative modeling (see~\Cref{sec:related_work}),
even though they are generally difficult to optimize~\citep{arjovsky2017towards, pmlr-v80-mescheder18a} due to their adversarial training and propensity to mode collapse. Both phenomena may be amplified during DP-SGD training. DMs, in contrast, have been shown to produce outputs as good or even better than GANs'~\citep{dhariwal2021diffusion}, while being trained with a very simple regression-like $L_2$-loss (\Cref{eq:diffusion_objective}), which makes them robust and scalable in practice. DMs are therefore arguably also well-suited for DP-SGD-based training and offer better stability under gradient clipping and noising than adversarial training frameworks.

\textbf{(ii) Sequential denoising.} In GANs and most other traditional generative modeling approaches, the generator directly learns the sampling function, i.e., the mapping of latents to synthesized samples, end-to-end.
In contrast, the sampling function in DMs is defined through a sequential denoising process, breaking the difficult generation task into many small denoising steps which are individually less complex than the one-shot synthesis
task performed by, for instance, a GAN generator. The denoiser neural network, the learnable component in DMs that is evaluated once per denoising step, is therefore simpler and smoother than the one-shot generator networks of other methods. We fit both a DM and a GAN to a two-dimensional toy distribution (mixture of Gaussians, see~\Cref{sec:toy_experiments}) and empirically verify that the denoiser $D$ is indeed significantly less complex (quantified by the Frobenius norm of the Jacobian) than the generator learnt by the GAN and also than the end-to-end multi-step synthesis process (Probability Flow ODE) of the DM (see~\Cref{fig:complexity_input_resnet_small_not_wrapped}; we calculate denoiser $\gJ_F(\sigma)$ at varying noise levels $\sigma$).
Generally, more complex functions require larger neural networks and are more difficult to learn. 
\TIM{In DP-SGD training we only have a limited number of training iterations available until the privacy budget is depleted. Consequently, the fact that DMs require less complexity out of their neural networks 
than typical one-shot generation methods, while still being able to represent expressive generative models due to the iterative synthesis process, makes them likely well-suited for DP generative modeling with DP-SGD.}\looseness=-1

\textbf{(iii) Stochastic diffusion model sampling.} As discussed in~\Cref{sec:diffusion_models}, generating samples from DMs with stochastic sampling can perform better than deterministic sampling when the score model is not learned well. Since we replace gradient estimates in DP-SGD training with biased large variance estimators, we cannot expect a perfectly accurate score model. In~\Cref{sec:ablation_studies}, we empirically show that stochastic sampling can in fact boost perceptual synthesis quality in DPDMs as measured by FID.\looseness=-1

\vspace{-2mm}
\subsection{Training Details, Design Choices, Privacy} \label{sec:dpsgd_training}
\vspace{-2mm}
The clipping and noising of the gradient estimates in DP-SGD (\Cref{eq:dpsgd_update}) pose a major challenge for efficient optimization. 
\TIM{Blindly reducing the added noise could be fatal, as it decreases the number of training iterations allowed within a certain ($\varepsilon, \delta$)-DP budget.}
Furthermore, as discussed the $L_2$-norm of the noise added in DP-SGD scales linearly to the number of parameters. Consequently, settings that work well for non-private DMs, such as relatively small batch sizes, a large number of training iterations, and heavily overparameterized models, may not work well for DPDMs. Below, we discuss how we propose to adjust DPDMs for successful DP-SGD training.\looseness=-1

\textbf{Noise multiplicity.} 
Recall that the DM objective in~\Cref{eq:diffusion_objective} involves three expectations. As usual, the expectation with respect to the data distribution $\pdata(\rvx)$ is approximated using mini-batching. %
For non-private DMs, the expectations over $\sigma$ and $\rvn$ are generally approximated using a single Monte Carlo sample $(\sigma_i, \rvn_i) \sim p(\sigma) \gN\left(\bm{0}, \sigma^2\right)$ per data point $\rvx_i$, resulting in the loss for training sample $i$
\begin{align} \label{eq:noisy_objective}
    l_i =\lambda(\sigma_i) \|D_\vtheta(\rvx_i + \rvn_i, \sigma_i) - \rvx_i \|_2^2.
\end{align}
The estimator $l_i$ is very noisy in practice. Non-private DMs counteract this by training for a large number of iterations in combination with an exponential moving average (EMA) of the trainable parameters $\vtheta$~\citep{song2020improved}. When training DMs with DP-SGD, we incur a privacy cost for each iteration, and therefore prefer a small number of iterations. Furthermore, since the per-example gradient clipping as well as the noise injection induce additional variance, we would like our objective function to be less noisy than in the non-DP case. We achieve this by estimating the expectation over $\sigma$ and $\rvn$ using an average over $K$ noise samples, $\{(\sigma_{ik}, \rvn_{ik})\}_{k =1}^K \sim p(\sigma) \gN\left(\bm{0}, \sigma^2\right)$ for each data point $\rvx_i$, replacing the non-private DM objective $l_i$ in~\Cref{eq:noisy_objective} with\looseness=-1
\begin{align} \label{eq:noise_mult}
    \tilde l_i = \frac{1}{K} \sum\nolimits_{k =1}^K\lambda(\sigma_{ik}) \|D_\vtheta(\rvx_i + \rvn_{ik}, \sigma_{ik}) - \rvx_i \|_2^2.
\end{align}
Importantly, we show that this modification comes at \emph{no} additional privacy cost (also see \Cref{sec:dp_proof}). We call this simple yet powerful modification of the DM objective, which is tailored to the DP setup, \emph{noise multiplicity}. 
\TC{
\begin{theorem} \label{th:noise_mult}
    \textit{The variance of the DM objective (\Cref{eq:noise_mult}) decreases with increased noise multiplicity $K$ as $1/K$.}
\end{theorem}
}
Proof in~\Cref{sec:variance_reduction_via_noise_multiplicity}. Intuitively, the key is that we first create a relatively accurate low-variance gradient estimate by averaging over multiple noise samples before performing gradient sanitization in the backward pass via clipping and noising. 
This averaging process increases computational cost, but provides better utility at the same privacy budget, which is the main bottleneck in DP generative modeling; see \Cref{sec:noise_multiplicity_cost} for further discussion.
\begin{figure}
\vspace{-3mm}
    \centering
    \includegraphics[width=0.5\textwidth]{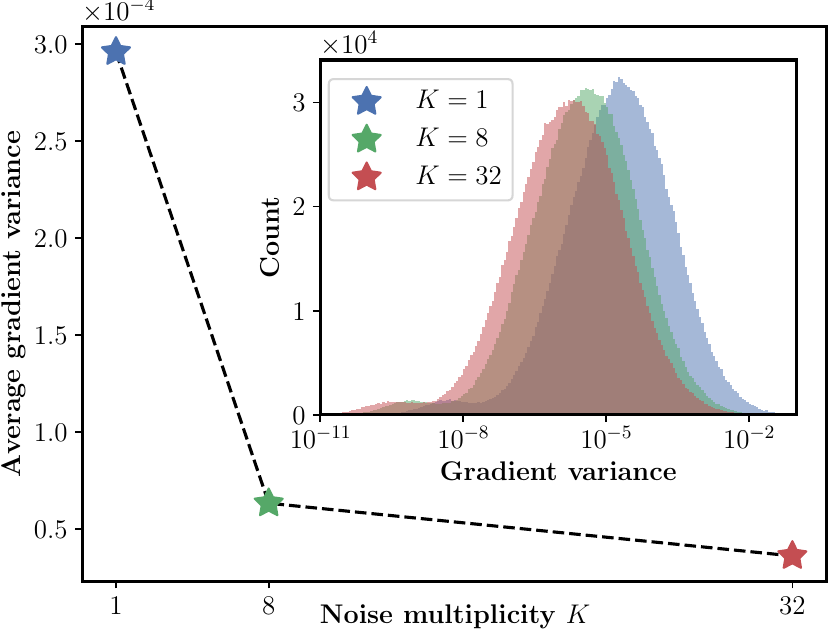}
    \caption{
    Increasing $K$ in \emph{noise multiplicity} leads to significant variance reduction of parameter gradient estimates during training (note logarithmic axis in inset). Enlarged version in~\Cref{fig:variance_reduction}.\looseness=-1
    }
    \vspace{-1mm}
    \label{fig:variance_reduction_small}
\end{figure}
We empirically showcase in~\Cref{sec:ablation_studies} that the variance reduction induced by noise multiplicity is a key factor in training strong DPDMs.
\TC{In~\Cref{fig:variance_reduction_small}, we show that the reduction of variance in the DM objective also empirically leads to lower variance gradient estimates (see~\Cref{sec:variance_reduction_via_noise_multiplicity} for experiment details). The noise multiplicity mechanism is also highlighted in~\Cref{fig:pipeline}: the figure describes the information flow during training for a single training sample (i.e., batch size $B=1$).  
Note that noise multiplicity is loosely inspired by \emph{augmentation multiplicity}~\citep{de2022unlocking}, a technique where multiple augmentations per image are used to train \emph{classifiers} with DP-SGD. In contrast to augmentation multiplicity, our novel noise multiplicity is carefully designed specifically for DPDMs and comes with theoretical proofs on its variance reduction.} \TIM{The reader may find a more detailed discussion on the difference between noise multiplicity and data multiplicity (for DPDMs) in~\Cref{sec:noise_multiplicity_vs_data_aug}.}

\textbf{Neural networks sizes.}
Current DMs are heavily overparameterized: For example, the current state-of-the-art image generation model (in terms of perceptual quality) on CIFAR-10 uses more than 100M parameters, despite the dataset consisting of only 50k training points~\citep{karras2022elucidating}. 
The per-example clipping operation of DP-SGD requires the computation of the loss gradient on each training example $\nabla_\vtheta \tilde l_i$, rather than the minibatch gradient. 
\TIM{In theory, this increases the memory footprint by at least $O(B)$; however, in common DP frameworks, such as Opacus~\citep{yousefpour2021opacus}, which we use, the peak memory requirement is $\mathcal{O}(B^2)$ compared to non-private training (recent methods such as \emph{ghost clipping}~\citep{bu2022scalable} require less memory, but are not widely implemented)}
On top of that, DP-SGD generally already relies on a significantly increased batch size, when compared to non-private training, to improve the privacy-utility trade-off. As a result, we train very small neural networks for DPDMs, when compared to their non-DP counterparts: our models on MNIST/Fashion-MNIST and CelebA have 1.75M and 1.80M parameters, respectively. \TIM{Furthermore, we found smaller models to perform better across our experiments which may be due to the $L_2$-norm of the noise added in our DP-SGD update scaling linearly with the number of parameters. This is in contrast to recent works in supervised DP learning, which show that larger models may perform better than smaller models~\citep{de2022unlocking, li2022large, anil2021large, li2022does}.}

\begin{wrapfigure}{r}{0.35\textwidth}
    \centering
    \includegraphics[width=0.35\textwidth]{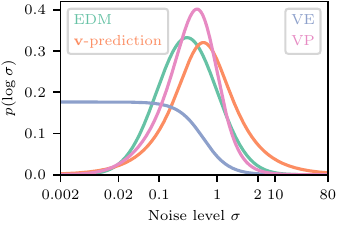}
    \caption{\small Noise level sampling for different DM configs; see~\Cref{sec:diffusion_backbones}. 
    }
    \label{fig:viz_dist}
\end{wrapfigure}
\textbf{Diffusion model config.} In addition to network size, we found the choice of DM config, i.e., %
denoiser parameterization $D_\vtheta$, weighting function $\lambda(\sigma)$, and noise distribution $p(\sigma)$, to be important. In particular the latter is crucial to obtain strong results with DPDMs. In~\Cref{fig:viz_dist}, we visualize the noise distributions of the four configs under consideration. 
We follow~\citet{karras2022elucidating} and plot the distribution $p(\log \sigma)$ over the log-noise level. Especially for high privacy settings (small $\varepsilon$), we found it important to use distributions that give sufficiently much weight to larger $\sigma$, such as the distribution of $\rvv$-prediction~\citep{salimans2022progressive}.
It is known that at large $\sigma$ the DM learns the global, coarse structure of the data, i.e., the low frequency content in the data (images, in our case). Learning global structure reasonably well is crucial to form visually coherent images that 
can also be used to train downstream models. This is relatively easy to achieve in the non-DP setting, due to the heavily smoothed diffused distribution at these high noise level. At high privacy levels, however, even training at such high noise levels can be challenging due to DP-SGD's gradient clipping and noising. We hypothesize that this is why it is beneficial to give relatively more weight to high noise levels when training in the DP setting.
In~\Cref{sec:ablation_studies}, we empirically demonstrate the importance of the right choice of the DM config.\looseness=-1

\setlength{\intextsep}{3pt}
\begin{figure}
\vspace{-3mm}
    \begin{algorithm}[H]
    \small
    \caption{{DPDM Training}
    }
    \begin{algorithmic} \label{algo:dpdm}
    
    \STATE {\bfseries Input:} Private data set $d = \{\rvx_{j}\}_{j=1}^N$, subsampling rate $B/N$, DP noise scale $\sigma_{\mathrm{DP}}$, clipping constant $C$, sampling function \textit{Poisson Sample} (\Cref{algo:poisson}), denoiser $D_\mathrm{\vtheta}$ with initial parameters $\vtheta$, noise distribution $p(\sigma)$, learning rate $\eta$, total steps $T$, noise multiplicity $K$, \textit{Adam}~\citep{kingma2015adam} optimizer
    \STATE {\bfseries Output:} Trained parameters $\vtheta$
    \FOR{$t=1$ {\bfseries to} $T$}
    \STATE $\sB \sim \textit{Poisson Sample}(N, B/N)$
    \FOR{$i \in \sB$}
        \STATE $\left\{(\sigma_{ik}, \mathrm{n}_{ik})\right \}_{k=1}^K \sim p(\sigma)\mathcal{N}(\bm{0},\sigma^2)$
        \STATE $\tilde{l}_{i} = \frac{1}{K} \sum_{k =1}^K\lambda(\sigma_{ik}) \|D_\vtheta(\rvx_i{+}\rvn_{ik}, \sigma_{ik}){-}\rvx_i \|_2^2$
    \ENDFOR
    \STATE $G_{batch} = \frac{1}{B} \sum_{i \in \sB} \texttt{clip}_C\left(\nabla_\vtheta \tilde{l}_i \right)$
    \STATE $\tilde{G}_{batch} = G_{batch} + (C/B) \rvz, \rvz \sim \gN(\bm{0}, \sigma_\mathrm{DP}^2)$
    \STATE $\vtheta = \vtheta - \eta * \textit{Adam}(\tilde{G}_{batch})$
    \ENDFOR
    \end{algorithmic}
    \end{algorithm}
\vspace{-3mm}
\end{figure}
\setlength{\intextsep}{12pt}
\textbf{DP-SGD settings.} Following~\citet{de2022unlocking} we use very large batch sizes: 4096 on MNIST/Fashion-MNIST and 2048 on CelebA. Similar to previous works~\citep{de2022unlocking, kurakin2022toward, li2022large}, we found that small clipping constants $C$ work better than larger clipping norms; in particular, we found $C=1$ to work well across our experiments. Decreasing $C$ even further had little effect; in contrast, increasing $C$ significantly worsened performance. Similar to non-private DMs, we use an EMA of the learnable parameters $\vtheta$. Incidentally, this has recently been reported to also have a positive effect on DP-SGD training of classifiers by~\citet{de2022unlocking}.\looseness=-1

\textbf{Privacy.}
We formulate privacy protection under the Rényi Differential Privacy (RDP)~\citep{mironov2017renyi} framework (see~\Cref{def:rdp}), which can be converted to $(\epsilon, \delta)$-DP.
For an algorithm for DPDM training with noise multiplicity see \Cref{algo:dpdm}. 
For the sake of completeness we also formally prove the DP of DPDMs (DP of releasing sanitized training gradients $\tilde{G}_{batch}$):\looseness=-1
\vspace{-1mm}
\begin{theorem}\label{thm:dp} \textit{For noise magnitude $\sigma_\mathrm{DP}$, releasing $\tilde{G}_{batch}$ in \Cref{algo:dpdm} satisfies $\left( \alpha, \alpha / 2 \sigma_\mathrm{DP}^2 \right)$-RDP.}
\end{theorem}
\vspace{-3mm}
The proof can be found in \Cref{sec:dp_proof}. Note that the strength of DP protection is independent of the noise multiplicity, as discussed above. 
In practice, we construct mini-batches by \emph{Poisson Sampling} (See \Cref{algo:poisson}) the training dataset for privacy amplification via sub-sampling~\citep{mironov2019r}, and compute the overall privacy cost of training DPDM via RDP composition~\citep{mironov2017renyi}. \TIM{Tighter privacy bounds, such as the one developed in~\citet{gopi2021numerical}, may lead to better results but are not widely implemented (not in Opacus~\citep{yousefpour2021opacus}, the DP-SGD library we use).}

\vspace{-1mm}
\section{Related Work} \label{sec:related_work}
\vspace{-1mm}
In the DP generative learning literature, several works~\citep{xie2018differentially, frigerio2019differentially, torkzadehmahani2019dp, chen2020gs} have explored applying DP-SGD~\citep{abadi2016deep} to GANs, while others~\citep{yoon2018pategan, long2019scalable, wang2021datalens} train GANs under the PATE~\citep{papernot2018scalable} framework, which distills private teacher models (discriminators) into a public student (generator) model. Apart from GANs,~
\citet{acs2018dpvae} train variational autoencoders on DP-sanitized data clusters, and~\citet{cao2021don} 
use the Sinkhorn divergence and DP-SGD.\looseness=-1

DP-MERF~\citep{harder2021dp} was the first work to perform one-shot privatization on the data, followed by non-private learning. It uses differentially private random Fourier features to construct a Maximum Mean Discrepancy loss, which is then minimized by a generative model. PEARL~\citep{liew2022pearl} instead minimizes an empirical characteristic function, also based on Fourier features.
DP-MEPF~\citep{harder2022differentially} extends DP-MERF to the mixed public-private setting with pre-trained feature extractors. While 
these approaches are efficient in the high-privacy/small dataset regime, they are limited in expressivity by the data statistics that can be extracted during one-shot privatization. As a result, the performance of these methods does not scale well in the low-privacy/large dataset regime.\looseness=-1

In our experimental comparisons, we excluded \cite{takagi2020p3gm} and \cite{chen2022dpgen} due to concerns regarding their privacy guarantees. The privacy analysis of \cite{takagi2020p3gm} relies on the Wishart mechanism, which has been retracted due to privacy leakage~\citep{sarwate2017wishart}. \cite{chen2022dpgen} attempt to train a score-based model while guaranteeing differential privacy through a data-dependent randomized response mechanism. In \Cref{sec:app_dpgen_failure}, we prove why their proposed mechanism leaks privacy, and further discuss other sources of privacy leakage. \looseness=-1 %

Our DPDM relies on DP-SGD~\citep{abadi2016deep} to enforce DP guarantees. 
DP-SGD has also been used to train DP classifiers~\citep{dormann2021not,tramer2021differentially,kurakin2022toward}. Recently, \citet{de2022unlocking} demonstrated how to train very large discriminative models with DP-SGD and proposed augmentation multiplicity, which 
is related to our noise multiplicity, as discussed in \Cref{sec:dpsgd_training}. Furthrmore, DP-SGD has been utilized to train and fine-tune large language models~\citep{anil2021large,li2022large,yu2022differentially}, to protect sensitive training data in the medical domain~\citep{ziller2021medical,ziller2021differentially,balelli2022differentially}, and to obscure geo-spatial location information~\citep{zeighami2021neural}.\looseness=-1

Our work 
builds on
DMs and score-based generative models~\citep{sohl2015,song2020,ho2020}. DMs have been used prominently for image synthesis~\citep{ho2021arxiv,nichol21,dhariwal2021diffusion,rombach2021highresolution,ramesh2022dalle2,saharia2022imagen,balaji2022eDiffi} and other image modeling tasks~\citep{meng2021sdedit,saharia2021palette,saharia2021image,li2021srdiff,sasaki2021unitddpm,kawar2022restoration}. They have also found applications in other areas, for instance in audio and speech generation~\citep{chen2021wavegrad,kong2021diffwave,jeong2021difftts}, video generation~\citep{ho2022video,ho2022imagenvideo,singer2023makeavideo,blattmann2023videoldm} and 3D synthesis~\citep{luo2021diffusion,zhou20213d,zeng2022lion,kim2023nfldm}.
Methodologically, DMs have been adapted, for example, for fast sampling~\citep{jolicoeur2021gotta,song2021denoising,salimans2022progressive,dockhorn2022scorebased,xiao2022tackling,watson2022learning,dockhorn2022genie} and maximum likelihood training~\citep{song2021maximum,kingma2021variational,vahdat2021score}. To the best of our knowledge, we are the first to train DMs under differential privacy guarantees.\looseness=-1
\section{Experiments} \label{sec:experiments}
\TC{In this section, we present results of DPDM on standard image synthesis benchmarks. Importantly, note that all models are private \emph{by construction} through training with DP-SGD. The privacy guarantee is given by the $(\varepsilon, \delta)$ parameters of DP-SGD, clearly stated for each experiment below.}

\begin{table}
\caption{\small{Class-conditional DP image generation performance (MNIST \& Fashion-MNIST). For PEARL~\citep{liew2022pearl}, we train models and compute metrics ourselves (\Cref{sec:metrics_baselines_datasets}). All other results taken from the literature. DP-MEPF (\textdagger) uses additional public data for training (only included for completeness).}}
    \label{tab:main_no_public_data}
    \centering
    \scalebox{0.8}{
        \begin{tabular}{l c c c c c c c c c}
            \toprule
            \multirow{3}{*}{Method} & \multirow{3}{*}{DP-$\varepsilon$} & \multicolumn{4}{c}{MNIST} & \multicolumn{4}{c}{Fashion-MNIST} \\
            \cmidrule(l{1em}r{1em}){3-6} \cmidrule(l{1em}r{1em}){7-10}
            & & \multirow{2}{*}{FID} & \multicolumn{3}{c}{Acc (\%)} & \multirow{2}{*}{FID} & \multicolumn{3}{c}{Acc (\%)} \\
            \cmidrule(l{1em}r{1em}){4-6} \cmidrule(l{1em}r{1em}){8-10}
            & & & Log Reg & MLP & CNN & & Log Reg & MLP & CNN \\
            \midrule 
            DPDM (FID) (\emph{ours}) & 0.2 & \textbf{61.9} & 65.3 & 65.8 & 71.9 & \textbf{78.4} & 53.6 & 55.3 & 57.0 \\
            DPDM (Acc) (\emph{ours}) & 0.2 & 104 & \textbf{81.0} & \textbf{81.7} & \textbf{86.3} & 128 & \textbf{70.4} & \textbf{71.3} & \textbf{72.3} \\
            PEARL~\citep{liew2022pearl} & 0.2 & 133 & 76.2 & 77.1 & 77.6 & 160 & 70.0 & 70.8 & 68.0 \\
            \midrule 
            DPDM (FID) (\emph{ours}) & 1 & \textbf{23.4} & 83.8 & 87.0 & 93.4 & \textbf{37.8} & 71.5 & 71.7 & 73.6 \\
            DPDM (Acc) (\emph{ours}) & 1 & 35.5 & \textbf{86.7} & \textbf{91.6} & \textbf{95.3} & 51.4 & \textbf{76.3} & \textbf{76.9} & \textbf{79.4} \\
            PEARL~\citep{liew2022pearl} & 1 & 121 & 76.0 & 79.6 & 78.2 & 109 & 74.4 & 74.0 & 68.3 \\
            DPGANr~\citep{bie2022private} & 1 & 56.2 & - & - & 80.1 & 121.8 & - & - & 68.0 \\
            DP-HP~\citep{vinaroz2022hermite} & 1 & - & - & - & 81.5 & - & - & - & 72.3 \\
            \midrule
            DPDM (FID) (\emph{ours}) & 10 & \textbf{5.01} & 90.5 & 94.6 & 97.3 & \textbf{18.6} & 80.4 & 81.1 & 84.9 \\
            DPDM (Acc) (\emph{ours}) & 10 & 6.65 & \textbf{90.8} & \textbf{94.8} & \textbf{98.1} & 19.1 & \textbf{81.1} & \textbf{83.0} & \textbf{86.2} \\
            PEARL~\citep{liew2022pearl} & 10 & 116 & 76.5 & 78.3 & 78.8 & 102 & 72.6 & 73.2 & 64.9 \\
            DPGANr~\citep{bie2022private} & 10 & 13.0 & - & - & 95.0 & 56.8 & - & - & 74.8 \\
            DP-Sinkhorn~\citep{cao2021don} & 10 & 48.4 & 82.8 & 82.7 & 83.2 & 128.3 & 75.1 & 74.6 & 71.1 \\
            G-PATE~\citep{long2019scalable} & 10 & 150.62 & - & - & 80.92 & 171.90 & - & - & 69.34 \\
            DP-CGAN~\citep{torkzadehmahani2019dp} & 10 & 179.2 & 60 & 60 & 63 & 243.8 & 51 & 50 & 46 \\
            DataLens~\citep{wang2021datalens} & 10 & 173.5 & - & - & 80.66 & 167.7 & - & - & 70.61 \\
            DP-MERF~\citep{harder2021dp} & 10 & 116.3 & 79.4 & 78.3 & 82.1 & 132.6 & 75.5 & 74.5 & 75.4 \\
            GS-WGAN~\citep{chen2020gs} & 10 & 61.3 & 79 & 79 & 80 & 131.3 & 68 & 65 & 65 \\
            \midrule
            \midrule
            DP-MEPF $(\phi_1)$~\citep{harder2022differentially} (\textdagger) & 0.2 & - & 72.1 & 77.1 & - & - & 71.7 & 69.0 & - \\
            DP-MEPF $(\phi_1, \phi_2)$~\citep{harder2022differentially} (\textdagger) & 0.2 & - & 75.8 & 79.9 & - & - & 72.5 & 70.4 & - \\
            \midrule
            DP-MEPF $(\phi_1)$~\citep{harder2022differentially} (\textdagger) & 1 & - & 79.0 & 87.5 & - & - & 76.2 & 75.0 & - \\
            DP-MEPF $(\phi_1, \phi_2)$~\citep{harder2022differentially} (\textdagger) & 1 & - & 82.5 & 89.3 & - & - & 75.4 & 74.7 & - \\
            \midrule
            DP-MEPF $(\phi_1)$~\citep{harder2022differentially} (\textdagger) & 10 & - & 80.8 & 88.8 & - & - & 75.5 & 75.5 & - \\
            DP-MEPF $(\phi_1, \phi_2)$~\citep{harder2022differentially} (\textdagger) & 10 & - & 83.4 & 89.8 & - & - & 75.7 & 76.0 & - \\
            \bottomrule
        \end{tabular}
    }
\end{table}
\begin{figure}
    \centering
    \centering
    \scalebox{0.5}{
    \includegraphics[trim=0 0 0 40,clip,width=\textwidth]{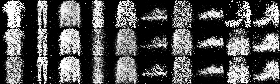}
    }
    \scalebox{0.5}{
    \includegraphics[trim=0 0 0 182,clip,width=\textwidth]{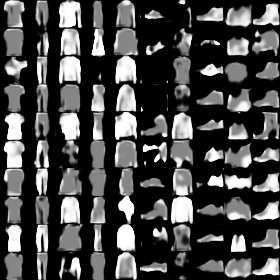}
    }
    \scalebox{0.5}{
    \includegraphics[width=\textwidth]{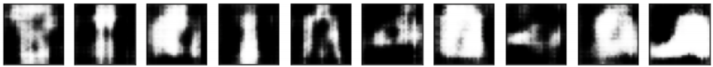}
    }
    \scalebox{0.5}{
    \includegraphics[trim=0 60 0 60,clip,width=\textwidth]{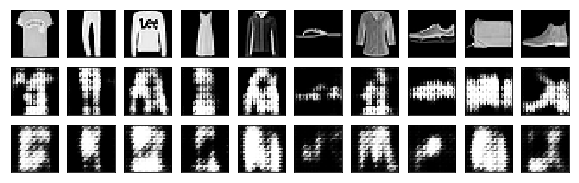}
    }
    \scalebox{0.5}{
    \includegraphics[trim=0 0 0 40,clip,width=\textwidth]{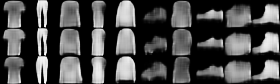}
    }
    \scalebox{0.5}{
    \includegraphics[trim=0 0 0 270px,clip,width=\textwidth]{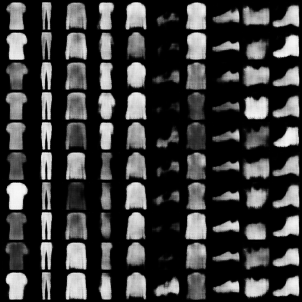}
    }
    \scalebox{0.5}{
    \includegraphics[width=\textwidth]{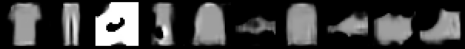}
    }
    \scalebox{0.5}{
    \includegraphics[width=\textwidth]{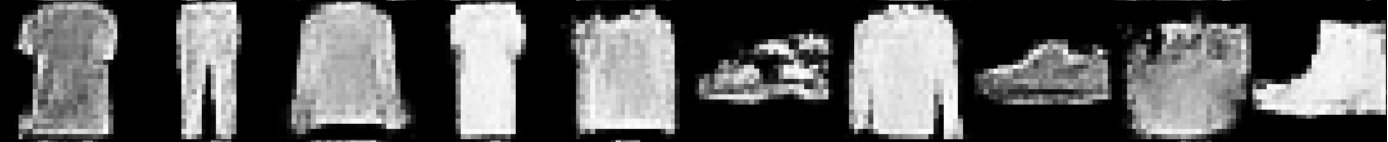}
    }
    \rule[4pt]{0.6\textwidth}{1pt}
    \scalebox{0.5}{
    \includegraphics[width=\textwidth]{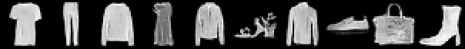}
    }
    \caption{Fashion-MNIST images generated by, from top to bottom, DP-CGAN~\citep{torkzadehmahani2019dp}, DP-MERF~\citep{harder2022differentially}, Datalens~\citep{wang2021datalens}, G-PATE~\citep{long2019scalable}, GS-WGAN~\citep{chen2020gs}, DP-Sinkhorn~\citep{cao2021don}, PEARL~\citep{liew2022pearl}, DPGANr~\citep{bie2022private} (all above bar), and our DPDM (below bar) using the privacy budget $\varepsilon{=}10$. See \Cref{sec:extended_qualitative_results} for more samples.\looseness=-1}
    \label{fig:all_samples}
\end{figure}
\vspace{-0.5mm}
\textbf{Datasets.} We focus on image synthesis and use MNIST~\citep{lecun2010mnist}, Fashion-MNIST~\citep{xiao2017fashion} (28x28), and CelebA~\citep{liu2015deep} (\TO{center-cropped; }downsampled to 32x32). These datasets are standard benchmarks in the DP generative modeling literature. \TC{In~\Cref{sec:add_experiments}, we consider more challenging datasets and provide initial results.}\looseness=-1

\textbf{Architectures.} We implement the neural networks of DPDMs using the DDPM++ architecture~\citep{song2020}.
See~\Cref{sec:model_architecture} for details.\looseness=-1 

\textbf{Evaluation.} We measure sample quality via Fréchet Inception Distance (FID)~\citep{heusel2017gans}. On MNIST and Fashion-MNIST, we also assess utility of class-labeled generated data by training classifiers on synthesized samples and compute class prediction accuracy on real data. As is standard practice, we consider logistic regression (Log Reg), MLP, and CNN classifiers;
see~\Cref{sec:metrics_baselines_datasets} for details.\looseness=-1

\textbf{Sampling.} We sample from DPDM using (stochastic) DDIM~\citep{song2020} and the Churn sampler introduced in~\citep{karras2022elucidating}. See~\Cref{sec:diffusion_solvers} for details.\looseness=-1

\textbf{Privacy implementation:} We implement DPDMs in PyTorch~\citep{paszke2019pytorch} and use Opacus~\citep{yousefpour2021opacus}, a DP-SGD library 
in
PyTorch, for training and privacy accounting. We use $\delta{=}10^{-5}$ for MNIST and Fashion-MNIST, and $\delta{=}10^{-6}$ for CelebA. These values are standard
~\citep{cao2021don} and chosen such that $\delta$ is smaller than the reciprocal of the number of training images. Similar to existing DP generative modeling work, we do not account for the (small) privacy cost of hyperparameter tuning. However, training and sampling is very robust with regards to hyperparameters, which makes DPDMs an ideal candidate for real privacy-critical situations; see~\Cref{sec:dpdm_hyperparameters}.\looseness=-1

\begin{table}
    \centering
    \caption{Class prediction accuracy on real test data. DP-SGD: Classifiers trained directly with DP-SGD and real training data. DPDM: Classifiers trained non-privately on synthesized data from DP-SGD-trained DPDMs \TIM{(using 60,000 samples, following \cite{cao2021don})}.\looseness=-1
        }
    \scalebox{0.80}{
        \begin{tabular}{c c c c c c c c c c c c c}
            \toprule
            \multirow{3}{*}{DP-$\varepsilon$} & \multicolumn{6}{c}{MNIST} & \multicolumn{6}{c}{Fashion-MNIST} \\
            \cmidrule(l{1em}r{1em}){2-7} \cmidrule(l{1em}r{1em}){8-13}
            & \multicolumn{2}{c}{Log Reg} & \multicolumn{2}{c}{MLP} & \multicolumn{2}{c}{CNN} & \multicolumn{2}{c}{Log Reg} & \multicolumn{2}{c}{MLP} & \multicolumn{2}{c}{CNN} \\
            \cmidrule(l{1em}r{1em}){2-3} \cmidrule(l{1em}r{1em}){4-5} \cmidrule(l{1em}r{1em}){6-7} \cmidrule(l{1em}r{1em}){8-9} \cmidrule(l{1em}r{1em}){10-11} \cmidrule(l{1em}r{1em}){12-13}
            & DP-SGD & DPDM & DP-SGD & DPDM & DP-SGD & DPDM & DP-SGD & DPDM & DP-SGD & DPDM & DP-SGD & DPDM \\
            \midrule 
            0.2 & \textbf{83.8} & 81.0 & \textbf{82.0} & 81.7 & 69.9 & \textbf{86.3} & \textbf{74.8} & 70.4 & \textbf{73.9} & 71.3 & 59.5 & \textbf{72.3} \\
            1   & \textbf{89.1} & 86.7 & 89.6 & \textbf{91.6} & 88.2 & \textbf{95.3} & \textbf{79.6} & 76.3 & \textbf{79.6} & 76.9 & 70.5 & \textbf{79.4} \\
            10  & \textbf{91.6} & 90.8 & 92.9 & \textbf{94.8} & 96.4 & \textbf{98.1} & \textbf{83.3} & 81.1 & \textbf{83.9} & 83.0 & 77.1 & \textbf{86.2} \\
            \bottomrule
        \end{tabular}
    }
    \label{tab:cnn_classifier_new}
    \vspace{-2mm}
\end{table}

\begin{table}
    \centering
    \caption{DM config ablation on MNIST for $\varepsilon{=}0.2$. See~\Cref{tab:design_choice} for extended results.\looseness=-1}
    \scalebox{1.0}{
    \begin{tabular}{l c c}
        \toprule
        DM config & FID & CNN-Acc (\%) \\
        \midrule 
        \footnotesize{VP~\citep{song2020}} & 197 & 24.2  \\
        \footnotesize{VE~\citep{song2020}} & 171 & 13.9  \\
        \footnotesize{$\rvv$-prediction~\citep{salimans2022progressive}} & \textbf{97.8} & \textbf{84.4} \\
        \footnotesize{EDM~\citep{karras2022elucidating}} & 119 & 49.2 \\
        \bottomrule
    \end{tabular}
    }
        \label{tab:design_choice_eps_0_2_mnist_fid_cnn_acc}
\end{table}

\subsection{Main Results}
\textbf{Class-conditional gray scale image generation.} 
For MNIST and Fashion-MNIST, we train models for three privacy settings: $\varepsilon{=}\{0.2, 1, 10\}$ (\Cref{tab:main_no_public_data}). Informally, the three settings provide high, moderate, and low amounts of privacy, respectively. The DPDMs use the $\rvv$-prediction DM config~\citep{salimans2022progressive} for $\varepsilon{=}0.2$ and the EDM config~\citep{karras2022elucidating} for $\varepsilon{=}\{1, 10\}$; see~\Cref{sec:ablation_studies}.
We use the Churn sampler~\citep{karras2022elucidating}:
the two settings (FID) and (Acc) are based on the same DM, differing only in sampler setting; see~\Cref{tab:sampler_settings_for_best_models_fid} and~\Cref{tab:sampler_settings_for_best_models_acc} for all sampler settings.\looseness=-1

DPDMs outperform all other existing models for all privacy settings and all metrics by large margins (see~\Cref{tab:main_no_public_data}). 
Interestingly, DPDM also outperforms DP-MEPF~\citep{harder2022differentially}, a method which is trained on additional public data, in 22 out of 24 setups. Generated samples for $\varepsilon{=}10$ are shown in~\Cref{fig:all_samples}. Visually, DPDM's samples appear to be of significantly higher quality than the baselines'.\looseness=-1

\textbf{Comparison to DP-SGD-trained classifiers.} Is it better to train a task-specific private classifier with DP-SGD directly, or can a non-private classifier trained on DPDM's synthethized data perform as well on downstream tasks? To answer this question, we train private classifiers with DP-SGD on real (training) data and compare them to our classifiers learnt using DPDM-synthesized data (details in~\Cref{sec:training_dp_cnn_classifiers}). For a fair comparison, we are using the same architectures that we have already been using in our main experiments to quantify downstream classification accuracy (results in \Cref{tab:cnn_classifier_new}; we test on real (test) data).
While direct DP-SGD training on real data outperforms the DPDM downstream classifier for logistic regression in all six setups (in line with empirical findings that it is easier to train classifiers with few parameters than large ones with DP-SGD~\citep{tramer2021differentially}), 
CNN classifiers trained on DPDM's synthetic data generally outperform DP-SGD-trained classifiers.
These results imply a very high utility of the synthetic data generated by DPDMs, demonstrating that DPDMs can potentially be used as an effective, privacy-preserving data sharing medium in practice. In fact, this approach is beneficial over training task-specific models with DP-SGD, because a user can generate as much data from DPDMs as they desire for various downstream applications without further privacy implications. %
To the best of our knowledge, it has not been demonstrated before in the DP generative modeling literature that image data generated by DP generative models can be used to train discriminative models on-par with directly DP-SGD-trained task-specific models.
\looseness=-1

\begin{table}
\centering
\caption{Unconditional CelebA generative performance. G-PATE and DataLens (\textdagger) use $\delta=10^{-5}$ (less privacy) and model images at 64x64 resolution.}
\scalebox{1.0}{
        \begin{tabular}{l c c}
            \toprule
            Method & DP-$\varepsilon$ & FID \\
            \midrule
            DPDM (\emph{ours}) & 1 & 71.8 \\
            \midrule
            DPDM (\emph{ours}) & 10 & \textbf{21.1} \\
            DP-Sinkhorn~\citep{cao2021don} & 10 & 189.5 \\
            DP-MERF~\citep{harder2021dp} & 10 & 274.0 \\
            \midrule 
            \midrule
            G-PATE~\citep{long2019scalable} (\textdagger) & 10 & 305.92 \\
            DataLens~\citep{wang2021datalens} (\textdagger) & 10 & 320.8 \\
            \bottomrule
        \end{tabular}
        }
    \label{tab:celeba_fid}
\end{table}

\begin{wraptable}{r}{0.3\textwidth}
    \caption{Noise multiplicity ablation on MNIST for $\varepsilon{=}1$. See~\Cref{tab:noise_augmentation_ablation}  for extended results.}
    \label{tab:noise_augmentation_ablation_fid_cnn_acc_no_wrapped}
    \centering
    \scalebox{0.95}{
    \begin{tabular}{l c c}
        \toprule
        $K$ & FID & CNN-Acc (\%) \\
        \midrule 
            1  & 76.9 & 91.7 \\
            2  & 60.1 & 93.1 \\
            4  & 57.1 & 92.8 \\
            8  & 44.8 & 94.1 \\
            16 & 36.9 & 94.2 \\
            32 & 34.8 & 94.4 \\
            \bottomrule
        \end{tabular}
    }
    \vspace{-8mm}
\end{wraptable}
\textbf{Unconditional color image generation.} On CelebA, we train models for $\varepsilon{=}\{1, 10\}$ (\Cref{tab:celeba_fid}). The two DPDMs use the EDM config~\citep{karras2022elucidating}
as well as the Churn sampler; see~\Cref{tab:sampler_settings_for_best_models_fid}.
For $\varepsilon{=}10$, DPDM again outperforms existing methods by a significant margin. DPDM's synthesized images (see~\Cref{fig:celeba_comparison}) appear much more diverse and vivid than the baselines' samples.

\subsection{Ablation Studies} \label{sec:ablation_studies}

\textbf{Noise multiplicity.} \Cref{tab:noise_augmentation_ablation_fid_cnn_acc_no_wrapped} shows results for DPDMs trained with different noise multiplicity $K$ (using the $\rvv$-prediction DM config)~\citep{salimans2022progressive}. 
As expected, increasing $K$ leads to a general trend of improving performance; however, the metrics start to plateau at around $K{=}32$.\looseness=-1

\textbf{Diffusion model config.} We train DPDMs with different DM configs (see~\Cref{sec:diffusion_backbones}). VP- and VE-based models~\citep{song2020} perform poorly
for all settings, while for $\varepsilon{=}0.2$
$\rvv$-prediction 
significantly outperforms the EDM config on MNIST (\Cref{tab:design_choice_eps_0_2_mnist_fid_cnn_acc}). On Fashion-MNIST, the advantage is less significant (extended~\Cref{tab:design_choice}).
For $\varepsilon{=}\{1, 10\}$, the EDM config performs better than $\rvv$-prediction.
Note that the denoiser parameterization for these configs is almost identical and their main difference is the noise distribution $p(\sigma)$ (\Cref{fig:viz_dist}). 
As discussed in \Cref{sec:dpsgd_training}, oversampling large noise levels $\sigma$ is expected to be especially important for the large privacy setting (small $\varepsilon$), which is validated by our ablation.\looseness=-1

\begin{table}
    \centering
    \caption{Sampler comparison on MNIST (see~\Cref{tab:ddim_sddim_churn} for results on Fashion-MNIST). We compare the Churn sampler~\citep{karras2022elucidating} to DDIM~\citep{song2021denoising}.\looseness=-1}
    \scalebox{0.94}{
        \begin{tabular}{l c c c c c}
            \toprule
            \multirow{2}{*}{Sampler} & \multirow{2}{*}{DP-$\varepsilon$} & \multirow{2}{*}{FID} & \multicolumn{3}{c}{Acc (\%)} \\
            \cmidrule(l{1em}r{1em}){4-6}
            & & & Log Reg & MLP & CNN \\
            \midrule 
            Churn (FID) & 0.2 & \textbf{61.9} & 65.3 & 65.8 & 71.9 \\
            Churn (Acc) & 0.2 & 104 & 81.0 & 81.7 & \textbf{86.3} \\
            Stochastic DDIM & 0.2 & 97.8 & 80.2 & 81.3 & 84.4 \\
            Deterministic DDIM & 0.2 & 120 & \textbf{81.3} & \textbf{82.1} & 84.8 \\
            \midrule
            Churn (FID) & 1 & \textbf{23.4} & 83.8 & 87.0 & 93.4 \\
            Churn (Acc) & 1 & 35.5 & \textbf{86.7} & 91.6 & \textbf{95.3} \\
            Stochastic DDIM & 1 & 34.2 & 86.2 & 90.1 & 94.9 \\
            Deterministic DDIM & 1 & 50.4 & 85.7 & \textbf{91.8} & 94.9 \\
            \midrule
            Churn (FID) & 10 & \textbf{5.01} & 90.5 & 94.6 & 97.3 \\
            Churn (Acc) & 10 & 6.65 & \textbf{90.8} & 94.8 & \textbf{98.1} \\
            Stochastic DDIM       & 10 & 6.13 & 90.4 & 94.6 & 97.5 \\
            Deterministic DDIM    & 10 & 10.9 & 90.5 & \textbf{95.2} & 97.7 \\
            \bottomrule
        \end{tabular}
    }
        \label{tab:ddim_sddim_churn_mnist_new}
\end{table}

\textbf{Sampling.} \Cref{tab:ddim_sddim_churn_mnist_new} shows results for different samplers: deterministic and stochastic DDIM~\citep{song2021denoising} as well as the Churn sampler
(tuned for high FID scores and downstream accuracy)\TIM{; see~\Cref{sec:diffusion_solvers} for details on the samplers}. Stochastic sampling is crucial to obtain good perceptual quality, as measured by FID (see poor performance of deterministic DDIM), while it is less important for downstream accuracy. We hypothesize that FID better captures image details that require a sufficiently accurate synthesis process. As discussed in~\Cref{sec:diffusion_models,sec:motivation}, stochastic sampling can help with that and therefore is particularly important in DP-SGD-trained DMs. We also observe that the advantage of the Churn sampler compared to stochastic DDIM becomes less significant as $\varepsilon$ increases. Moreover, in particular for $\varepsilon{=}0.2$ the FID-adjusted Churn sampler performs poorly on downstream accuracy. This is arguably because its settings sacrifice sample diversity, which downstream accuracy usually benefits from, in favor of synthesis quality (also see samples in \Cref{sec:extended_qualitative_results}).\looseness=-1

\vspace{-3.5mm}
\section{Conclusions} \label{sec:conclusion}
\vspace{-2.5mm}
We propose \textit{Differentially Private Diffusion Models} (DPDMs), which use DP-SGD to enforce DP guarantees. DMs are strong candidates for DP generative learning due to their robust training objective and intrinsically less complex denoising neural networks. \TC{To reduce the gradient variance during training, we introduce noise multiplicity and find that} DPDMs achieve state-of-the-art performance in common DP image generation benchmarks. Furthermore, downstream classifiers trained with DPDM-generated synthetic data perform on-par with task-specific discriminative models trained with DP-SGD directly. 
\TC{Note that despite their state-of-the-art results, DPDMs are based on a straightforward idea, this is, to carefully combine DMs with DP-SGD (leveraging the novel noise multiplicity). This ``simplicity'' is a crucial advantage, as it makes DPDMs a potentially powerful tool that can be easily adopted by DP practitioners. Based on our promising results, we conclude that DMs are an ideal generative modeling framework for DP generative learning.}
Moreover, we believe that advancing DM-based DP generative modeling is a pressing topic, considering the extremely fast progress of DM-based large-scale photo-realistic image generation systems~\citep{rombach2021highresolution,saharia2022imagen,ramesh2022dalle2,balaji2022eDiffi}. As future directions we envision applying our DPDM approach
during training of such large image generation DMs, as well as applying DPDMs to other types of data. \TIM{Furthermore, it may be interesting to pre-train our DPDMs with public data that is not subject to privacy constraints, similar to \citet{harder2022differentially}, which may boost performance.} \TIM{Also see \Cref{sec:ethics_and_reproducibility} for further discussion on ethics, reproducibility, limitations and more future work.}

\bibliography{iclr2022_conference,neurips_2022,iclr_2021,icml2023_conferece}
\bibliographystyle{tmlr}

\newpage
\tableofcontents
\appendix
\section{Differential Privacy and Proof of Theorem~\ref{thm:dp}}\label{sec:dp_proof}
In this section, we provide a short proof that the gradients released by the Gaussian mechanism in DPDM are DP. By DP, we are specifically refering to the $(\varepsilon, \delta)$-DP as defined in \cref{def:dp}, which approximates $(\varepsilon)$-DP. For completeness, we state the definition of R\'enyi Differential Privacy (RDP) \citep{mironov2017renyi}:
\begin{definition}(R\'enyi Differential Privacy) \label{def:rdp}
    A randomized mechanism  $\mathcal{M}: \mathcal{D} \to \mathcal{R}$ with domain $\mathcal{D}$ and range $\mathcal{R}$ satisfies $(\alpha, \epsilon)$-RDP if for any adjacent $d, d' \in \mathcal{D}$:
    \begin{align}
        D_\alpha(\mathcal{M}(d)|\mathcal{M}(d')) \leq \epsilon,
    \end{align}
    where $D_\alpha$ is the R\'enyi divergence of order $\alpha$. 
\end{definition}
Gaussian mechanism can provide RDP according to the following theorem:
\begin{theorem} (RDP Gaussian mechanism~\citep{mironov2017renyi}) \label{def:gauss_rdp} For query function $f$ with Sensitivity $S = \max_{d,d}||f(d)-f(d^\prime) ||_2$, the mechanism that releases $f(d) + \mathcal{N}(0, \sigma_\mathrm{DP}^2)$ satisfies $\left(\alpha, \alpha S^2/(2\sigma^2)\right)$-RDP.
\end{theorem}

Note that any $\mathcal{M}$ that satisfies $(\alpha,\epsilon)$-RDP also satisfies $(\epsilon+\frac{\log {1/\delta}}{\alpha-1}, \delta)$-DP. 

We slightly deviate from the notation used in the main text to make the dependency of variables on input data explicit. Recall from the main text that the per-data point loss is computed as an average over $K$ noise samples:
\begin{align} \label{eq:appendix_li}
    \tilde l_i = \frac{1}{K} \sum_{k = 1}^K \lambda(\sigma_{ik}) \|D_\vtheta(\rvx_i + \rvn_{ik}, \sigma_{ik}) - \rvx_i \|_2^2,\, \text{where} \, \{(\sigma_{ik}, \rvn_{ik})\}_{k = 1}^K \sim p(\sigma) \gN\left(\bm{0}, \sigma^2\right).
\end{align}
In each iteration of \Cref{algo:dpdm}, we are given a (random) set of indices $\sB$ of expected size $B$ with no repeated indices, from which we construct a mini-batch $\left\{\rvx_i\right\}_{i \in \sB}$.
In our implementation (which is based on~\citet{yousefpour2021opacus}) of the Gaussian mechanism for gradient sanitization, we compute the gradient of $l_i$ and apply clipping with norm $C$, and then divide the clipped gradients by the expected batch size $B$ to obtain the batched gradient $G_{batch}$:
\begin{align} \label{eqn:gradient_batch}
    G_{batch}(\{\rvx_i\}_{i \in \sB}) = \frac{1}{B} \sum_{i \in \sB} \texttt{clip}_C\left(\nabla_\vtheta l(\rvx_i)\right).
\end{align}
Finally, Gaussian noise $\rvz \sim \gN(\bm{0}, \sigma_\mathrm{DP}^2)$ is added to $G_{batch}$ and released as the response $\tilde{G}_{batch}$:
\begin{align}
    \tilde{G}_{batch}(\{\rvx_i\}_{i \in \sB}) = G_{batch}(\{\rvx_i\}_{i \in \sB}) + \frac{C}{B} \rvz, \quad \rvz \sim \gN(\bm{0}, \sigma_\mathrm{DP}^2 \mI)
\end{align}

Now, we can restate Theorem~\ref{thm:dp} as follows with our modified notation:
\begin{theorem} For noise magnitude $\sigma_\mathrm{DP}$, dataset $d = \{\rvx_i\}_{i = 1}^N$, and set of (non-repeating) indices $\sB$, releasing $\tilde{G}_{batch}(\{\rvx_i\}_{i \in \sB})$ satisfies $\left( \alpha, \alpha / 2 \sigma_\mathrm{DP}^2 \right)$-RDP.
\end{theorem}
\begin{proof}
Without loss of generality, consider two neighboring datasets $d = \{\rvx_i\}_{i = 1}^N$ and $d^\prime = d \cup \rvx^\prime$, $\rvx^\prime \notin d$, and mini-batches $\{\rvx_i\}_{i \in \sB}$ and $\rvx^\prime \cup \{\rvx_i\}_{i \in \sB}$, where the counter-factual set/batch has one additional entry $\rvx^\prime$. We can bound the difference of their gradients in $L_2$-norm as:
\begin{align*}
    & \left \lVert G_{batch}(\{\rvx_i\}_{i \in \sB}) - G_{batch}(\rvx' \cup \{\rvx_i\}_{i \in \sB}) \right \rVert_2 \\
    &= \left\lVert \frac{1}{B} \sum_{i \in \sB} \mathrm{clip}_C\left(\nabla_\vtheta l(\rvx_i)\right) - \left (\frac{1}{B} \mathrm{clip}_C\left(\nabla_\vtheta l(\rvx')\right) + \frac{1}{B} \sum_{i \in \sB} \mathrm{clip}_C\left(\nabla_\vtheta l(\rvx_i)\right) \right)  \right \rVert_2 \\
    &= \left \lVert - \frac{1}{B} \mathrm{clip}_C\left(\nabla_\vtheta l(\rvx')\right) \right \rVert_2\\
    &= \frac{1}{B} \left \lVert \mathrm{clip}_C\left(\nabla_\vtheta l(\rvx')\right)  \right \rVert_2 \leq \frac{C}{B}.
\end{align*}
We thus have \emph{sensitivity} $S(G_{batch}) = \frac{C}{B}$. Furthermore, since $\rvz \sim \gN(\bm{0}, \sigma_\mathrm{DP}^2)$, $ (C/B)\rvz \sim \gN(\bm{0}, (C/B)^2 \sigma_\mathrm{DP}^2)$.
Following standard arguments, releasing $ \tilde{G}_{batch}(\{\rvx_i\}_{i \in \sB}) = G_{batch}(\{\rvx_i\}_{i \in \sB}) +  (C/B) \rvz$ satisfies $\left(\alpha, \alpha / 2 \sigma^2_\mathrm{DP} \right)$-RDP ~\citep{mironov2017renyi}.
\end{proof}

In practice, we construct mini-batches by sampling the training dataset for privacy amplification via Poisson Sampling~\citep{mironov2019r}, and compute the overall privacy cost of training DPDM via RDP composition~\citep{mironov2017renyi}. We use these processes as implemented in Opacus~\citep{yousefpour2021opacus}.

For completeness, we also include the Poisson Sampling algorithm in \Cref{algo:poisson}.

\begin{algorithm}[H]
\caption{Poisson Sampling} %
\begin{algorithmic}\label{algo:poisson}
\small
\STATE {\bfseries Input }: Index range $N$,  subsampling rate $q$%
\STATE {\bfseries Output}: Random batch of indices $\sB$ (of expected size $B$)%
\STATE $\mathrm{c} = \{c_i\}_{i=1}^N {\sim} \texttt{Bernoulli}(q)$
\STATE $\sB = \{ j : j \in \{1, \dots, N\}, \,  \mathrm{c}_j = 1 \} $
\end{algorithmic}
\end{algorithm}

\section{DPGEN Analysis}\label{sec:app_dpgen_failure}

In this section, we provide a detailed analysis of the privacy guarantees provided in DPGEN \citep{chen2022dpgen}. %

As a brief overview, \citet{chen2022dpgen} proposes to learn an energy function $q_\vartheta(\rvx)$ by optimizing the following objective (\cite{chen2022dpgen}, Eq. 7):
\begin{equation*}
    l(\theta; \sigma) = \frac{1}{2}\mathbb{E}_{p(\rvx)} \mathbb{E}_{\tilde{\rvx}\sim \mathcal{N}(\rvx, \sigma^2)}\left[\left|\left| \frac{\tilde{\rvx} - \rvx}{\sigma^2} - \nabla_\rvx \log q_\vartheta (\rvx) \right|\right|^2 \right].
\end{equation*}
In practice, the first expectation is replaced by averaging over examples in a private training set $d=\{x_i : x_i \in Y, i \in 1, \dots, m\}$, and $\frac{\tilde{x}-x}{\sigma^2}$ is replaced by $d_i^r = (\tilde{x}_i - x_i^r)/\sigma_i^2$ for each $i$ in $[1, m]$ (not to be confused with $d$ which denotes the dataset in the DP context), where $x_i^r$ is the query response produced by a data-dependent randomized response mechanism.

We believe that there are three errors in DPGEN that renders the privacy gurantee in DPGEN false. We formally prove the first error in the following section, and state the other two errors which are factual but not mathematical. The three errors are:
\begin{itemize}
    \item The randomized response mechanism employed in DPGEN has a output space that is only supported (has non-zero probability) on combinations of its input \emph{private} dataset. $\epsilon$-differential privacy cannot be achieved as outcomes with non-zero probability\footnote{probability over randomness in the privacy mechanism} can have zero probability when the input dataset is changed by one element. Furthermore, adversaries observing the output can immediately deduce elements of the private dataset.
    \item The $k$-nearest neighbor filtering used by DPGEN to reduce the number of candidates for the randomized response mechanism is a function of the private data. The likelihood of the $k$-selected set varies with the noisy image $\tilde{x}$ (line 20 of algorithm 1 in DPGEN), and is not correctly accounted for in DPGEN.
    \item The objective function used to train the denoising network in DPGEN depends on both the ground-truth denoising direction and a noisy image provided to the denoising network. The noisy image is dependent on the training data, and hence leaks privacy. The privacy cost incurred by using this noisy image is not accounted for in DPGEN.
\end{itemize}

To prove the first error, we begin with re-iterating the formal definition of differential privacy (DP):
\begin{definition}($\epsilon$-Differential Privacy)\label{def:edp}
A randomized mechanism $\mathcal{M}: \mathcal{D} \to \mathcal{I}$ with domain $\mathcal{D}$ and image $\mathcal{I}$ satisfies $(\varepsilon)$-DP if for any two adjacent inputs $d, d' \in \mathcal{D}$ differing by at most one entry, and for any subset of outputs $S \subseteq \mathcal{I}$ it holds that
\begin{equation}
    \mathbf{Pr}\left[\mathcal{M}(d) \in S \right] \leq e^{\varepsilon} \mathbf{Pr}\left[\mathcal{M}(d') \in S \right].
\end{equation}
\end{definition}

The randomized response (RR) mechanism is a fundamental privacy mechanism in differential privacy. A key assumption required in the RR mechanism is that the choices of random response are not dependent on private information, such that when a respondent draws their response randomly from the possible choices, no private information is given. More formally, we give the following definition for randomized response over multiple choices\footnote{This mechanism is analogous to the coin flipping mechanism, where the participant first flip a biased coin to determine whether they'll answer truthfully or lie with probability of lying $\frac{k}{e^\epsilon + k - 1}$, and if they were to lie, they then roll a fair $k$ dice to determine the response.}:
\begin{definition}\label{def:rr}
Given a fixed response set $Y$ of size $k$.  Let $d=\{x_i : x_i \in Y, i \in 1, \dots, m\}$ be an input dataset. Define ``randomized response" mechanism $\mathcal{RR}$ as:
\begin{equation}
    \mathcal{RR}(d) = \{G(x_i)\}_{i \in [1,m]}
\end{equation}
where, 
\begin{equation}
    G(x_i) = \begin{cases} x_i, \, \text{with probability} \; \frac{e^\epsilon}{e^\epsilon + k - 1} \\ x_i' \in Y \setminus x_i , \, \text{with probability} \; \frac{1}{e^\epsilon + k - 1} \end{cases}.
\end{equation}
\end{definition}

A classical result is that the mechanism $\mathcal{RR}$ satisfies $\epsilon$-DP \citep{dwork2014algorithmic}.

DPGEN considers datasets of the form $d=\{x_i : x_i \in \mathbb{R}^n, i \in 1, \dots, m\}$.
It claims to guarantee differential privacy by applying a stochastic function $H$ to each element of the dataset defined as follows (Eq. 8 of \cite{chen2022dpgen}):

\begin{equation*}
    \Pr[H(\tilde{x}_i) = w] = \begin{cases} \frac{e^\epsilon}{e^\epsilon + k - 1},\, w=x_i \\ \frac{1}{e^\epsilon + k - 1},\, w=x_i' \in \mathrm{X} \setminus x_i \end{cases},
\end{equation*}
where $\mathrm{X} = \{x_j : max(\tilde{x}_i - x_j)/\sigma_j \leq \beta, x_j \in d\}$ (max is over the dimensions of $\tilde{x}_i - x_j$), $|\mathrm{X}| = k \geq 2$, and $\tilde{x}_i = x_i + z_i$, $z_i \sim \mathcal{N}(0, \sigma^2 I)$. We first note that $H$ is not only a function of $\tilde{x}_i$ but also $\mathrm{X} \cup x_i$, since its image is determined by $\mathrm{X} \cup x_i$. That is, changes in $X$ will alter the possible outputs of $H$, independently from the value of $\tilde{x}_i$. We make this dependency explicit in our formulation here-forth. This distinction is important as it determines the set of possible outcomes that we need to consider for in the privacy analysis. The authors also noted that $z_i$ is added for training with the denoising objective, not for privacy, so this added Gaussian noise is not essential to the privacy analysis. Furthermore, since $k$ (or equivalently $\beta$) is a hyperparameter that can be tuned, we consider the simpler case where $k=m$, i.e. $\mathrm{X} = d$, as done in the appendix (Eq. 9) by the authors. Thereby we define the privacy mechanism utilized in DPGEN as follows:
\begin{definition}\label{def:arr}
Let $d=\{x_i : x_i \in \mathbb{R}^n, i \in 1, \dots, m\}$ be an input dataset. Define ``data dependent randomized response" $\mathcal{M}$ as:
\begin{equation}
    \mathcal{M}(d) = \{H(x_i, d)\}_{i \in [1,m]}
\end{equation}
where, 
\begin{equation}
    H(x_i, d) = \begin{cases} x_i, \, \text{with probability} \; \frac{e^\epsilon}{e^\epsilon + m - 1} \\ x_i' \in d \setminus x_i , \, \text{with probability} \; \frac{1}{e^\epsilon + m - 1} \end{cases}.
\end{equation}
\end{definition}

Since the image of $H(x_i, d)$ is $d$, $\mathcal{M}(d)$ is only supported on $d^m$.\footnote{We mean dataset-exponentiation in the sense of repeated cartesian products between sets, i.e. $d^2 = d \otimes d$} In other words, the image of $\mathcal{M}$ is data dependent, and any outcome $O$ (which are sets of $\mathbb{R}^n$ tensors, of cardinality $m$) that include elements which are not in $d$ would have a probability of zero to be the outcome of $\mathcal{M}(d)$, i.e. if there exists $z \in O$ and $z \notin d$, then $\Pr[{\mathcal{M}(d)=O}] = 0$.

To construct our counter-example, we start with considering two neighboring datasets: the training data $\mathrm{d}=\{x_i : x_i \in \mathbb{R}^n, i \in 1, \dots, m\}$, and a counter-factual dataset $\mathrm{d'}=\{x_1':  x_1' \in \mathbb{R}^n, x_i : x_i \in \mathbb{R}^n, i \in 2, \dots, m\}$, differing in their first element ($x_1 \neq x_1'$). Importantly, since differential privacy requires that the likelihoods of outputs to be similar for all valid pairs of neighboring datasets, we are free to assume that elements of $\mathrm{d}$ are unique, i.e. no two rows of $\mathrm{d}$ are identical.

Another requirement of differential privacy is that the likelihood of any subsets of outputs must be similar, hence we are free to choose any valid response for the counter-example. Thus, letting $O$ denote the outcome of $\mathcal{M}(\mathrm{d})$, we choose $O = d = \{x_1, \dots, x_m\}$. Clearly, by Definition 0.3, this is a plausible outcome of $\mathcal{M}(\mathrm{d})$ as it is in the support $\mathrm{d}^m$. 
However, $O$ is not in the support of $\mathcal{M}(\mathrm{d'})$ since the first element $x_1$ is not in the image of $H(\cdot \, , \mathrm{d'})$; that is $\Pr[H(x, \mathrm{d'})=x_1] = 0$ for all $x \in \mathrm{d'}$.
Privacy protection is violated since any adversary observing $O$ can immediately deduce the participation of $x_1$ in the data release as opposed to any counterfactual data $x_1'$. 

More formally, consider response set $T = \{O\} \subset \mathrm{d}^m$, and $\mathrm{d}^m$ is the image of $\mathcal{M}(\mathrm{d})$, we have

\begin{align}
    \Pr[\mathcal{M}(d)\in T] &= \Pr[\mathcal{M}(d)=O] \\
    &= \Pr[H(x_1)=x_1] \prod_{i = 2}^m \Pr[H(x_i)=x_i] \quad \text{(independent dice rolls)} \\
    &= \frac{e^\epsilon}{e^\epsilon + m - 1}\prod_{i = 2}^m \Pr[H(x_i)=x_i] \quad \text{(apply \cref{def:arr})}  \\
    &> 0\prod_{i = 2}^m \Pr[H(x_i)=x_i] \\
    &= \Pr[H(x_1')=x_1] \prod_{i = 2}^m \Pr[H(x_i)=x_i] \\
    &= \Pr[\mathcal{M}(d')=O] = \Pr[\mathcal{M}(d')\in T].
\end{align}
We can observe that $ \Pr[\mathcal{M}(d')\in T] = 0$, as shown in line 9. Clearly, this result violates $\epsilon$-DP for all $\epsilon$, which requires $\Pr[\mathcal{M}(d)\in T] \leq e^\epsilon \Pr[\mathcal{M}(d')\in T]$.

In essence, by using private data to form the response set, we make the image of the privacy mechanism data-dependent. This in turn leaks privacy, since an adversary can immediately rule-out all counter-factual datasets that do not include every element of the response $O$, as these counter-factuals now have likelihood 0. To fix this privacy leak, one could determine a response set a-priori, and use the $\mathcal{RR}$ mechanism in \Cref{def:rr} to privately release data. This modification may not be feasible in practice, since constructing a response set of finite size ($k$) suitable for images is non-trivial. Hence, we believe that it would require fundamental modifications to DPGEN to achieve differential privacy.

Regarding error 2, we point out that in the paragraph following Eq. 8 in DPGEN, $X$ is defined as the set of $k$ points in $\mathrm{d}$ that are closest to $\tilde{x}_i$ when weighted by $\sigma_j$. This means that the membership of $X$ is dependent on the value of $\tilde{x}_i$. Thus, any counter-factual input $x_i'$ and $\tilde{x}'_i$ with a different set of $k$ nearest neighbors could have many possible outcomes with 0 likelihood under the true input. In essence, this is a more extreme form of data-dependent randomized response where the response set is dependent on both $\mathrm{d}$ and $x_i$.

Regarding error 3, the loss objective in DPGEN (Eq. 7 of DPGEN, $l = \frac{1}{2} E_{p(x)}E_{\tilde{x} \sim N(x, \sigma^2)}\left[||\frac{\tilde{x} - x}{\sigma^2} - \nabla_x \log q_\theta (\tilde{x}) ||^2 \right]$) includes the term $\nabla_x \log q_\theta(\tilde{x})$, and $\tilde{x}$ is also a function of the private data that is yet to be accounted for at all in the privacy analysis of DPGEN. Hence, one would need to further modify the learning algorithm in DPGEN, such that the inputs to the score model are either processed through an additional privacy mechanism, or sampled randomly without dependence on private data. 

Regarding justifying the premise that DPGEN implements the data-dependent randomized response mechanism, we have verified that the privacy mechanism implemented in the repository of DPGEN (\url{https://github.com/chiamuyu/DPGEN}\footnote{In particular, we refer to the code at commit: 1f684b9b8898bef010838c6a29c030c07d4a5f87.}) is indeed data-dependent: 

In line 30 of \texttt{losses/dsm.py}: 
\begin{verbatim}
    sample_ix = random.choices(range(k), weights=weight)[0]
\end{verbatim}
randomly selects an index in the range of $[0,k-1]$, which is then used in line 46, 
\begin{verbatim}
    sample_buff.append(samples[sample_ix]),
\end{verbatim}
to index the private training data and assigned to the output of
\begin{verbatim}
    sample_buff.
\end{verbatim}
Values of this variable are then accessed on line 85 to calculate the $\frac{\tilde{x} - x^r}{\sigma^2}$ (as $x^r$) term in the objective function (\cite{chen2022dpgen}, Eq. 7).
\section{Model and Implementation Details} \label{sec:model_and_implementation_details}
\subsection{Diffusion Model Configs} \label{sec:diffusion_backbones}
\definecolor{olive}{rgb}{0.5, 0.5, 0.0}
\definecolor{maroon}{rgb}{0.69, 0.19, 0.38}
\definecolor{celestialblue}{rgb}{0.29, 0.59, 0.82}
\definecolor{darkgreen}{rgb}{0.0, 0.6, 0.0}
\definecolor{grey}{rgb}{0.5,0.5,0.5}
\definecolor{darkblue}{rgb}{0.19, 0.19, 0.62}
\definecolor{silver}{rgb}{0.7,0.7,0.7}

\newcommand{\TODO}[1]{{\color{red}TODO: #1}}
\newcommand{\Todo}[1]{{\color{red}TODO: #1}}
\newcommand{\TBD}[1]{{\color{red}#1}}
\newcommand{\WEAK}[1]{\color{silver}{#1}}
\newcommand{\FINAL}[1]{#1}

\def\clap#1{\hbox to 0pt{\hss #1\hss}}%
\def\initials#1{\protect\clap{\smash{\raisebox{1.4ex}{\tiny{\textsf{\textit{#1}}}}}}}%
\newcommand{\EDIT}[4][]{\strut{\color{#3}{\hspace{0pt}\initials{#2}{\color{red}\sout{#1}}{#4}}}}
\newcommand{\TK}[2][]{\protect\EDIT[#1]{TK}{olive}{#2}}
\newcommand{\TA}[2][]{\protect\EDIT[#1]{TA}{maroon}{#2}}
\newcommand{\JL}[2][]{\protect\EDIT[#1]{JL}{celestialblue}{#2}}
\newcommand{\SL}[2][]{\protect\EDIT[#1]{SL}{darkgreen}{#2}}
\newcommand{\MA}[2][]{\protect\EDIT[#1]{MA}{darkblue}{#2}}

\newcommand{\norm}[1]{\left\lVert#1\right\rVert}
\newcommand{\abs}[1]{\lvert#1\rvert}
\newcommand{\real}{\mathbb{R}}
\newcommand{\integer}{\mathbb{Z}}
\newcommand{\low}[1]{\raisebox{0pt}[0pt][0pt]{#1}}
\newcommand{\lowsqrt}[1]{\low{$\sqrt{#1}$}}

\newcommand{\xx}{\boldsymbol{x}}
\newcommand{\XX}{\boldsymbol{X}}
\newcommand{\xxt}{\boldsymbol{\tilde x}}
\newcommand{\xxh}{\boldsymbol{\hat x}}
\newcommand{\zz}{\boldsymbol{z}}

\newcommand{\conv}{\ast}
\newcommand{\mult}{\odot}

\newcommand{\config}[1]{config~{\sc #1}}

\newcommand\undefcolumntype[1]{\expandafter\let\csname NC@find@#1\endcsname\relax}
\newcommand\forcenewcolumntype[1]{\undefcolumntype{#1}\newcolumntype{#1}}

\newcommand{\s}{\hphantom{0}}

\newcommand{\plusplus}{\raisebox{0.15ex}[0pt][0pt]{\scalebox{0.9}{++}}}
\newcommand{\ddpm}{DDPM}
\newcommand{\ddpmpp}{DDPM\plusplus}
\newcommand{\ncsnpp}{NCSN\plusplus}

\newcommand{\stext}[1]{\text{\raisebox{0pt}[0pt][0pt]{#1}}}
\newcommand{\smin}{\sigma_\stext{min}}
\newcommand{\smax}{\sigma_\stext{max}}
\newcommand{\sdyn}{\sigma_\stext{d}}
\newcommand{\tmin}{\odetime_\stext{min}}
\newcommand{\tmax}{\odetime_\stext{max}}
\newcommand{\bmin}{\beta_\text{min}}
\newcommand{\bmax}{\beta_\text{max}}
\newcommand{\bdyn}{\beta_\text{d}}
\newcommand{\sdata}{\sigma_\text{data}}
\newcommand{\cskip}{c_\text{skip}}
\newcommand{\cout}{c_\text{out}}
\newcommand{\cin}{c_\text{in}}
\newcommand{\cnoise}{c_\text{noise}}

\newcommand{\origT}[1]{u_{#1}}
\newcommand{\origTsup}[2]{u_{#1}^{#2}}
\newcommand{\origN}{M}

\newcommand{\Schurn}{S_\text{churn}}
\newcommand{\Stmin}{S_\text{tmin}}
\newcommand{\Stmax}{S_\text{tmax}}
\newcommand{\Snoise}{S_\text{noise}}
\newcommand{\StminStmax}{S_\text{tmin,tmax}}
\newcommand{\StminStmaxSnoise}{S_\text{tmin,tmax,noise}}

\newcommand{\confhdr}[1]{\makebox[1em]{\textsc{#1}}}

\newcommand{\vparagraph}[1]{\vspace*{-1mm}\paragraph{#1}}

\newcommand{\pd}{q}
\newcommand{\condu}{\kappa}
\newcommand{\freq}{\boldsymbol{\nu}}
\newcommand{\fp}{r}
\newcommand{\ff}{\boldsymbol{f}}
\newcommand{\gb}{\boldsymbol{g}}
\newcommand{\Fargs}{\scalebox{0.85}{$\cin(\sigma) (\signal {+} \noise); \cnoise(\sigma)$}}

\pgfplotsset{xtick style={draw=none}}
\pgfplotsset{ytick style={draw=none}}
\pgfplotsset{major grid style={gray!40}}
\pgfplotsset{every axis plot/.style={thick, mark size=1.5pt}}
\pgfplotsset{legend image code/.code={\draw[mark repeat=2, mark phase=2] plot coordinates {(0cm, 0cm) (0.2cm, 0cm) (0.4cm, 0cm)};}} %

\definecolor{C0}{rgb}{0.121569, 0.466667, 0.705882}
\definecolor{C1}{rgb}{1.000000, 0.498039, 0.054902}
\definecolor{C2}{rgb}{0.172549, 0.627451, 0.172549}
\definecolor{C3}{rgb}{0.839216, 0.152941, 0.156863}
\definecolor{C4}{rgb}{0.580392, 0.403922, 0.741176}
\definecolor{C5}{rgb}{0.549020, 0.337255, 0.294118}
\definecolor{C6}{rgb}{0.890196, 0.466667, 0.760784}
\definecolor{C7}{rgb}{0.498039, 0.498039, 0.498039}
\definecolor{C8}{rgb}{0.737255, 0.741176, 0.133333}
\definecolor{C9}{rgb}{0.090196, 0.745098, 0.811765}

\newcommand{\fillbetween}[3][]{\addplot+[name path=A, draw=none, mark=none, forget plot] #2; \addplot+[name path=B, draw=none, mark=none, forget plot] #3; \addplot[#1] fill between[of=A and B]}

\newcommand{\cs}{}
\newcommand{\hh}{0mm}
\newcommand{\hhh}{0mm}
\newcommand{\hhhh}{0mm}
\newcommand{\vvv}{0mm}
\newcommand{\vvvv}{0mm}

\newcommand{\vlabel}[3]{\makebox[0mm][l]{\rotatebox{90}{\makebox[#2][c]{#3}}}\hspace{#1}}
\newcommand{\hrlabel}[1]{\hfill\makebox[0mm]{#1}\hfill}
\newcommand{\vrlabel}[1]{\vfill\makebox[0mm]{#1}\vfill}

\newcommand{\tickYtop}[1]{\raisebox{0ex}[1ex][1ex]{#1}}
\newcommand{\tickYtopD}[1]{\raisebox{-1.5ex}[0ex][0ex]{#1}}
\newcommand{\tickFID}{\tickYtopD{FID}}
\newcommand{\tickLoss}{\tickYtopD{loss}}
\newcommand{\tickTau}{\tickYtopD{$\lVert\lte\rVert$}}
\newcommand{\tickNFE}[1]{\smash{NFE$=$}{$#1$}\hspace*{1em}}
\newcommand{\tickSchurn}[1]{$\mathllap{\smash{\Schurn{=}}}{#1}$}
\newcommand{\tickSchurnB}[1]{\smash{$\Schurn{=}$}{$#1$}\hspace*{2em}}
\newcommand{\tickRho}[1]{$\mathllap{\smash{\rho{=}}}{#1}$}
\newcommand{\tickSigma}[1]{$\mathllap{\smash{\sigma{=}}}{#1}$}

\newcommand{\atphantom}{\vphantom{${}^2$}}
\newcommand{\AProcedure}[2]{\Procedure{\smash{#1}}{\smash{#2}}}
\newcommand{\AComment}[1]{\Comment{\smash{#1}}}
\newcommand{\AState}[1]{\State{\smash{#1}}}
\newcommand{\AFor}[1]{\For{\smash{#1}}}
\newcommand{\AIf}[1]{\If{\smash{#1}}}
\newcommand{\DState}[1]{\State{#1 \vphantom{$\displaystyle\Bigg)$}}}
\newcommand{\MState}[1]{\State\raisebox{0mm}[3.2ex][1.8ex]{#1}}

\newcommand{\eqq}[2]{\pbox[t][5ex]{\linewidth}{{#1}\\{#2}}}
\newcommand{\eqqq}[3]{{\pbox[t][8ex]{\linewidth}{{#1}\\{#2}\\{#3}}}}
\newcommand{\ccbox}[1]{\makebox[3.6em][l]{#1}}
\newcommand{\csbox}[1]{\makebox[2.0em][l]{#1}}
\newcommand{\cnbox}[2]{{#1}\hspace*{1em}\scalebox{0.9}{(note: #2)}}
\newcommand{\cpboxa}[1]{\makebox[1.2em][l]{#1}}
\newcommand{\cpboxb}[1]{\makebox[2.0em][l]{#1}}
\newcommand{\cpboxc}[1]{\makebox[1.4em][l]{#1}}
\newcommand{\cpboxd}[1]{\makebox[2em][l]{#1}}
\newcommand{\rroot}[1]{\,{\vphantom{X}}^\rho\!\!\!\!\sqrt{#1}}
\newcommand{\pphantom}{\vphantom{$= 1 / \sqrt{\sigma^2 + 1} \sdata^2 / \left(\sigma^2 + \sdata^2 \right)$}}
\newcommand{\lphantom}{\vphantom{$= \frac{1}{4} / \sqrt{\sigma^2 + 1} \sdata^2 / \left(\sigma^2 + \sdata^2 \right)$}}
\newcommand{\lfphantom}{\vphantom{$= 1 / \sqrt{\sigma^2 + 1} \sdata^2 / \left(\sigma^2 + \sdata^2 \right)_i$}}
\newcommand{\llphantom}{\lphantom\vphantom{$\underset{i}{\argmin}$}}
\newcommand{\titlerowww}[1]{\makebox[0mm][l]{{\bf #1}}\\}
\newcommand{\titlerowwb}[1]{{\bf #1}\pphantom}

\newcommand{\tabSpecifics}{
\tabulinesep=0.00ex%
\tabulinestyle{0.17mm}%
\begin{table}[t]%
\centering%
\caption{\label{tab:specifics}%
Four popular DM configs from the literature.
}
\vspace{1.5mm}%

\resizebox{\textwidth}{!}{%
\begin{tabu}{@{}lllll@{\hspace*{-1mm}}}
\tabucline{-}
& {\bf VP~\citep{song2020}}
& {\bf VE~\citep{song2020}}
& {\bf $\rvv$-prediction ~\citep{salimans2022progressive}}
& {\bf EDM~\citep{karras2022elucidating}}
\\
\hline\vspace*{-3.5mm}\\
\titlerowww{Network and preconditioning}
Skip scaling\hfill$\cskip(\sigma)$\lphantom
& \csbox{$1$}
& \csbox{$1$}
& \csbox{$1 / \sqrt{1 + \sigma^2}$}
& \csbox{$\sdata^2 / \left(\sigma^2 + \sdata^2 \right)$}
\vspace*{0.7mm}\\
Output scaling \hfill$\cout(\sigma)$\lphantom
& \csbox{$-\sigma$}
& \csbox{$\sigma$}
& \csbox{$-\sigma / \sqrt{1 + \sigma^2}$}
& \csbox{$\sigma \cdot \sdata / \sqrt{\sdata^2 + \sigma^2}$}
\vspace*{0.7mm}\\
Input scaling\hfill$\cin(\sigma)$\lphantom
& \csbox{$1 / \sqrt{\sigma^2 + 1}$}
& \csbox{$1$}
& \csbox{$1 / \sqrt{1 + \sigma^2}$}
& \csbox{$1 / \sqrt{\sigma^2 + \sdata^2}$}
\vspace*{0.7mm}\\
Noise cond.\hfill$\cnoise(\sigma)$\llphantom
& \csbox{$(\origN - 1) ~t$}
& \csbox{$\ln(\frac{1}{2} \sigma)$}
& \csbox{$t$}
& \csbox{$\frac{1}{4} \ln(\sigma)$}
\vspace*{-1.0mm}\\
\hline\vspace*{-3.5mm}\\
\titlerowww{Training}
Noise distribution\lphantom
& {$t \sim \mathcal{U}(\epsilon_\text{t}, 1)$}
& {$\ln(\sigma)\!\sim\!\mathcal{U}(\ln(\smin),$}
& {$t \sim \mathcal{U}(\epsilon_\mathrm{min}, \epsilon_\mathrm{max})$}
& {$\ln(\sigma) \sim \mathcal{N}(P_\stext{mean}^{}, P_\stext{std}^2)$}
\\
&& {\hspace{4.3em}$\ln(\smax))$} &&
\vspace*{-1.0mm}\\
Loss weighting\hfill$\lambda(\sigma)$\llphantom\lphantom
& {$1 / \sigma^2$}
& {$1 / \sigma^2$}
& {$\left( \sigma^2\!+\! 1 \right) / \sigma^2$ \; (``SNR+1'' weighting)}
& {$\left( \sigma^2\!+\!\sdata^2 \right) / (\sigma \cdot \sdata)^2$}
\vspace*{-1.5mm}\\
\tabucline{-}\\
\vspace*{-7mm}\\
\titlerowwb{Parameters}
& \cpboxa{\footnotesize{$\bdyn$}}{\footnotesize{$= 19.9,  \bmin=0.1$}}
& \cpboxb{\footnotesize{$\smin$}}{\footnotesize{$= 0.002$}}
& \cpboxc{\footnotesize{$\epsilon_\mathrm{min}$}}{\footnotesize{\; $= \tfrac{2}{\pi} \arccos \tfrac{1}{\sqrt{1 + e^{-13}}}$}}
& \cpboxd{\footnotesize{$P_\text{mean}$}}{\footnotesize{$= -1.2$, $P_\text{std} = 1.2$}}
\\
\pphantom
& \cpboxa{\footnotesize{$\epsilon_\text{t}$}}{\footnotesize{$= 10^{-5},M = 1000$}}
& \cpboxb{\footnotesize{$\smax$}}{\footnotesize{$= 80$}}
& \cpboxc{\footnotesize{$\epsilon_\mathrm{max}$}}{\footnotesize{\; $= \tfrac{2}{\pi} \arccos \tfrac{1}{\sqrt{1 + e^{9}}}$}}
& \cpboxb{\footnotesize{$\sdata$}}{\footnotesize{$= \sqrt{\tfrac{1}{3}}$}}
\\
\pphantom
& \cpboxa{\footnotesize{$\sigma(t)$}}{\footnotesize{\; $= \sqrt{e^{\frac{1}{2}\bdyn t^2 + \bmin t}\!-\!1}$}}
&
& \cpboxc{\footnotesize{$\sigma(t)$}}{\footnotesize{\; $=\sqrt{\cos^{-2}(\pi t / 2) - 1}$}}
&
\\
\tabucline{-}
\end{tabu}}%
\end{table}%
}
\tabSpecifics
As discussed in~\Cref{sec:background}, previous works proposed various denoiser models $D_\vtheta$, noise distributions $p(\sigma)$, and weighting functions $\lambda(\sigma)$. We refer to the triplet $(D_\vtheta, p, \lambda)$ as DM config. In this work, we consider four such configs: \emph{variance preserving} (VP)~\citep{song2020}, \emph{variance exploding} (VE)~\citep{song2020}, $\rvv$-prediction~\citep{salimans2022progressive}, and  EDM~\citet{karras2022elucidating}. The triplet for each of these configs can be found in~\Cref{tab:specifics}. Note, that we use the parameterization of the denoiser model $D_\vtheta$ from~\citep{karras2022elucidating}
\begin{align}
    D_\vtheta(\rvx; \sigma) = c_\mathrm{skip}(\sigma) \rvx + c_\mathrm{out}(\sigma) F_\vtheta(c_\mathrm{in}(\sigma) \rvx; c_\mathrm{noise}(\sigma)),
\end{align}
where $F_\vtheta$ is the raw neural network. To accommodate for our particular sampler setting (we require to learn the denoiser model for $\sigma \in [0.002, 80]$; see~\Cref{sec:diffusion_solvers}) we slightly modified the parameters of VE and $\rvv$-prediction. For VE, we changed $\sigma_\mathrm{min}$ and $\sigma_\mathrm{max}$ from 0.02 to 0.002 and from 100 to 80, respectively. For $\rvv$-prediction, we changed $\epsilon_\mathrm{min}$ and $\epsilon_\mathrm{max}$ from $\tfrac{2}{\pi} \arccos \tfrac{1}{\sqrt{1 + e^{-20}}}$ to $\tfrac{2}{\pi} \arccos \tfrac{1}{\sqrt{1 + e^{-13}}}$ and $\tfrac{2}{\pi} \arccos \tfrac{1}{\sqrt{1 + e^{20}}}$ to $\tfrac{2}{\pi} \arccos \tfrac{1}{\sqrt{1 + e^{9}}}$, respectively. Furthermore, we cannot base our EDM models on the true (training) data standard deviation $\sigma_\mathrm{data}$ as releasing this information would result in a privacy cost. Instead, we set $\sigma_\mathrm{data}$ to the standard deviation of a uniform distribution between $-1$ and $1$, assuming no prior information on the modeled image data.

\subsubsection{Noise Level Visualization}
In the following, we provide details on how exactly the noise distributions of the four configs are visualized in~\Cref{fig:viz_dist}. The reason we want to plot these noise distributions is to understand how the different configs assign weight to different noise levels $\sigma$ during training through sampling some $\sigma$'s more and others less. However, to be able to make a meaningful conclusion, we also need to take into account the loss weighting $\lambda(\sigma)$. 

Therefore, we consider the effective ``importance-weighted'' distributions $p(\sigma)\frac{\lambda(\sigma)}{ \lambda_\mathrm{EDM}(\sigma)}$, where we use the loss weighting from the EDM config as reference weighting.

The $\frac{\lambda(\sigma)}{ \lambda_\mathrm{EDM}(\sigma)}$ weightings for VP, VE, $\rvv$-prediction, and EDM are then, $\sigma_\mathrm{data}^2 / (\sigma^2 + \sigma_\mathrm{data}^2)$, $\sigma_\mathrm{data}^2 / (\sigma^2 + \sigma_\mathrm{data}^2)$, $\sigma_\mathrm{data}^2 (\sigma^2 + 1) / (\sigma^2 + \sigma_\mathrm{data}^2)$, and 1, respectively. \Cref{fig:viz_dist} then visualizes the ``importance-weighted'' distributions in log-$\sigma$ space, following~\citet{karras2022elucidating} (that way, the final visualized log-$\sigma$ distribution of EDM remains a normal distribution $\gN(P_\mathrm{mean}, P_\mathrm{std}^2)$).

\subsection{Model Architecture} \label{sec:model_architecture}
We focus on image synthesis and implement the neural network backbone of DPDMs using the DDPM++ architecture~\citep{song2020}. For class-conditional generation, we add a learned class-embedding to the $\sigma$-embedding as is common practice~\citep{dhariwal2021diffusion}. All model hyperparameters and training details can be found in~\Cref{tab:model_hyperparameters_and_training_details}.
\begin{table}
    \centering
    \caption{Model hyperparameters and training details.}
    \scalebox{0.96}{
    \begin{tabular}{l c c}
        \toprule
         Hyperparameter & MNIST \& Fashion-MNIST & CelebA \\
         \midrule
         \textbf{Model} & \\
         Data dimensionality (in pixels) & 28 & 32 \\
         Residual blocks per resolution & 2 & 2 \\
         Attention resolution(s) & 7 & 8,16 \\
         Base channels & 32 & 32 \\
         Channel multipliers & 1,2,2 & 1,2,2 \\
         EMA rate & 0.999 & 0.999 \\
         \# of parameters & 1.75M & 1.80M \\
         Base architecture & DDPM++~\citep{song2020} & DDPM++~\citep{song2020} \\
         \midrule
         \textbf{Training} & \\
         \# of epochs & 300 & 300 \\
         Optimizer & Adam~\citep{kingma2015adam} & Adam~\citep{kingma2015adam}  \\
         Learning rate & $3\cdot 10^{-4}$ & $3\cdot 10^{-4}$  \\
         Batch size & 4096 & 2048 \\
         Dropout & 0 & 0 \\
         Clipping constant $C$ & 1 & 1 \\
         DP-$\delta$ & $10^{-5}$ & $10^{-6}$ \\
         \bottomrule
    \end{tabular}
    }
    \label{tab:model_hyperparameters_and_training_details}
\end{table}
\subsection{Sampling from Diffusion Models} \label{sec:diffusion_solvers}
Let us recall the differential equations we can use to generate samples from DMs:
\begin{align}
    \text{ODE:} \; d\rvx &= -\dot \sigma(t) \sigma(t) \nabla_\rvx \log p(\rvx; \sigma(t)) \, dt, \\
    \text{SDE:} \; d\rvx &= - \dot \sigma(t) \sigma(t) \nabla_\rvx \log p(\rvx; \sigma(t)) \, dt - \beta(t) \sigma^2(t) \nabla_\rvx \log p(\rvx; \sigma(t)) \, dt + \sqrt{2 \beta(t)} \sigma(t)\, d\omega_t.
\end{align}
Before choosing a numerical sampler, we first need to define a sampling schedule. In this work, we follow~\citet{karras2022elucidating} and use the schedule 
\begin{align}
    \sigma_i = \left(\sigma_\mathrm{max}^{1/\rho} + \frac{i}{M-1} (\sigma_\mathrm{min}^{1/\rho} - \sigma_\mathrm{max}^{1/\rho}) \right)^\rho, i \in \{0, \dots, M-1\},
\end{align}
with $\rho{=}7.0$, $\sigma_\mathrm{max}{=}80$ and $\sigma_\mathrm{min}{=}0.002$. We consider two solvers: the (stochastic \TIM{($\eta=1$)}/deterministic \TIM{($\eta=0$)}) DDIM solver~\citep{song2021denoising} as well as the stochastic Churn solver introduced in~\citep{karras2022elucidating}, for pseudocode see~\Cref{alg:ddim} and~\Cref{alg:churn}, respectively. Both implementations can readily be combined with classifier-free guidance, which is described in~\Cref{sec:diffusion_guidance}, in which case the denoiser $D_\vtheta(\rvx; \sigma)$ may be replaced by $D^w_\vtheta(\rvx; \sigma, \rvy)$, where the guidance scale $w$ is a hyperparameter. Note that the Churn sampler has four additional hyperparameters which should be tuned empirically~\citep{karras2022elucidating}. If not stated otherwise, we set $M{=}1000$ for the Churn sampler and the stochastic DDIM sampler, and $M{=}50$ for the deterministic DDIM sampler.
\begin{algorithm}[H]
\small
\caption{DDIM sampler~\citep{song2021denoising}
}
\begin{algorithmic} \label{alg:ddim}

\STATE {\bfseries Input:} Denoiser $D_\vtheta(\rvx; \sigma)$, Schedule $\{\sigma_i\}_{i \in \{0, \dots, M-1\}}$
\STATE {\bfseries Output:} Sample $\rvx_M$
\STATE Sample $\rvx_{0} \sim \gN\left(\bm{0}, \sigma_0^2\mI\right)$
\FOR{$n=0$ {\bfseries to} $M-2$}
\STATE Evaluate denoiser $\rvd_n = D_\vtheta(\rvx_i, \sigma_i)$
\IF{\TIM{Stochastic DDIM}}
    \STATE $\rvx_{n+1} = \rvx_n + 2 \frac{\sigma_{n+1} - \sigma_n}{\sigma_n} (\rvx_n - \rvd_n) + \sqrt{2 (\sigma_n - \sigma_{n+1}) \sigma_n} \rvz_n, \quad \rvz_n \sim \gN(\bm{0}, \mI)$
\ELSIF{\TIM{Deterministic DDIM}}
    \STATE $\rvx_{n+1} = \rvx_n + \frac{\sigma_{n+1} - \sigma_n}{\sigma_n} (\rvx_n - \rvd_n)$
\ENDIF
\ENDFOR
\STATE Return $\rvx_M = D(\rvx_{N-1}, \sigma_{M-1})$
\end{algorithmic}
\end{algorithm}
\begin{algorithm}[H]
\small
\caption{Churn sampler~\citep{karras2022elucidating}
}
\begin{algorithmic} \label{alg:churn}

\STATE {\bfseries Input:} Denoiser $D_\vtheta(\rvx; \sigma)$, Schedule $\{\sigma_i\}_{i \in \{0, \dots, M-1\}}$, $S_\mathrm{noise}$, $S_\mathrm{churn}$, $S_\mathrm{min}$, $S_\mathrm{max}$
\STATE {\bfseries Output:} Sample $\rvx_M$
\STATE Set $\sigma_{M} = 0$
\STATE Sample $\rvx_{0} \sim \gN\left(\bm{0}, \sigma_0^2\mI\right)$
\FOR{$n=0$ {\bfseries to} $M-1$}
\IF{$\sigma_i \in [S_\mathrm{min}, S_\mathrm{max}]$}
    \STATE $\gamma_i = \min(\tfrac{S_\mathrm{churn}}{M}, \sqrt{2} - 1)$
\ELSE
    \STATE $\gamma_i = 0$
\ENDIF
\STATE Increase noise level $\widetilde \sigma_n = (1+\gamma_n) \sigma_n$
\STATE Sample $\rvz_n \sim  \gN\left(\bm{0}, S_\mathrm{noise}^2\mI\right)$ and set $\widetilde \rvx_n = \rvx_n + \sqrt{\widetilde \sigma_n^2 - \sigma_n^2} \rvz_n$
\STATE Evaluate denoiser $\rvd_n = D_\vtheta(\widetilde \rvx_n, \widetilde \sigma_n)$ and set $\rvf_n = \frac{\widetilde \rvx_n - \rvd_n}{\widetilde \sigma_n}$ 
\STATE $\rvx_{n+1} = \widetilde \rvx_M + (\sigma_{n+1} - \widetilde \sigma_n) \rvf_n$
\IF{$\sigma_{n+1} \neq 0$}
    \STATE Evaluate denoiser $\rvd_n^\prime = D_\vtheta(\rvx_{n + 1}, \sigma_{n+1})$ and set $\rvf_n^\prime = \frac{\rvx_{n + 1}- \rvd_n^\prime}{\sigma_{n+1}}$
    \STATE Apply second order correction: $\rvx_{n+1} = \widetilde \rvx_n + \tfrac{1}{2} (\sigma_{n+1} - \widetilde \sigma_n) (\rvf_n + \rvf_n^\prime)$
\ENDIF
\ENDFOR
\STATE Return $\rvx_M$
\end{algorithmic}
\end{algorithm}

\subsubsection{Guidance} \label{sec:diffusion_guidance}

Classifier guidance~\citep{song2020, dhariwal2021diffusion} is a technique to guide the diffusion sampling process towards a particular conditioning signal $\rvy$ using gradients, with respect to $\rvx$, of a pre-trained, noise-conditional classifier $p(\rvy|\rvx, \sigma)$. Classifier-free guidance~\citep{ho2021classifierfree}, in contrast, avoids training additional classifiers by mixing denoising predictions of an unconditional and a conditional model, according to a \emph{guidance scale} $w$, by replacing $D_\vtheta(\rvx; \sigma)$ in 
the score parameterization $s_\vtheta = (D_\vtheta(\rvx; \sigma) - \rvx) / \sigma^2$ with
\begin{align}
    D^w_\vtheta(\rvx; \sigma, \rvy) = (1-w) D_\vtheta(\rvx; \sigma) + w D_\vtheta(\rvx; \sigma, \rvy).
\end{align}
$D_\vtheta(\rvx; \sigma)$ and $D_\vtheta(\rvx; \sigma, \rvy)$ can be trained jointly; to train $D_\vtheta(\rvx; \sigma)$ the conditioning signal $\rvy$ is discarded at random and replaced by a \emph{null token}~\citep{ho2021classifierfree}. Increased guidance scales $w$ tend to drive samples deeper into the model's modes defined by $\rvy$ at the cost of sample diversity.
\subsection{Hyperparameters of Differentially Private Diffusion Models} \label{sec:dpdm_hyperparameters}
Tuning hyperparameters for DP models generally induces a privacy cost which should be accounted for~\citep{papernot2022hyperparameter}. Similar to existing works~\citep{de2022unlocking}, we neglect the (small) privacy cost associated with hyperparameter tuning. Nonetheless, in this section we want to point out that our hyperparameters show consistent trends across different settings. As a result, we believe our models need little to no hyperparameter tuning in similar settings to the ones considered in this work.

\textbf{Model.} We use the DDPM++~\citep{song2020} architecture for all models in this work. Across all three datasets (MNIST, Fashion-MNIST, and CelebA) we found the EDM~\citep{karras2022elucidating} config to perform best for $\varepsilon{=}\{1, 10\}$. On MNIST and Fashion-MNIST, we use the $\rvv$-prediction config for $\varepsilon=0.2$ (not applicable to CelebA).

\textbf{DP-SGD training.} In all settings, we use 300 epochs and clipping constant $C{=}1$. We use batch size $B{=}4096$ for MNIST and Fashion-MNIST and decrease the batch size of CelebA to $B{=}2048$ for the sole purpose of fitting the entire batch into GPU memory. The DP noise $\sigma_\mathrm{DP}$ values for each setup can be found in~\Cref{tab:dp_noise}

\begin{table}
    \centering
    \caption{DP noise $\sigma_\mathrm{DP}$ used for all our experiments.}
    \label{tab:dp_noise}
    \begin{tabular}{c c c c}
        \toprule
        $\varepsilon$ & MNIST & Fashion-MNIST & CelebA \\
        \midrule
        0.2 & 82.5 & 82.5 & N/A \\
        1 & 18.28125 & 18.28125 & 8.82812\\
        10 & 2.48779 & 2.48779 & 1.30371 \\
        \bottomrule
    \end{tabular}
\end{table}

\textbf{DM Sampling.} We experiment with different DM solvers in this work. We found the DDIM sampler~\citep{song2021denoising} (in particular the stochastic version), which does not have any hyperparameters (without guidance), to perform well across all settings. Using the Churn sampler~\citep{karras2022elucidating}, we could improve perceptual quality (measured in FID), however, out of the five (four without guidance) hyperparameters, we only found two (one without guidance) to improve results significantly. We show results for all samplers in~\Cref{sec:extended_qualitative_results}.
\section{Variance Reduction via Noise Multiplicity} \label{sec:variance_reduction_via_noise_multiplicity}
As discussed in~\Cref{sec:dpsgd_training}, we introduce \emph{noise multiplicity} to reduce gradient variance. 

\subsection{Proof of~\Cref{th:noise_mult}}
\begin{theorem*}
    \textit{The variance of the DM objective (\Cref{eq:noise_mult}) decreases with increased noise multiplicity $K$ as $1/K$.}
\end{theorem*}
\begin{proof}
    The DM objective in~\Cref{eq:noise_mult} is a Monte Carlo estimator of the true intractable $L_2$-loss in~\Cref{eq:diffusion_objective}, using one data sample $\rvx_i \sim \pdata$ and $K$ noise-level-noise tuples $\{(\sigma_{ik}, \rvn_{ik})\}_{k =1}^K \sim p(\sigma) \gN\left(\bm{0}, \sigma^2\right)$. Replacing expectations with Monte Carlo estimates is a common practice to ensure numerical tractability. For a generic function $r$ over distribution $p(\rvk)$, we have $\E_{p(\rvk)}[r(\rvk)] \approx \frac{1}{K} \sum_{i=1}^K r(\rvk_i)$, where $\{\rvk_i\}_{i=1}^K \sim p(\rvk)$ (Monte Carlo estimator for expectation of function $r$ with respect to distribution $p$). The Monte Carlo estimate is a noisy unbiased estimator of the expectation $\E_{p(\rvk)}[r(\rvk)]$ with variance $\frac{1}{K} \Var_p[r]$, where $\Var_p[r]$ is the variance of $r$ itself. This is a well-known fact; see for example Chapter 2 of the excellent book by~\citet{owenMCbook}. %
    This proves that the variance of the DM objective in~\Cref{eq:noise_mult} decreases with increased noise multiplicity $K$ as $1/K$.
    \end{proof}

\subsection{Variance Reduction Experiment}
In this section, we empirically show how the reduced variance of the DM objective from noise multiplicity leads to reduced gradient variance during training. In particular, we set $\rvx_i$ to a randomly sampled MNIST image and set the denoiser $D_\vtheta$ to our trained model on MNIST. We then compute gradients for different noise multiplicities $K$. We resample the noise values 1k times (for each $K$) to estimate the variance of the gradient for each parameter. In~\Cref{fig:variance_reduction}, we show the histogram over gradient variance as well as the average gradient variance (averaged over all parameters in the model). Note that the variance of each gradient is a random variable itself (which is estimated using 1k Monte Carlo samples). We find that an increased $K$ leads to significantly reduced training parameter gradient variance. 

\begin{figure}
    \centering
    \includegraphics[scale=0.8]{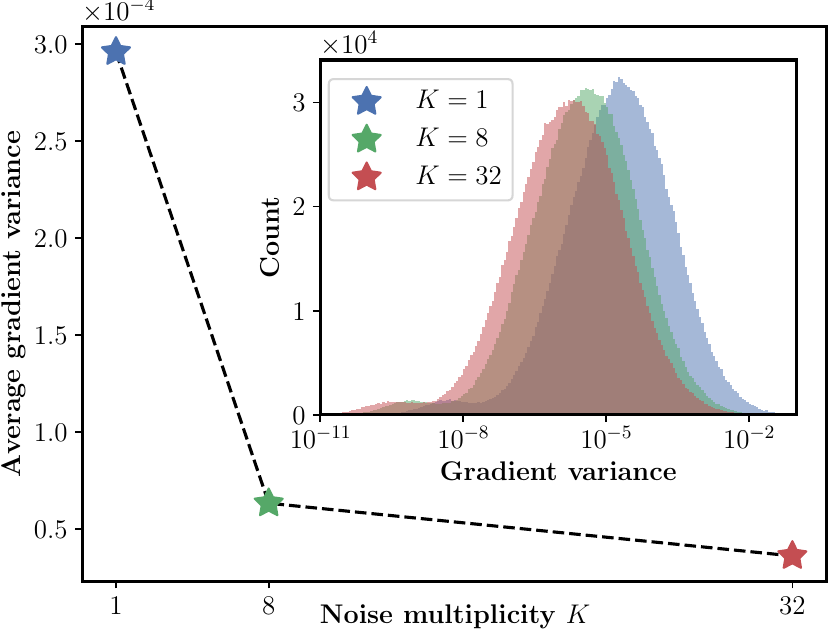}
    \caption{Variance reduction via noise multiplicity. Increasing $K$ in \emph{noise multiplicity} leads to significant variance reduction of parameter gradient estimates during training (note logarithmic axis in inset). This is an enlarged version of~\Cref{fig:variance_reduction_small}.}
    \label{fig:variance_reduction}
\end{figure}

\subsection{Computational Cost of Noise Multiplicity} \label{sec:noise_multiplicity_cost}
In terms of computational cost, noise multiplicity is expensive and likely not useful for non-DP DMs. The computational cost increases linearly with $K$, as the denoiser needs to run $K$ times. \TIM{Furthermore, in theory, noise multiplicity increases the memory footprint by at least $\gO(K)$; however, in common DP frameworks, such as Opacus~\citep{yousefpour2021opacus}, which we use, the peak memory requirement is $\mathcal{O}(K^2)$ compared to non-private training. Recent methods such as \emph{ghost clipping}~\citep{bu2022scalable} require less memory, but are currently not widely implemented.}

That being said, in DP generative modeling, and DP machine learning more generally, computational cost is hardly ever the bottleneck; the main bottleneck is the privacy restriction. The privacy restriction implies only a finite number of training iterations and we need to use that budget of training iterations in the most efficient way (this is, using training gradients that suffer from as little noise as possible). Our numerical experiments clearly show that noise multiplicity is a technique to shift the privacy-utility trade-off, effectively getting better utility at the same privacy budget, using additional computational cost.

\subsection{On the Difference between Noise Multiplicity and Augmentation Multiplicity} \label{sec:noise_multiplicity_vs_data_aug}

\TIM{Augmentation multiplicity~\citep{de2022unlocking} is a technique where multiple augmentations per image are used to train classifiers with DP-SGD. Image augmentations have also been shown to be potentially helpful in data-limited (image) generative modeling, for example, for autoregressive models~\citep{jun2020distribution} and DMs~\citep{karras2022elucidating}. In stark contrast to discriminative modeling where the data distribution $\pdata$ can simply be replaced by the \emph{augmented data distribution} and the neural backbone can be left as is, in generative modeling both the loss function and the neural backbone need to be adapted. For example, for DMs~\citep{karras2022elucidating}, the standard DM loss (\Cref{eq:diffusion_objective}) is formally replaced by
\begin{align}
    \E_{\tilde \rvx \sim \pdata(\tilde \rvx), c\sim p_\mathrm{aug}(c), \rvx \sim p_\mathrm{augdata}(\rvx \mid \tilde \rvx, c), (\sigma, \rvn) \sim p(\sigma, \rvn)} \left[\lambda_\sigma \|D_\vtheta(\rvx + \rvn, \sigma, c) - \rvx \|_2^2 \right],
\end{align}
where $p_\mathrm{aug}(c)$ is the distribution over augmentation choices $c$ (for example, cropping at certain coordinates or other image transformations or perturbations),
and $p_\mathrm{augdata}(\rvx \mid \tilde \rvx, c)$
is the 
distribution over augmented images $\rvx$ given the original dataset images $\tilde \rvx$ and the augmentation $c$. Importantly, note that the neural backbone $D_\vtheta$ also needs to be conditioned on the augmentation choice $c$ as, at inference time, we generally only want to generate ``clean images'' with $c(\rvx) = \rvx$ (no augmentation). \\
\\
While noise multiplicity provably reduces the variance of the standard diffusion loss (see \Cref{th:noise_mult}), augmentation multiplicity, that is, averaging over multiple augmentations for a given clean image $\rvx$, only provably reduces the variance of the augmented diffusion loss, which is by definition more noisy due to the additional expectations.
Furthermore, it is not obvious how minimizing the augmented diffusion loss relates to minimizing the true diffusion loss. In contrast to noise multiplicity, augmentation multiplicity \emph{does not} provably reduce the variance of the original diffusion loss; rather, it is a data augmentation technique for enriching training data. \\
\\
Pointing out the orthogonality of the two ideas again, note that noise multiplicity is still applicable for the above modified augmented diffusion loss objective. Furthermore, we would like to point out that noise multiplicity is applicable for DPDMs in any domain, beyond images; in contrast, data augmentations need to be handcrafted and may not readily available in all fields.}

\section{Toy Experiments} \label{sec:toy_experiments}
\begin{figure}
    \centering
    \begin{subfigure}[b]{0.32\textwidth}
        \centering
        \includegraphics[scale=1.3]{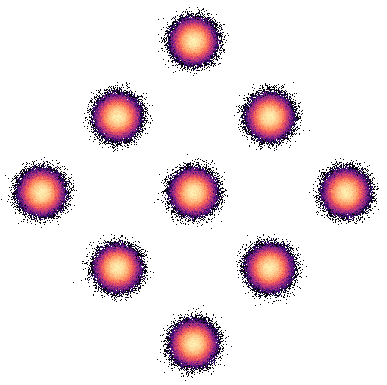}
        \caption{Data $p_\mathrm{data}$}
        \label{fig:large_gmm}
    \end{subfigure}
    \begin{subfigure}[b]{0.32\textwidth}
        \centering
        \includegraphics[scale=1.3]{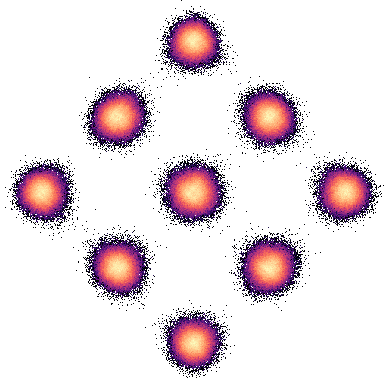}
        \caption{Samples from DM.}
        \label{fig:large_diffusion_resnet}
    \end{subfigure}
    \begin{subfigure}[b]{0.32\textwidth}
        \centering
        \includegraphics[scale=1.3]{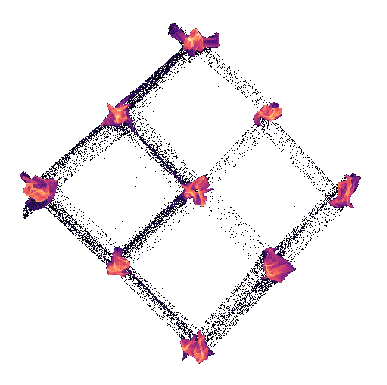}
        \caption{Samples from GAN.}
        \label{fig:large_gan_resnet}
    \end{subfigure}
    \caption{Mixture of Gaussians: data distribution and (1M) samples from a DM as well as a GAN. Our visualization is based on the log-histogram, which shows single data points as black dots.}
    \label{fig:gmm_and_dm_large}
\end{figure}
In this section, we describe the details of the toy experiment from paragraph \textbf{(ii) Sequential denoising} in~\Cref{sec:motivation}. For this experiment, we consider a two-dimensional simple Gaussian mixture model of the form
\begin{align}
    p_\mathrm{data}(\rvx) = \sum_{k=1}^9 \frac{1}{9} p^{(k)}(\rvx),
\end{align}
where $p^{(k)}(\rvx) = \gN(\rvx; \vmu_k; \sigma_0^2)$ and 
\begin{alignat*}{5}
    \vmu_1 &= \begin{pmatrix} -a \\ 0 \end{pmatrix},\quad &&\vmu_2 &&= \begin{pmatrix} -a/2 \\ a/2 \end{pmatrix},\quad &&\vmu_3 &&= \begin{pmatrix} 0 \\ a \end{pmatrix}, \\
    \vmu_4 &= \begin{pmatrix} -a/2 \\ -a/2 \end{pmatrix},\quad &&\vmu_5 &&= \begin{pmatrix} 0 \\ 0 \end{pmatrix},\quad &&\vmu_6 &&= \begin{pmatrix} a/2 \\ a/2 \end{pmatrix},\\
    \vmu_7 &= \begin{pmatrix} 0 \\ -a \end{pmatrix},\quad &&\vmu_8 &&= \begin{pmatrix} a/2 \\ -a/2 \end{pmatrix},\quad &&\vmu_9 &&= \begin{pmatrix} a \\ 0 \end{pmatrix},
\end{alignat*}
where $\sigma_0 = 1/25$ and $a=1/\sqrt{2}$. The data distribution is visualized in~\Cref{fig:large_gmm}.

\textbf{Fitting.} Initially, we fitted a DM as well as a GAN to the mixture of Gaussians. The neural networks of the DM and the GAN generator use similar ResNet architectures with 267k and 264k (1.1\% smaller) parameters, respectively (see~\Cref{sec:toy_training} for training details). The fitted distributions are visualized in~\Cref{fig:gmm_and_dm_large}. In this experiment, we use deterministic DDIM (\Cref{alg:ddim})~\citep{song2021denoising}, a numerical solver for the Probability Flow ODE (\Cref{eq:probability_flow_ode})~\citep{song2020}, with 100 neural function evaluations (DDIM-100) as the end-to-end multi-step synthesis process for the DM. Even though our visualization shows that the DM clearly fits the distribution better (\Cref{fig:gmm_and_dm_large}), the GAN does not do bad either. Note that our visualization is based on the log-histogram of the sampling distributions, and therefore puts significant emphasis on single data point outliers.

We provide a second method to assess the fitting: In particular, we measure the percentage of points (out of 1M samples) that are within a $h$-standard deviation vicinity of any of the nine modes. A point $\rvx$ is said to be within a $h$-standard deviation vicinity of the mode $\vmu_k$ if $\|\rvx - \vmu_k\| < h \sigma_0$. We present results for this metric in~\Cref{tab:fitting} for $h{=}\{1, 2, 3, 4, 5, 6\}$. Note that any mode is at least $12.5$ standard deviations separated to the next mode, and therefore no point can be in the $h$-standard deviation vicinity of more than two modes for $h\leq6$.

The results in~\Cref{tab:fitting} indicate that the GAN is slightly too sharp, that is, it puts too many points within the $1$- and $2$-standard deviation vicinity of modes. Moreover, for larger $h$, the result in~\Cref{tab:fitting} suggests that the samples in~\Cref{fig:large_gan_resnet} that appear to ``connect'' the GAN's modes are heavily overemphasized---these samples actually represent less than 1\% of the total samples; 99.3\% of samples are within a 4-standard deviation vicinity of a mode while modes are at least 12.5 standard deviations separated.

\begin{table}
    \centering
    \caption{$h$-standard deviation vicinity metric as defined in the paragraph $\textbf{Fitting}$ of~\Cref{sec:toy_experiments}.}%
    \label{tab:fitting}
    \begin{tabular}{l c c c }
    \toprule
         $h$ & Data & DM & GAN   \\
         \midrule
         1 & 39.4 & 37.2 & 56.8 \\
         2 & 86.5 & 83.3 & 95.3 \\
         3 & 98.9 & 97.7 & 98.9 \\
         4 & 100  & 99.8 & 99.3 \\
         5 & 100  & 100  & 99.6 \\
         6 & 100  & 100  & 99.9 \\
         \bottomrule
    \end{tabular}
\end{table}

\textbf{Complexity.} Now that we have ensured that both the GAN as well as the DM fit the target distribution reasonably well, we can measure the complexity of the DM denoiser $D$, the generator defined by the GAN, as well as the end-to-end multi-step synthesis process (DDIM-100) of the DM. In particular, we measure the complexity of these functions using the Frobenius norm of the Jacobian~\citep{dockhorn2022scorebased}.
In particular, we define
\begin{align}
    \gJ_F(\sigma) = \E_{\rvx \sim p(\rvx, \sigma)} \| \nabla_\rvx D_\vtheta(\rvx, \sigma) \|_F^2.
\end{align}
Note that the convolution of a mixture of Gaussian with i.i.d. Gaussian noise is simply the sum of the convolution of the mixture components, i.e.,
\begin{align}
    p(\rvx; \sigma) &= \left(\pdata \ast \gN\left(\bm{0}, \sigma^2\right)\right)(\rvx) \\
    &=\sum_{k=1}^9 \frac{1}{9} \gN(\rvx; \vmu_k; \sigma_0^2 + \sigma^2).
\end{align}
We then compare %
$\gJ_F(\sigma)$ with the complexity of the GAN generator ($S_1$) and the end-to-end synthesis process of the DM ($S_2$). In particular, we define
\begin{align}
    \gJ_F = \E_{\rvx \sim \gN\left(\bm{0}, \mI \right)} \| \nabla_\rvx S_i(\rvx) \|_F^2, \quad i \in \{1, 2\}.
\end{align}
We want to clarify that for $S_2$ we do not have to backpropagate through an ODE but rather through its discretization, i.e., deterministic DDIM with 100 function evaluations (\Cref{alg:ddim}), since that is how we define the end-to-end multi-step synthesis process of the DM in this experiment. Furthermore, we chose the latent space of the GAN to be two-dimensional such that $\nabla_\rvx S_i(\rvx) \in \R^{2\times 2}$ for both the GAN and the DM; this ensures a fair comparison.
The final complexities are visualized in~\Cref{fig:complexity_input_resnet_small_not_wrapped}.

\subsection{Training Details} \label{sec:toy_training}
\textbf{DM training.} Training the diffusion model is very simple. We use the EDM config and train for 50k iterations (with batch size $B{=}256$) using Adam with learning rate $3 \cdot 10^{-4}$. We use an EMA rate of 0.999.

\textbf{GAN training.} Training GANs on two-dimensional mixture of Gaussians is notoriously difficult (see, for example, Sec. 5.1 in~\citep{ganema}). We experimented with several setups and found the following to perform well: We train for 50k iterations (with batch size $B{=}256$) using Adam with learning rate $3 \cdot 10^{-4}$ and ($\beta_1{=}0.0$, $\beta_2=0.9$) for both the generator and the discriminator. Following~\citet{ganema}, we use EMA (rate of 0.999 as in the DM). We found it crucial to make the discriminator bigger than the generator; in particular, we use twice as many hidden layers in the discriminator's ResNet. Furthermore, we use \texttt{ReLU} and \texttt{LeakyReLU} for the generator and the discriminator, respectively.
\section{Image Experiments} \label{sec:image_experiments}
\subsection{Evaluation Metrics, Baselines, and Datasets} \label{sec:metrics_baselines_datasets}
\textbf{Metrics.} We measure sample quality via Fréchet Inception Distance (FID)~\citep{heusel2017gans}. We follow the DP generation literature and use 60k generated samples. The particular Inception-v3 model used for FID computation is taken from~\citet{karras2021alias}\footnote{\scalebox{.8}{https://api.ngc.nvidia.com/v2/models/nvidia/research/stylegan3/versions/1/files/metrics/inception-2015-12-05.pkl}}. On MNIST and Fashion-MNIST, we follow the standard procedure of repeating the channel dimension three times before feeding images into the Inception-v3 model. 

On MNIST and Fashion-MNIST, we additionally assess the utility of generated data by training classifiers on synthesized samples and compute class prediction accuracy on real data. Similar to previous works, we consider three classifiers: logistic regression (Log Reg), MLP, and CNN classifiers. The model architectures are taken from the \href{https://github.com/nv-tlabs/DP-Sinkhorn_code}{\texttt{DP-Sinkhorn}} repository~\citep{cao2021don}.

For downstream classifier training, we follow the DP generation literature and use 60k synthesized samples. We follow~\citet{cao2021don} and split the 60k samples into a training set (90\%) and a validation set (remaining 10\%). We train all models for 50 epochs, using Adam with learning rate $3 \cdot 10^{-4}$. We regularly save checkpoints during training and use the checkpoint that achieves the best accuracy on the validation split for final evaluation.
Final evaluation is performed on real, non-synthetic data. We train all models for 50 epochs, using Adam with learning rate $3 \cdot 10^{-4}$. 

\TIM{Note that we chose to only use 60k synthesized samples to follow prior work and therefore be able to compare to baselines in a fair manner. That being said, during this project we did explore training classifiers with more samples but did not find any significant improvements in downstream accuracy. We hypothesize that 60k samples are enough to accurately represent the underlying learned distribution by the DPDM and to train good classifiers on MNIST/FashionMNIST. We believe that a more detailed study on how many samples are needed to get a certain accuracy is an interesting avenue for future work.}

\textbf{Baselines.} We run baseline experiments for \href{https://anonymous.4open.science/r/pearl-518E/README.md}{PEARL}~\citep{liew2022pearl}. In particular, we train models for $\varepsilon{=}\{0.2, 1, 10\}$ on MNIST and Fashion-MNIST. We confirmed that our models match the performance reported in their paper. In fact, our models perform slightly better (in terms of the LeNet-FID metric~\citet{liew2022pearl} uses). We then follow the same evaluation setup (see \textbf{Metrics} above) as for our DPDMs. Most importantly, we use the standard Inception network-based FID calculation, similarly as most works in the (DP) image generative modeling literature.

\textbf{Datasets.} We use three datasets in our main experiments: MNIST~\citep{lecun2010mnist}, Fashion-MNIST~\citep{xiao2017fashion} and CelebA~\citep{liu2015deep}. \TC{Furthermore, we provide initial results on CIFAR-10~\citep{krizhevsky2009learning} and ImageNet~\citep{deng2009imagenet} (as well as CelebA on a higher resolution); see~\Cref{sec:add_experiments}.}

\TC{We would like to point out that these dataset may contain multiple images per identity (e.g. person, animal, etc.), whereas our method, as well as all other baselines in this work, considers the per-image privacy guarantee. For an identity with $k$ images in the dataset, a model with $(\varepsilon, \delta)$ per-image DP affords $(k \varepsilon, k e^{(k-1)\varepsilon} \delta)$-DP to the individual according to the Group Privacy theorem~\citep{dwork2014algorithmic}. We leave a more rigorous study of DPDMs with Group Privacy to future research and note that these datasets currently simply serve as benchmarks in the community. Nonetheless, we believe that it is important to point out that these datasets do not necessarily serve as a realistic test bed for per-image DP generative models in privacy critical applications.}

\subsection{Computational Resources} \label{sec:computational_resources}
For all experiments, we use an in-house GPU cluster of V100 NVIDIA GPUs. On eight GPUs, models on MNIST and Fashion-MNIST trained for roughly one day and models on CelebA for roughly four days. We tried to maximize performance by using a large number of epochs, which results in a good privacy-utility trade-off, as well as high noise multiplicity; this results in relatively high training time (when compared to existing DP generative models). Using a smaller noise multiplicity $K$ decreases computation, although generally at the cost of model performance; see also~\Cref{sec:noise_multiplicity_cost}.
\subsection{Training DP-SGD Classifiers} \label{sec:training_dp_cnn_classifiers}
We train classifiers on MNIST and Fashion-MNIST using DP-SGD directly. We follow the setup used for training DPDMs, in particular, batchsize $B=4096$, 300 epochs and clipping constant $C=1$. Recently, \citet{de2022unlocking} found EMA to be helpful in training image classifiers: we follow this suggestion and use an EMA rate of 0.999 (same rate as used for training DPDMs).
\subsection{Extended Quantitative Results} \label{sec:extended_quantitative_results}
In this section, we show additional quantitative results not presented in the main paper. In particular, we present extended results for all ablation experiments. 
\subsubsection{Noise Multiplicity} \label{sec:noise_augmentation_ablation}
In the main paper, we present noise multiplicity ablation results on MNIST with $\varepsilon{=}1$ (\Cref{tab:noise_augmentation_ablation_fid_cnn_acc_no_wrapped}). All results for MNIST and Fashion-MNIST on all three privacy settings ($\varepsilon{=}\{0.2, 1, 10\}$) can be found in~\Cref{tab:noise_augmentation_ablation}.
\begin{table}
    \caption{Noise multiplicity ablation on MNIST and Fashion-MNIST.}
    \label{tab:noise_augmentation_ablation}
    \centering
    \scalebox{1.0}{
        \begin{tabular}{l c c c c c c c c}
            \toprule
            \multirow{3}{*}{$K$} & \multicolumn{4}{c}{MNIST} & \multicolumn{4}{c}{Fashion-MNIST} \\
            \cmidrule(l{1em}r{1em}){2-5} \cmidrule(l{1em}r{1em}){6-9}
            & \multirow{2}{*}{FID} & \multicolumn{3}{c}{Acc (\%)} & \multirow{2}{*}{FID} & \multicolumn{3}{c}{Acc (\%)} \\
            \cmidrule(l{1em}r{1em}){3-5} \cmidrule(l{1em}r{1em}){7-9}
            & & Log Reg & MLP & CNN & & Log Reg & MLP & CNN \\
            \midrule
            1  & 76.9 & 84.2 & 87.5 & 91.7 & 72.5 & 76.0 & 76.3 & 75.9 \\
            2  & 60.1 & 84.8 & 88.3 & 93.1 & 61.4 & 76.7 & 77.0 & 77.4 \\
            4  & 57.1 & 85.2 & 88.0 & 92.8 & 61.1 & 76.7 & 77.2 & 77.0 \\
            8  & 44.8 & 86.2 & 89.2 & 94.1 & 58.2 & 75.2 & 76.3 & 77.4 \\
            16 & 36.9 & 86.0 & 89.8 & 94.2 & 58.5 & 77.0 & 77.4 & 78.8 \\
            32 & 34.8 & 86.8 & 90.1 & 94.4 & 57.7 & 76.4 & 77.0 & 77.1 \\
            \bottomrule
        \end{tabular}
    }
\end{table}
\subsubsection{Diffusion Model Config} \label{sec:diffusion_model_backbone_ablation}
In the main paper, we present DM config ablation results on MNIST with $\varepsilon{=}0.2$ (\Cref{tab:noise_augmentation_ablation_fid_cnn_acc_no_wrapped}). All results for MNIST and Fashion-MNIST on all three privacy settings ($\varepsilon{=}\{0.2, 1, 10\}$) can be found in~\Cref{tab:design_choice}.
\begin{table}
    \caption{DM config ablation.}
    \label{tab:design_choice}
    \centering
    \scalebox{0.82}{
        \begin{tabular}{l c c c c c c c c c}
            \toprule
            \multirow{3}{*}{Method} & \multirow{3}{*}{DP-$\varepsilon$} & \multicolumn{4}{c}{MNIST} & \multicolumn{4}{c}{Fashion-MNIST} \\
            \cmidrule(l{1em}r{1em}){3-6} \cmidrule(l{1em}r{1em}){7-10}
            & & \multirow{2}{*}{FID} & \multicolumn{3}{c}{Acc (\%)} & \multirow{2}{*}{FID} & \multicolumn{3}{c}{Acc (\%)} \\
            \cmidrule(l{1em}r{1em}){4-6} \cmidrule(l{1em}r{1em}){8-10}
            & & & Log Reg & MLP & CNN & & Log Reg & MLP & CNN \\
            \midrule 
            VP~\citep{song2020} & 0.2 & 197 & 23.1 & 25.5 & 24.2 & 146 & 49.7 & 51.6 & 51.7 \\
            VE~\citep{song2020} & 0.2 & 171 & 17.9 & 15.4 & 13.9 & 178 & 22.2 & 27.9 & 49.4 \\
            V-prediction~\citep{salimans2022progressive} & 0.2 & 97.8 & 80.2 & 81.3 & 84.4 & 115 & 71.3 & 70.9 & 71.8 \\
            EDM~\citep{karras2022elucidating} & 0.2 & 119 & 62.4 & 67.3 & 49.2 & 93.5 & 64.7 & 65.9 & 66.6\\
            \midrule
            VP~\citep{song2020} & 1 & 82.2 & 59.4 & 69.3 & 72.6 & 73.4 & 68.3 & 70.4 & 72.7 \\
            VE~\citep{song2020} & 1 & 165 & 17.9 & 20.5 & 26.0 & 156 & 30.7 & 36.0 & 49.8 \\
            V-prediction~\citep{salimans2022progressive} & 1 & 34.8 & 86.8 & 90.1 & 94.4 & 57.7 & 76.4 & 77.0 & 77.1 \\
            EDM~\citep{karras2022elucidating} & 1 & 34.2 & 86.2 & 90.1 & 94.9 & 47.1 & 77.4 & 78.0 & 79.4\\
            \midrule
            VP~\citep{song2020} & 10 & 12.3 & 88.8 & 94.1 & 97.0 & 22.3 & 81.2 & 81.6 & 84.5 \\
            VE~\citep{song2020} & 10 & 88.6 & 48.0 & 56.9 & 63.8 & 83.2 & 69.0 & 70.4 & 75.4 \\
            V-prediction~\citep{salimans2022progressive} & 10 & 7.65 & 90.4 & 94.4 & 97.7 & 23.1 & 82.0 & 83.7 & 85.5\\
            EDM~\citep{karras2022elucidating} & 10 & 6.13 & 90.4 & 94.6 & 97.5 & 17.4 & 82.6 & 84.1 & 86.2 \\
            \bottomrule
        \end{tabular}
    }
\end{table}

\subsubsection{Diffusion Sampler Grid Search and Ablation} \label{sec:diffusion_sampler}
\textbf{Churn sampler grid search.} We run a small grid search for the hyperparameters of the Churn sampler (together with the guidance weight $w$ for classifier-free guidance). For MNIST and Fashion-MNIST on $\varepsilon{=}0.2$ we run a two-stage grid search. Using $S_\mathrm{min}{=}0.05$, $S_\mathrm{max}{=}50$, and $S_\mathrm{noise}{=}1$, which we found to be sensible starting values, we ran an initial grid search over $w{=}\{0, 0.125, 0.25, 0.5, 1.0, 2.0\}$ and $S_\mathrm{churn}{=}\{0, 5, 10, 25, 50, 100, 150, 200\}$, which we found to be the two most critical hyperparameters of the Churn sampler. Afterwards, we ran a second grid search over $S_\mathrm{noise}{=}\{1, 1.005\}$, $S_\mathrm{min}{=}\{0.01, 0.02, 0.05, 0.1, 0.2\}$, and $S_\mathrm{max}{=}\{10, 50, 80\}$ using the best $(w, S_\mathrm{churn})$ setting for each of the two models. For MNIST and Fashion-MNIST on $\varepsilon{=}\{1, 10\}$, we ran a single full grid search over $w{=}\{0, 0.25, 0.5, 1.0, 2.0\}$, $S_\mathrm{churn}{=}\{10, 25, 50, 100\}$, and $S_\mathrm{min}{=}\{0.025, 0.05, 0.1, 0.2\}$ while setting $S_\mathrm{noise}{=}1$. For CelebA, on both $\varepsilon{=}1$ and $\varepsilon{=}10$, we also ran a single full grid search over $S_\mathrm{churn}{=}\{50, 100, 150, 200\}$, and $S_\mathrm{min}{=}\{0.005, 0.05\}$ while setting $S_\mathrm{noise}{=}1$. The best settings for FID metric and downstream CNN accuracy can be found in~\Cref{tab:sampler_settings_for_best_models_fid} and~\Cref{tab:sampler_settings_for_best_models_acc}, respectively.

Throughout all experiments we found two consistent trends that are listed in the following:
\begin{itemize}
    \item If optimizing for FID, set $S_\mathrm{churn}$ relatively high and $S_\mathrm{min}$ relatively small. Increase $S_\mathrm{churn}$ and decrease $S_\mathrm{min}$ as $\varepsilon$ is decreased.
    \item If optimizing for downstream accuracy, set $S_\mathrm{churn}$ relatively small and $S_\mathrm{min}$ relatively high.
\end{itemize}

\textbf{Sampling ablation.} In the main paper, we present a sampler ablation for MNIST (\Cref{tab:ddim_sddim_churn_mnist_new}). Results for Fashion-MNIST (as well as) MNIST can be found in~\Cref{tab:ddim_sddim_churn}.

\begin{table}
    \caption{Diffusion sampler comparison. We compare the Churn sampler~\citep{karras2022elucidating} to stochastic and determistic DDIM~\citep{song2021denoising}.}
    \label{tab:ddim_sddim_churn}
    \centering
    \scalebox{0.9}{
        \begin{tabular}{l c c c c c c c c c}
            \toprule
            \multirow{3}{*}{Sampler} & \multirow{3}{*}{DP-$\varepsilon$} & \multicolumn{4}{c}{MNIST} & \multicolumn{4}{c}{Fashion-MNIST} \\
            \cmidrule(l{1em}r{1em}){3-6} \cmidrule(l{1em}r{1em}){7-10}
            & & \multirow{2}{*}{FID} & \multicolumn{3}{c}{Acc (\%)} & \multirow{2}{*}{FID} & \multicolumn{3}{c}{Acc (\%)} \\
            \cmidrule(l{1em}r{1em}){4-6} \cmidrule(l{1em}r{1em}){8-10}
            & & & Log Reg & MLP & CNN & & Log Reg & MLP & CNN \\
            \midrule 
            Churn (FID) & 0.2 & \textbf{61.9} & 65.3 & 65.8 & 71.9 & \textbf{78.4} & 53.6 & 55.3 & 57.0 \\
            Churn (Acc) & 0.2 & 104 & 81.0 & 81.7 & \textbf{86.3} & 128 & 70.4 & 71.3 & \textbf{72.3} \\
            Stochastic DDIM       & 0.2 & 97.8 & 80.2 & 81.3 & 84.4 & 115 & 71.3 & 70.9 & 71.8 \\
            Deterministic DDIM    & 0.2 & 120 & \textbf{81.3} & \textbf{82.1} & 84.8 & 132 & \textbf{71.5} & \textbf{71.6} & 71.8 \\
            \midrule
            Churn (FID) & 1 & \textbf{23.4} & 83.8 & 87.0 & 93.4 & \textbf{37.8} & 71.5 & 71.7 & 73.6 \\
            Churn (Acc) & 1 & 35.5 & \textbf{86.7} & 91.6 & \textbf{95.3} & 51.4 & 76.3 & 76.9 & \textbf{79.4} \\
            Stochastic DDIM       & 1 & 34.2 & 86.2 & 90.1 & 94.9 & 47.1 & 77.4 & 78.0 & \textbf{79.4} \\
            Deterministic DDIM    & 1 & 50.4 & 85.7 & \textbf{91.8} & 94.9 & 60.6 & \textbf{77.5} & \textbf{78.2} & 78.9 \\
            \midrule
            Churn (FID) & 10 & \textbf{5.01} & 90.5 & 94.6 & 97.3 & 18.6 & 80.4 & 81.1 & 84.9 \\
            Churn (Acc) & 10 & 6.65 & \textbf{90.8} & 94.8 & \textbf{98.1} & 19.1 & 81.1 & 83.0 & \textbf{86.2} \\
            Stochastic DDIM       & 10 & 6.13 & 90.4 & 94.6 & 97.5 & \textbf{17.4} & \textbf{82.6} & \textbf{84.1} & \textbf{86.2} \\
            Deterministic DDIM    & 10 & 10.9 & 90.5 & \textbf{95.2} & 97.7 & 19.7 & 81.9 & 83.9 & \textbf{86.2} \\
            \bottomrule
        \end{tabular}
    }
\end{table}

\begin{table}
    \caption{Best Churn sampler settings for FID metric.}
    \label{tab:sampler_settings_for_best_models_fid}
    \centering
    \begin{tabular}{l c c c c c c c c}
        \toprule
        \multirow{2}{*}{Parameter} & \multicolumn{3}{c}{MNIST} & \multicolumn{3}{c}{Fashion-MNIST} & \multicolumn{2}{c}{CelebA} \\
        \cmidrule(l{1em}r{1em}){2-4} \cmidrule(l{1em}r{1em}){5-7} \cmidrule(l{1em}r{1em}){8-9}
        & $\varepsilon{=}0.2$ & $\varepsilon{=}1$ & $\varepsilon{=}10$ & $\varepsilon{=}0.2$ & $\varepsilon{=}1$ & $\varepsilon{=}10$ & $\varepsilon{=}1$ & $\varepsilon{=}10$ \\
        \midrule 
        $ w $              & 1    & 0    & 0.25 & 2    & 1     & 0.25   & N/A   & N/A  \\
        $S_\mathrm{churn}$ & 200  & 100  & 50   & 150  & 50    & 25     & 200   & 50 \\
        $S_\mathrm{min}$   & 0.01 & 0.05 & 0.05 & 0.02 & 0.025 & 0.2    & 0.005 & 0.005 \\
        $S_\mathrm{max}$   & 50   & 50   & 50   & 10   & 50    & 50     & 50    & 50  \\
        $S_\mathrm{noise}$ & 1    & 1    & 1    & 1    &  1    & 1      & 1     & 1 \\
        \bottomrule
    \end{tabular}
\end{table}
\begin{table}
    \caption{Best Churn sampler settings for downstream CNN accuracy.}
    \label{tab:sampler_settings_for_best_models_acc}
    \centering
    \begin{tabular}{l c c c c c c}
        \toprule
        \multirow{2}{*}{Parameter} & \multicolumn{3}{c}{MNIST} & \multicolumn{3}{c}{Fashion-MNIST} \\
        \cmidrule(l{1em}r{1em}){2-4} \cmidrule(l{1em}r{1em}){5-7}
        & $\varepsilon{=}0.2$ & $\varepsilon{=}1$ & $\varepsilon{=}10$ & $\varepsilon{=}0.2$ & $\varepsilon{=}1$ & $\varepsilon{=}10$ \\
        \midrule 
        $ w $              & 0.125 & 0    & 0     & 0.125 & 0     & 0    \\
        $S_\mathrm{churn}$ & 10    & 10   & 10    & 5     & 10    & 10    \\
        $S_\mathrm{min}$   & 0.2   & 0.1  & 0.025 & 0.02  & 0.025 & 0.1  \\
        $S_\mathrm{max}$   & 10    & 50   & 50    & 80    & 50    & 50  \\
        $S_\mathrm{noise}$ & 1.005 & 1    & 1     & 1.005 & 1     & 1   \\
        \bottomrule
    \end{tabular}
\end{table}

\subsubsection{Distribution Matching Analysis}
We perform a distribution matching analysis on MNIST using the CNN classifier, that is, computing per-class downstream accuracies at different privacy levels. In~\Cref{tab:distribution_matching_analysis}, we can see that with increased privacy (lower $\varepsilon$) classification performance degrades roughly similar for most digits, implying that our DPDMs learn a well-balanced distribution and cover all modes of the data distribution faithfully even under strong privacy. The only result that stands out to us is that class 1 appears significantly easier than all other classes for $\varepsilon {=} 0.2$. However, that may potentially be due to the class 1 in MNIST being a ``line'' which may be easy to classify by a CNN even if training data is noisy.
\begin{table}
    \centering
    \caption{Distribution matching analysis for MNIST using downstream CNN accuracy.}
    \begin{tabular}{l c c c c c c c c c c}
    \toprule
    Class &0 &1 &2 &3 &4 &5 &6 &7 &8 &9 \\
    \midrule
    $\varepsilon{=}10$ & 	98.5	&98.7	&98.7	&98.4	&99.1	&99.0	&98.2	&97.7	&98.1	&96.1\\
    $\varepsilon{=}1$ & 95.4	&98.9	&96.7	&96.1	&96.1	&97.3	&96.6	&91.3	&92.3	&93.2\\
    $\varepsilon{=}0.2$ & 81.9	&96.0	&80.2	&82.3	&87.3	&85.2	&90.7	&86.9	&83.2	 &83.5\\
    \bottomrule
    \end{tabular}
    \label{tab:distribution_matching_analysis}
\end{table}

\subsection{Extended Qualitative Results} \label{sec:extended_qualitative_results}
In this section, we show additional generated samples by our DPDMs. On MNIST, see~\Cref{fig:mnist_eps_10}, ~\Cref{fig:mnist_eps_1}, and ~\Cref{fig:mnist_eps_0_2} for $\varepsilon{=}10$, $\varepsilon{=}1$, and $\varepsilon{=}0.2$, respectively. On Fashion-MNIST, see~\Cref{fig:fmnist_eps_10}, ~\Cref{fig:fmnist_eps_1}, and ~\Cref{fig:fmnist_eps_0_2} for $\varepsilon{=}10$, $\varepsilon{=}1$, and $\varepsilon{=}0.2$, respectively. On CelebA, see~\Cref{fig:celeba_eps_10_app} and~\Cref{fig:celeba_eps_1_app} for $\varepsilon{=}10$ and $\varepsilon{=}1$, respectively. \TC{For a visual comparison of our CelebA samples to other works in the literature, see~\Cref{fig:celeba_comparison}.}
\begin{figure}
    \centering
    \begin{subfigure}[b]{0.47\textwidth}
        \centering
        \includegraphics[scale=0.4]{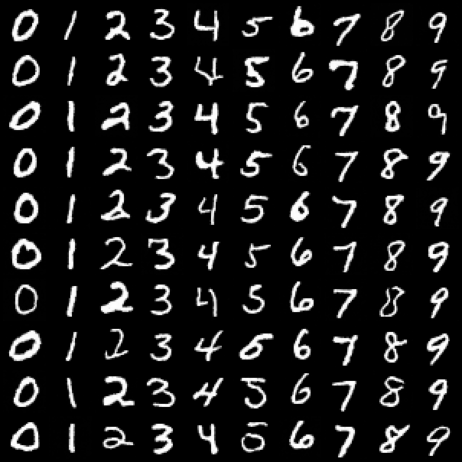}
    \end{subfigure}
    \hskip -5em
    \begin{subfigure}[b]{0.47\textwidth}
        \centering
        \includegraphics[scale=0.4]{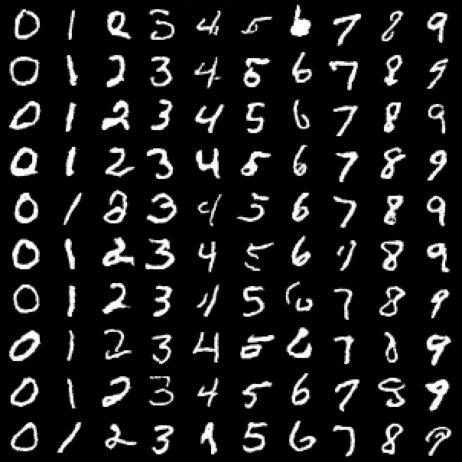}
    \end{subfigure}
    \par\vskip 0.5em
    \begin{subfigure}[b]{0.47\textwidth}
        \centering
        \includegraphics[scale=0.4]{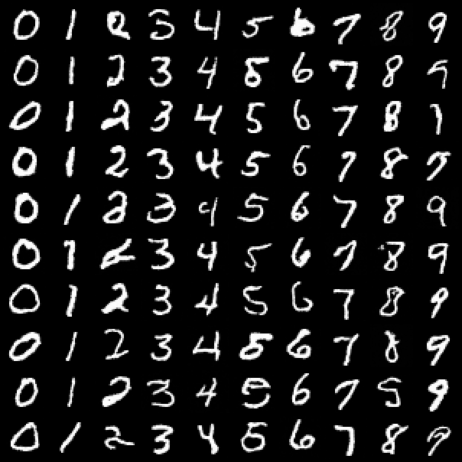}
    \end{subfigure}
    \hskip -5em
    \begin{subfigure}[b]{0.47\textwidth}
        \centering
        \includegraphics[scale=0.4]{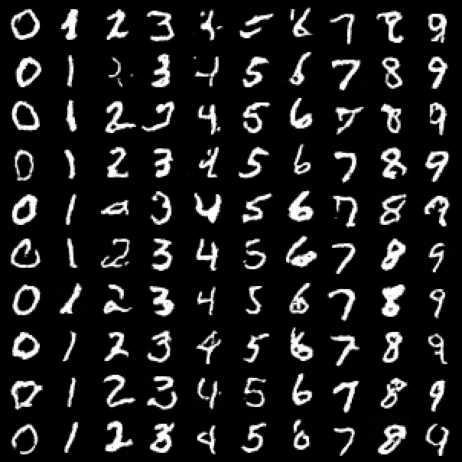}
    \end{subfigure}
    \caption{Additional images generated by DPDM on MNIST for $\varepsilon{=}10$ using Churn (FID) (\emph{top left}), Churn (Acc) (\emph{top right}), stochastic DDIM (\emph{bottom left}), and deterministic DDIM (\emph{bottom right}).}
    \label{fig:mnist_eps_10}
\end{figure}
\begin{figure}
    \centering
    \begin{subfigure}[b]{0.47\textwidth}
        \centering
        \includegraphics[scale=0.4]{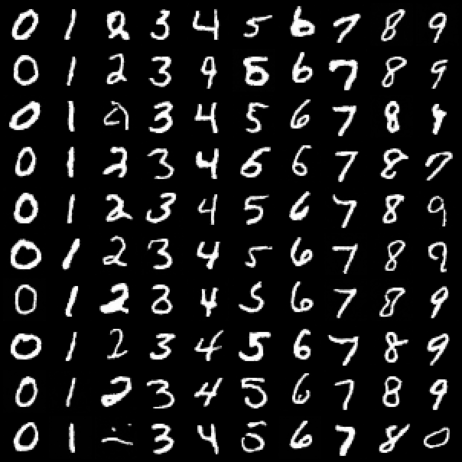}
    \end{subfigure}
    \hskip -5em
    \begin{subfigure}[b]{0.47\textwidth}
        \centering
        \includegraphics[scale=0.4]{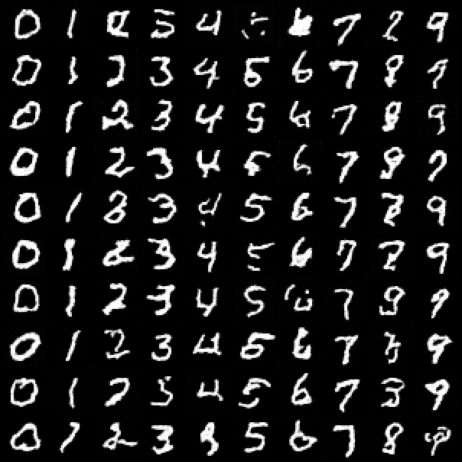}
    \end{subfigure}
    \par\vskip 0.5em
    \begin{subfigure}[b]{0.47\textwidth}
        \centering
        \includegraphics[scale=0.4]{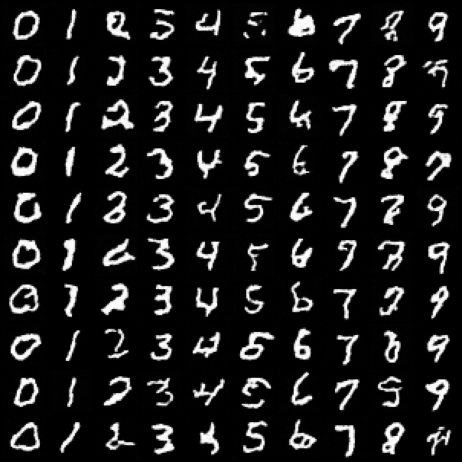}
    \end{subfigure}
    \hskip -5em
    \begin{subfigure}[b]{0.47\textwidth}
        \centering
        \includegraphics[scale=0.4]{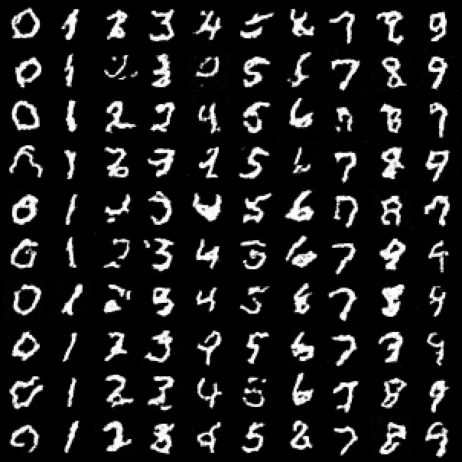}
    \end{subfigure}
    \caption{Additional images generated by DPDM on MNIST for $\varepsilon{=}1$ using Churn (FID) (\emph{top left}), Churn (Acc) (\emph{top right}), stochastic DDIM (\emph{bottom left}), and deterministic DDIM (\emph{bottom right}).}
    \label{fig:mnist_eps_1}
\end{figure}
\begin{figure}
    \centering
    \begin{subfigure}[b]{0.47\textwidth}
        \centering
        \includegraphics[scale=0.4]{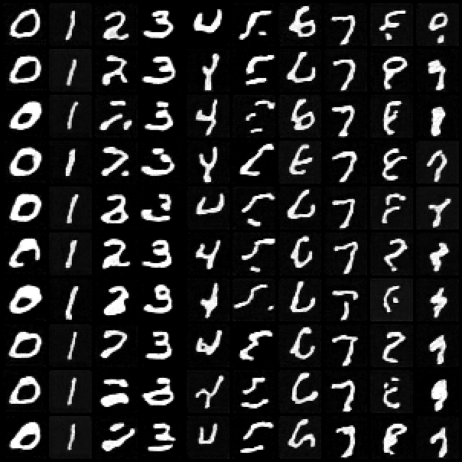}
    \end{subfigure}
    \hskip -5em
    \begin{subfigure}[b]{0.47\textwidth}
        \centering
        \includegraphics[scale=0.4]{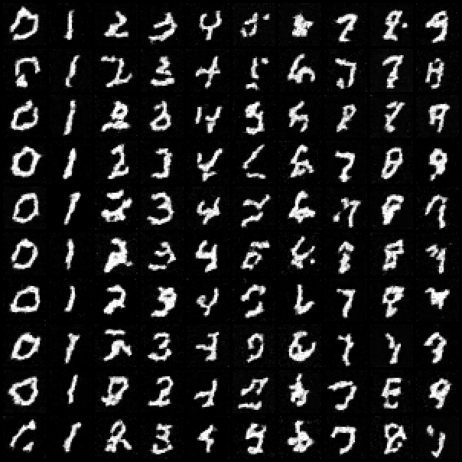}
    \end{subfigure}
    \par\vskip 0.5em
    \begin{subfigure}[b]{0.47\textwidth}
        \centering
        \includegraphics[scale=0.4]{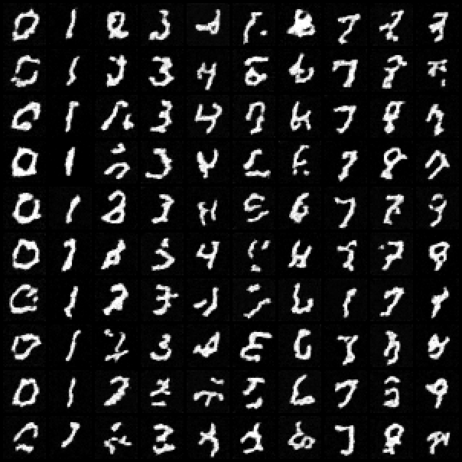}
    \end{subfigure}
    \hskip -5em
    \begin{subfigure}[b]{0.47\textwidth}
        \centering
        \includegraphics[scale=0.4]{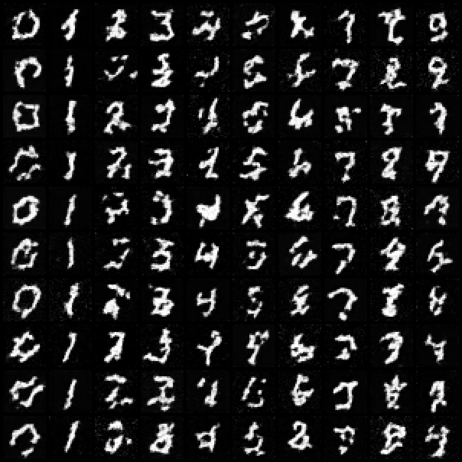}
    \end{subfigure}
    \caption{Additional images generated by DPDM on MNIST for $\varepsilon{=}0.2$ using Churn (FID) (\emph{top left}), Churn (Acc) (\emph{top right}), stochastic DDIM (\emph{bottom left}), and deterministic DDIM (\emph{bottom right}).}
    \label{fig:mnist_eps_0_2}
\end{figure}
\begin{figure}
    \centering
    \begin{subfigure}[b]{0.47\textwidth}
        \centering
        \includegraphics[scale=0.4]{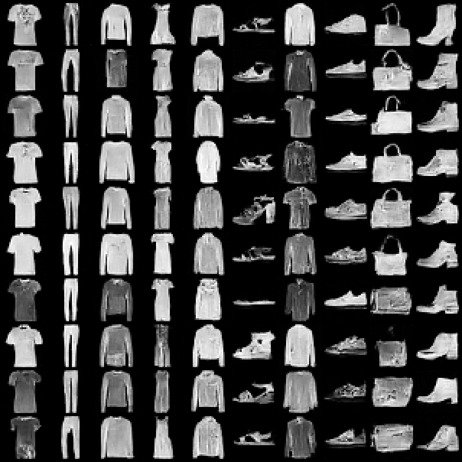}
    \end{subfigure}
    \hskip -5em
    \begin{subfigure}[b]{0.47\textwidth}
        \centering
        \includegraphics[scale=0.4]{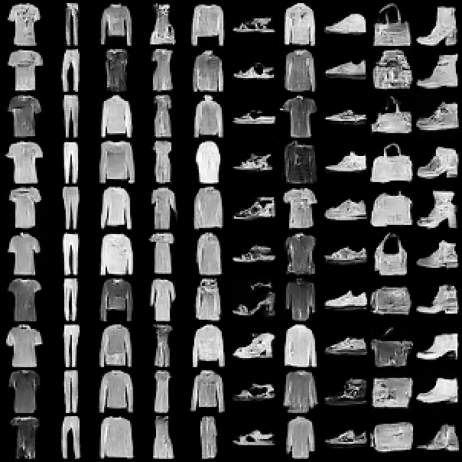}
    \end{subfigure}
    \par\vskip 0.5em
    \begin{subfigure}[b]{0.47\textwidth}
        \centering
        \includegraphics[scale=0.4]{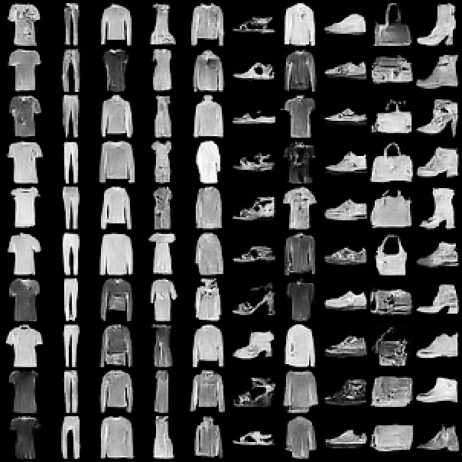}
    \end{subfigure}
    \hskip -5em
    \begin{subfigure}[b]{0.47\textwidth}
        \centering
        \includegraphics[scale=0.4]{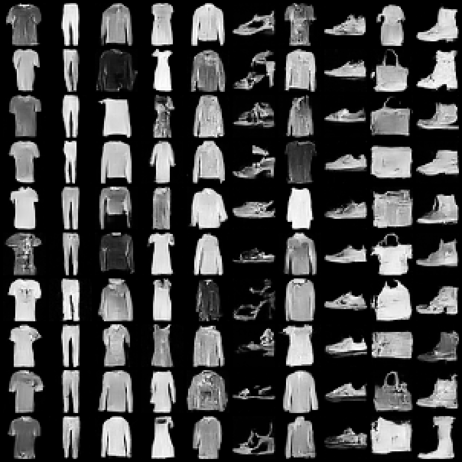}
    \end{subfigure}
    \caption{Additional images generated by DPDM on Fashion-MNIST for $\varepsilon{=}10$ using Churn (FID) (\emph{top left}), Churn (Acc) (\emph{top right}), stochastic DDIM (\emph{bottom left}), and deterministic DDIM (\emph{bottom right}).}
    \label{fig:fmnist_eps_10}
\end{figure}
\begin{figure}
    \centering
    \begin{subfigure}[b]{0.47\textwidth}
        \centering
        \includegraphics[scale=0.4]{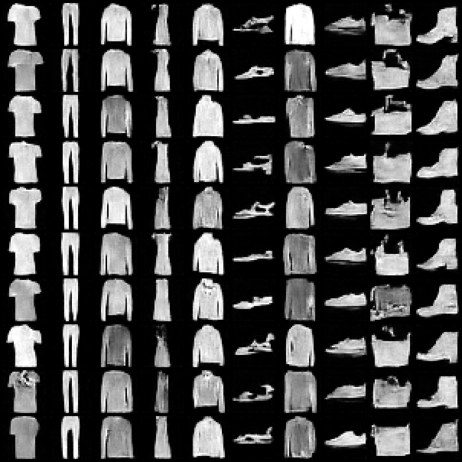}
    \end{subfigure}
    \hskip -5em
    \begin{subfigure}[b]{0.47\textwidth}
        \centering
        \includegraphics[scale=0.4]{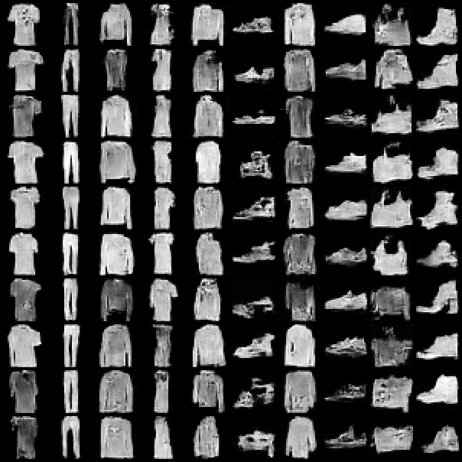}
    \end{subfigure}
    \par\vskip 0.5em
    \begin{subfigure}[b]{0.47\textwidth}
        \centering
        \includegraphics[scale=0.4]{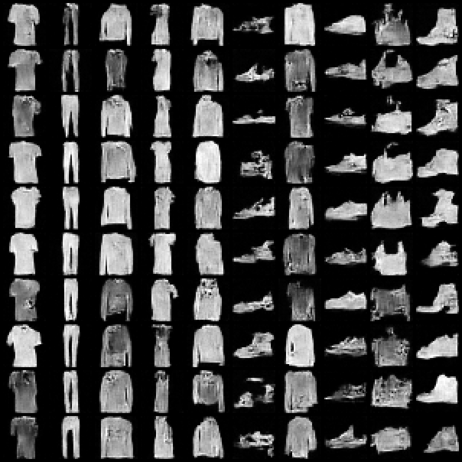}
    \end{subfigure}
    \hskip -5em
    \begin{subfigure}[b]{0.47\textwidth}
        \centering
        \includegraphics[scale=0.4]{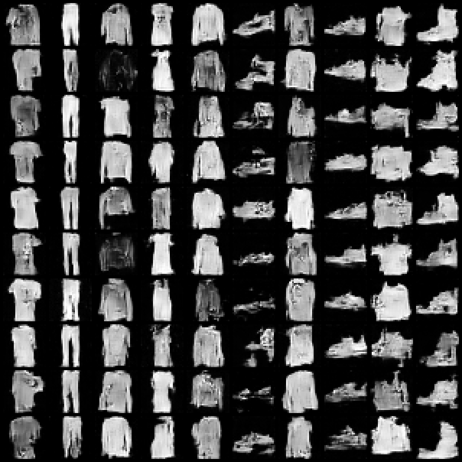}
    \end{subfigure}
    \caption{Additional images generated by DPDM on Fashion-MNIST for $\varepsilon{=}1$ using Churn (FID) (\emph{top left}), Churn (Acc) (\emph{top right}), stochastic DDIM (\emph{bottom left}), and deterministic DDIM (\emph{bottom right}).}
    \label{fig:fmnist_eps_1}
\end{figure}
\begin{figure}
    \centering
    \begin{subfigure}[b]{0.47\textwidth}
        \centering
        \includegraphics[scale=0.4]{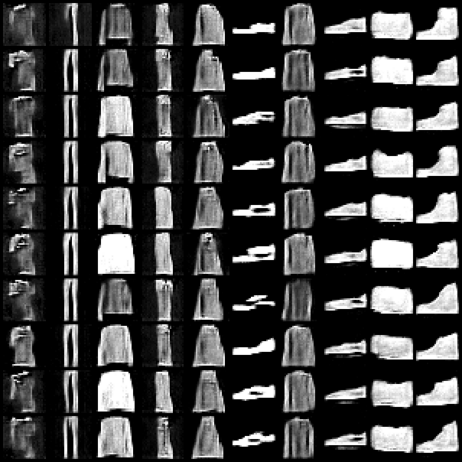}
    \end{subfigure}
    \hskip -5em
    \begin{subfigure}[b]{0.47\textwidth}
        \centering
        \includegraphics[scale=0.4]{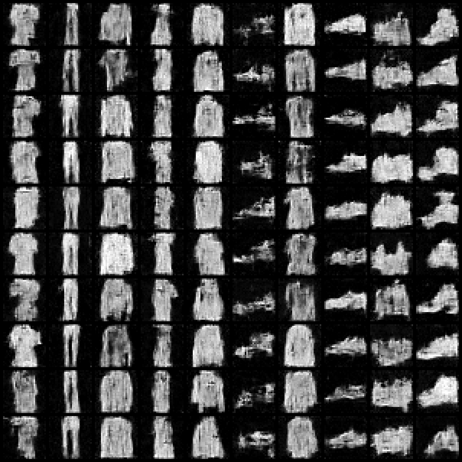}
    \end{subfigure}
    \par\vskip 0.5em
    \begin{subfigure}[b]{0.47\textwidth}
        \centering
        \includegraphics[scale=0.4]{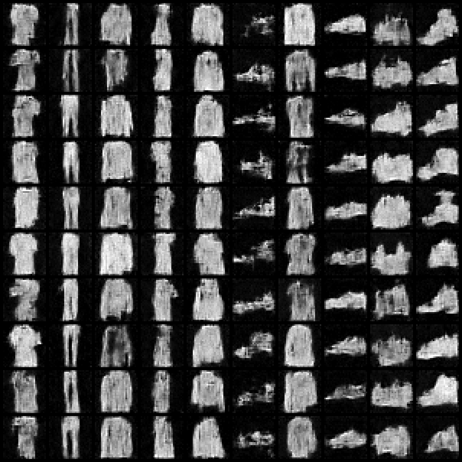}
    \end{subfigure}
    \hskip -5em
    \begin{subfigure}[b]{0.47\textwidth}
        \centering
        \includegraphics[scale=0.4]{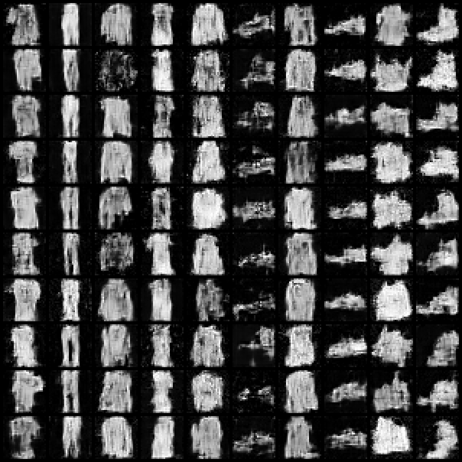}
    \end{subfigure}
    \caption{Additional images generated by DPDM on Fashion-MNIST for $\varepsilon{=}0.2$ using Churn (FID) (\emph{top left}), Churn (Acc) (\emph{top right}), stochastic DDIM (\emph{bottom left}), and deterministic DDIM (\emph{bottom right}).}
    \label{fig:fmnist_eps_0_2}
\end{figure}
\begin{figure}
    \centering
    \begin{subfigure}[b]{\textwidth}
        \centering
        \includegraphics[scale=0.57]{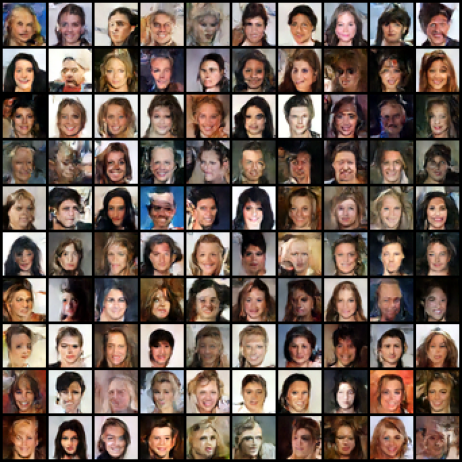}
    \end{subfigure}
    \par\vskip 1em
    \begin{subfigure}[b]{\textwidth}
        \centering
        \includegraphics[scale=0.57]{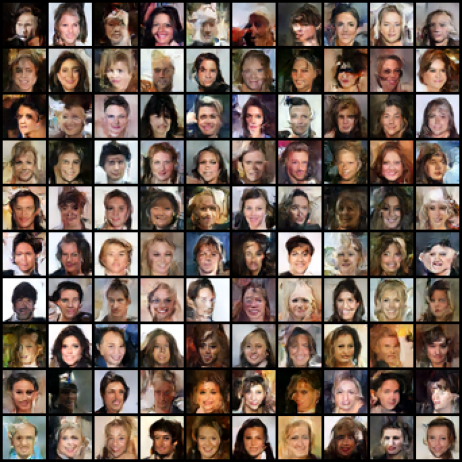}
    \end{subfigure}
    \par\vskip 1em
        \begin{subfigure}[b]{\textwidth}
        \centering
        \includegraphics[scale=0.57]{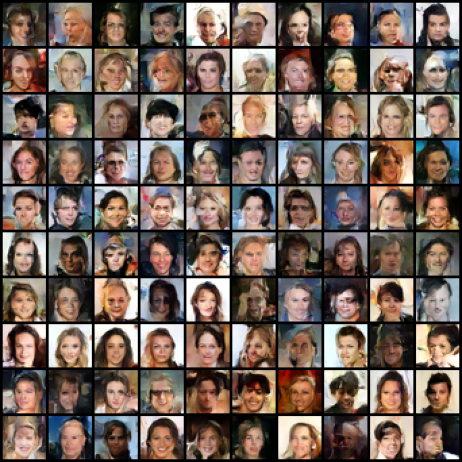}
    \end{subfigure}
    \caption{Additional images generated by DPDM on CelebA for $\varepsilon{=}10$ using Churn (\emph{top}), stochastic DDIM (\emph{middle}), and deterministic DDIM (\emph{bottom}).}
    \label{fig:celeba_eps_10_app}
\end{figure}
\begin{figure}
    \centering
    \begin{subfigure}[b]{\textwidth}
        \centering
        \includegraphics[scale=0.57]{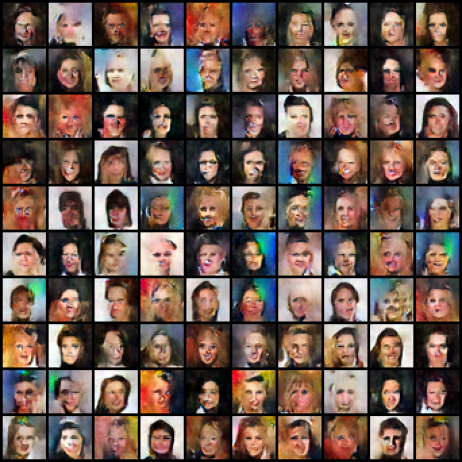}
    \end{subfigure}
    \par\vskip 1em
    \begin{subfigure}[b]{\textwidth}
        \centering
        \includegraphics[scale=0.57]{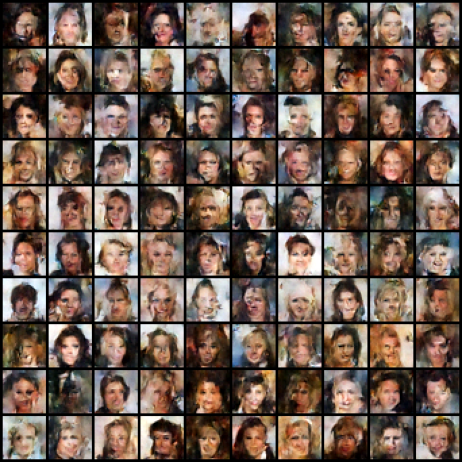}
    \end{subfigure}
    \par\vskip 1em
        \begin{subfigure}[b]{\textwidth}
        \centering
        \includegraphics[scale=0.57]{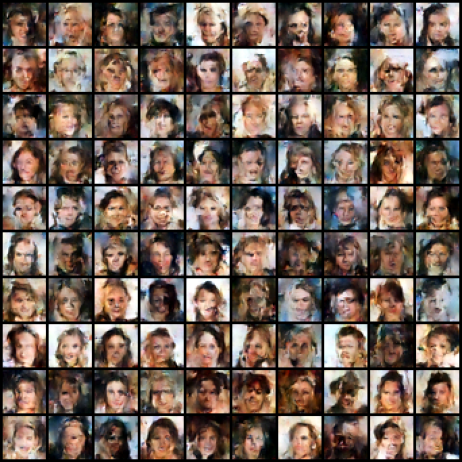}
    \end{subfigure}
    \caption{Additional images generated by DPDM on CelebA for $\varepsilon{=}1$ using Churn (\emph{top}), stochastic DDIM (\emph{middle}), and deterministic DDIM (\emph{bottom}).}
    \label{fig:celeba_eps_1_app}
\end{figure}
\begin{figure}
    \centering
    \begin{minipage}[c]{0.55\textwidth}
    \includegraphics[width=\textwidth]{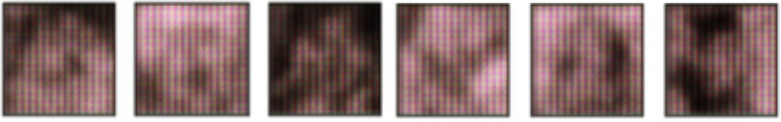}
    \includegraphics[trim=0 23 162 0,clip,width=0.495\textwidth]{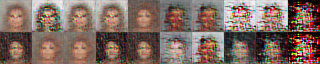}\hspace{-1pt}
    \includegraphics[trim=162 23 0 0,clip,width=0.495\textwidth]{images/dp_sinkhorn/celeba/samples_ep20.png}
    \includegraphics[trim=5 560 3820 5,clip,width=0.495\textwidth]{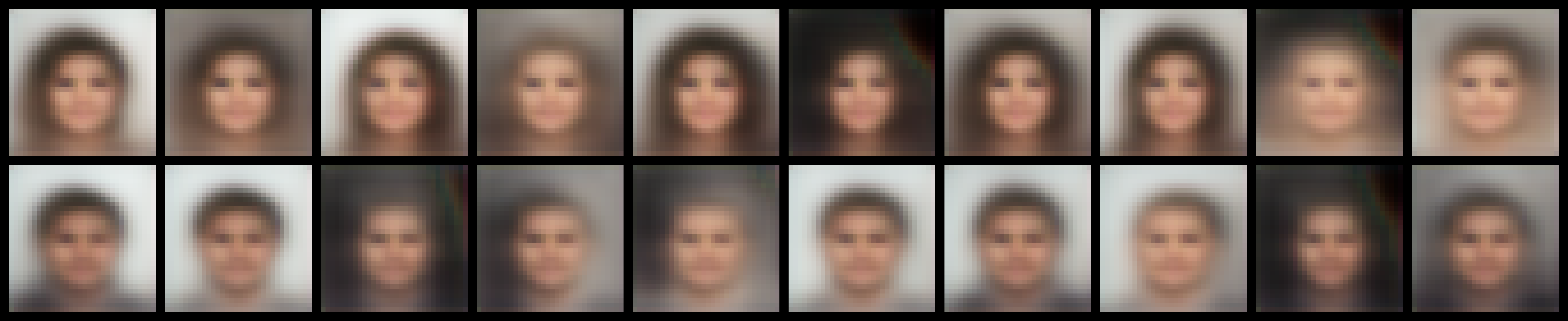}\hspace{-1pt}
    \includegraphics[trim=5 5 3820 560,clip,width=0.495\textwidth]{images/dp_sinkhorn/celeba/image_3.png}
    \rule[4pt]{\textwidth}{1pt}
    \includegraphics[trim=0 0 133 0,clip,width=\textwidth]{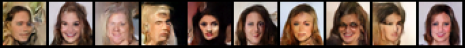}
    \end{minipage}
    \caption{CelebA images generated by \textbf{DataLens} (\emph{1st row}), \textbf{DP-MEPF} (\emph{2nd row}), \textbf{DP-Sinkhorn} (\emph{3rd row}), and our \textbf{DPDM} (\emph{4th row}) for DP-$\varepsilon{=}10$.}
        \label{fig:celeba_comparison}
\end{figure}

\section{\TC{Additional Experiments on More Challenging Problems}} \label{sec:add_experiments}
\subsection{Diverse Datasets}
We provide results for additional experiments on challenging diverse datasets, namely, CIFAR-10~\citep{krizhevsky2009learning} and ImageNet~\citep{deng2009imagenet} (resolution 32x32), both in the class-conditional setting similar to our other experiments on MNIST and Fashion-MNIST. To the best of our knowledge, we are the first to attempt pure DP image generation on ImagenNet. 

For both experiments, we use the same neural network architecture as for CelebA (32x32) in the main paper; see model hyperparameters in~\Cref{tab:model_hyperparameters_and_training_details}. On CIFAR-10, we train for 500 epochs using noise multiplicity $K=32$ under the privacy setting $(\varepsilon=10, \delta=10^{-5})$. In ImageNet, we train for 100 epochs using noise multiplicity $K=8$ under the privacy setting $(\varepsilon=10, \delta=7 \cdot 10^{-7})$; training for longer (or using larger $K$) was not possible on ImageNet due to its sheer size. We achieve FIDs of 97.7 and 61.3 for CIFAR-10 and ImageNet, respectively. No previous works reported FID scores on these datasets and for these privacy settings, but we hope that our scores can serve as reference points for future work. In~\Cref{fig:more_at_32}, we show samples for both datasets from our DPDMs and visually compare to an existing DP generative modeling work on CIFAR-10, DP-MERF~\citep{harder2021dp}. Our DPDMs cannot learn clear objects; however, overall image/pixel statistics seem to be captured correctly. In contrast, the DP-MERF baseline collapses entirely. We are not aware of any other works tackling these tasks. Hence, we believe that DPDMs represent a major step forward.
\begin{figure}
     \centering
     \begin{subfigure}[b]{0.4\textwidth}
         \centering
         \includegraphics[width=\textwidth]{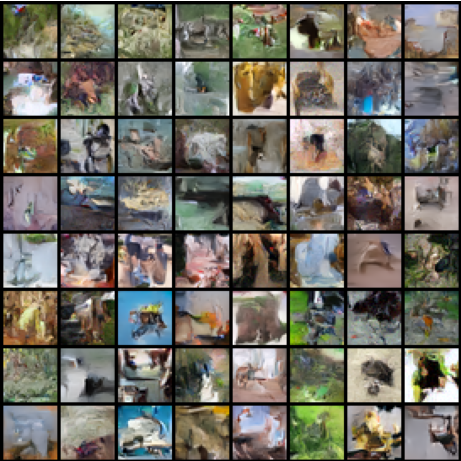}
         \caption{DPDM (\emph{ours}) (ImageNet)}
     \end{subfigure}
     \hfill
     \begin{subfigure}[b]{0.4\textwidth}
         \centering
         \includegraphics[width=\textwidth]{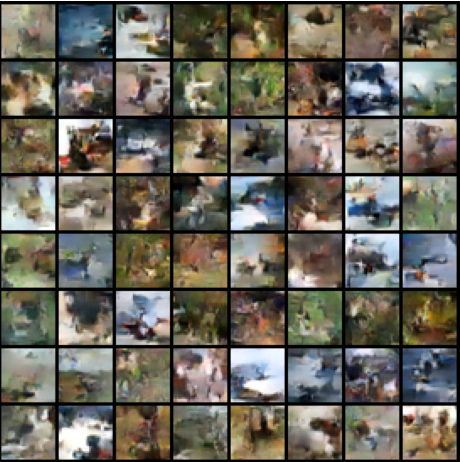}
         \caption{DPDM (\emph{ours}) (CIFAR-10)}
     \end{subfigure}
     \hfill
     \begin{subfigure}[b]{0.10\textwidth}
         \centering
         \includegraphics[width=\textwidth]{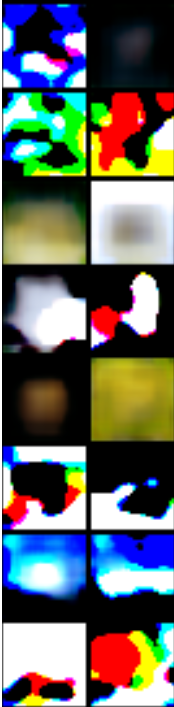}
         \caption{}
     \end{subfigure}
    \caption{Additional experiments on challenging diverse datasets. Samples from our DPDM on ImageNet and CIFAR-10, as well as CIFAR-10 samples from DP-MERF~\citep{harder2021dp} in (c).}
    \label{fig:more_at_32}
\end{figure}
\subsection{Higher Resolution}
We provide results for additional experiments on CelebA at higher resolution (64x64). To accommodate the higher resolution, we added an additional upsampling/downsampling layer to the U-Net, which results in roughly a 11\% increase in the number of parameters, from 1.80M to 2.00M parameters. The only row that changed in the CelebA model hyperparameter table (\Cref{tab:model_hyperparameters_and_training_details}) is the one about the channel multipliers. It is adapted from (1,2,2) to (1,2,2,2). We train for 300 epochs using $K=8$ under the privacy setting $(\varepsilon=10, \delta=10^{-6})$. We achieve an FID of 78.3 (again, for reference; no previous works reported quantitative results on this task). In~\Cref{fig:more_at_64}, we show samples and visually compare to existing DP generative modeling work on CelebA at 64x64 resolution. Although the faces generated by our DPDM are somewhat distorted, the model overall is able to clearly generate face-like structures. In contrast, DataLens generates incoherent very low quality outputs. No other existing works tried generating 64x64 CelebA images with rigorous DP guarantees, to the best of our knowledge. Also this experiment implies that DPDMs can be considered a major step forward for DP generative modeling. 
\begin{figure}
     \centering
     \begin{subfigure}[b]{0.8\textwidth}
         \centering
         \includegraphics{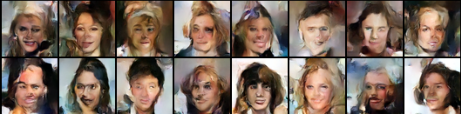}
         \caption{DPDM (\emph{ours})}
     \end{subfigure}
     \hfill
     \begin{subfigure}[b]{0.8\textwidth}
         \centering
         \includegraphics[width=\textwidth]{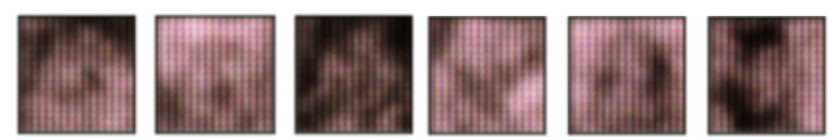}
         \caption{DataLens~\citep{wang2021datalens}}
     \end{subfigure}
     \caption{Additional experiments on CelebA at higher resolution (64x64). Samples from our method and DataLens~\citep{wang2021datalens}.}
    \label{fig:more_at_64}
\end{figure}

\section{Ethics, Reproducibility, Limitations and Future Work} \label{sec:ethics_and_reproducibility}
\vspace{-1mm}
\TIM{Our work improves the state-of-the-art in differentially private generative modeling and we validate our proposed DPDMs on image synthesis benchmarks. Generative modeling of images has promising applications, for example for digital content creation and artistic expression~\citep{Bailey2020thetools}, but it can in principle also be used for malicious purposes~\citep{vaccari2020deepfakes,mirsky2021deepfakesurvey,nguyen2021deep}. However, differentially private image generation methods, including our DPDM, are currently not able to produce photo-realistic content, which makes such abuse unlikely. 

As discussed in~\Cref{sec:introduction}, a severe issue in modern generative models is that they can easily overfit to the data distribution, thereby closely reproducing training samples and leaking privacy of the training data. Our DPDMs aim to rigorously address such problems via the well-established DP framework and fundamentally protect the privacy of the training data and prevent overfitting to individual data samples. This is especially important when training generative models on diverse and privacy-sensitive data. Therefore, DPDMs can potentially act as an effective medium for data sharing without needing to worry about data privacy, which we hope will benefit the broader machine learning community.
Note, however, that although DPDM provides privacy protection in generative learning, information about individuals cannot be eliminated entirely, as no useful model can be learned under DP-$(\varepsilon{=}0, \delta{=}0)$. This should be communicated clearly to dataset participants.

An important question for future work is scaling DPDMs to \textit{(a)} larger datasets and \textit{(b)} more complicated, potentially higher-resolution image datasets. Regarding \textit{(a)}, recall that an increased number of data points leads to less noise injection during DP-SGD training (while keeping all other parameters fixed). Therefore, we believe that DPDMs should scale well with respect to larger datasets; see, for example, our results on ImageNet, a large dataset, which are to some degree better than the results on CIFAR-10, a relatively small dataset (\Cref{sec:add_experiments}). Regarding  \textit{(b)}, however, scaling to more complicated, potentially higher-resolution datasets is challenging if the number of data points is kept fixed. Higher-resolution data requires larger neural networks, but this comes with more parameters, which can be problematic during DP training (see \Cref{sec:dpsgd_training}). Using parameter-efficient architectures for diffusion models may be promising; for instance, see concurrent work by~\cite{jabri2023scalable}. Generally, we believe that scaling both \textit{(a)} and \textit{(b)} are interesting avenues for future work.

To aid reproducibility of the results and methods presented in our paper, we made source code to reproduce all quantitative and qualitative results of the paper publicly available, including detailed instructions. Moreover, all training details and hyperparameters are already described in detail in this appendix, in particular in~\Cref{sec:model_and_implementation_details}.}

\end{document}